\documentclass[11pt]{article}
\usepackage{geometry}
\usepackage{microtype}
\usepackage[parfill]{parskip} 
\usepackage{subcaption}
\usepackage{cite}
\usepackage{xspace}
\usepackage{booktabs} 
\pdfoutput=1
\usepackage{hyperref}
\usepackage[toc,page]{appendix}
\usepackage{amssymb}
\usepackage{amsmath}
\usepackage{amsthm}
\usepackage{mathtools} 
\usepackage{comment}
\usepackage{cleveref}
\usepackage{multirow}
\usepackage{multicol}
\usepackage[export]{adjustbox}
 \usepackage{arydshln}
 \usepackage{makecell}
 \usepackage{algorithm}
\usepackage{algorithmic}
\usepackage{enumitem}
\usepackage{xcolor}

\DeclareMathOperator*{\argmin}{argmin}
\DeclareMathOperator*{\argmax}{argmax}

\newcommand{\RR}{\mathbb{R}}

\newcommand{\WW}{\mathcal{W}}
\newcommand{\one}{\mathit{1}}
\newcommand{\sign}{\mathop{\mathrm{sign}}}
\newcommand{\vol}{\textnormal{vol}}
\newcommand{\supp}{\textnormal{supp}}
\newcommand{\ex}{\textnormal{ex}}

\usepackage{apptools}
\AtAppendix{\counterwithin{theorem}{section}} 

\newtheorem{theorem}{Theorem}

\newtheorem{lemma}[theorem]{Lemma}

\newtheorem{claim}{Claim}

\newtheorem{assumption}{Assumption}

\begin{document}

\title{Local Hyper-Flow Diffusion}

\author{
        \name Kimon~Fountoulakis\textsuperscript{\hspace{0.1cm}} \email kfountou@uwaterloo.ca \\ 
        \addr School of Computer Science, \\ University of Waterloo, \\ Waterloo, ON, Canada.
        \AND
        \name Pan~Li\textsuperscript{\hspace{0.1cm}} \email panli@purdue.edu \\ 
        \addr Department of Computer Science, \\ Purdue University, \\ West Lafayette, IN, United States.
        \AND
	\name Shenghao~Yang\textsuperscript{\hspace{0.1cm}} \email shenghao.yang@uwaterloo.ca \\ 
        \addr School of Computer Science, \\ University of Waterloo, \\ Waterloo, ON, Canada.
}

\author{
	Kimon~Fountoulakis%
        \thanks{School of Computer Science, University of Waterloo, Waterloo, ON, Canada. E-mail: kfountou@uwaterloo.ca.
        }
        \and
        Pan~Li%
        \thanks{Department of Computer Science, Purdue University, West Lafayette, IN, USA. E-mail: panli@purdue.edu.
        } 
        \and
	Shenghao~Yang%
        \thanks{School of Computer Science, University of Waterloo, Waterloo, ON, Canada. E-mail: s286yang@uwaterloo.ca.
        }
}

\maketitle

\begin{abstract}
\fontsize{10pt}{12pt}\selectfont
Recently, hypergraphs have attracted a lot of attention due to their ability to capture complex relations among entities.\ The insurgence of hypergraphs has resulted in data of increasing size and complexity that exhibit interesting small-scale and local structure, e.g., small-scale communities and localized node-ranking around a given set of seed nodes.\ Popular and principled ways to capture the local structure are the local hypergraph clustering problem and related seed set expansion problem.\ In this work, we propose the first local diffusion method that achieves edge-size-independent Cheeger-type guarantee for the problem of local hypergraph clustering while applying to a rich class of higher-order relations that covers many previously studied special cases.\ Our method is based on a primal-dual optimization formulation where the primal problem has a natural network flow interpretation, and the dual problem has a cut-based interpretation using the $\ell_2$-norm penalty on associated cut-costs.\ We demonstrate the new technique is significantly better than state-of-the-art methods on both synthetic and real-world data.\
\end{abstract}
\section{Introduction}\label{intro}

Hypergraphs~\cite{berge1984hypergraphs} generalize graphs by allowing a hyperedge to consist of multiple nodes that capture higher-order relationships in complex systems and datasets~\cite{milo2002network}.\ Hypergraphs have been used for music recommendation on Last.fm data~\cite{BTCWWZH10}, news recommendation~\cite{LL13}, sets of product reviews on Amazon~\cite{JJJ19}, and sets of co-purchased products at Walmart~\cite{AVB20}.\ Beyond the internet, hypergraphs are used for analyzing higher-order structure in neuronal, air-traffic and food networks~\cite{BGL16,LM17}.\

In order to explore and understand higher-order relationships in hypergraphs, recent work has made use of cut-cost functions that are defined by associating each hyperedge with a specific set function. These functions assign specific penalties of separating the nodes within individual hyperedges. They generalize the notion of hypergraph cuts and are crucial for determining small-scale community structure~\cite{LM17,LVHLG20}. The most popular cut-cost functions with increasing capability to model complex multiway relationships are the unit cut-cost~\cite{Lawler73,IWW93,Hadley95}, cardinality-based cut-cost~\cite{VBK20,VBK20b} and general submodular cut-cost~\cite{LM18,yoshida2019cheeger}.\ An illustration of a hyperedge and the associated cut-cost function is given in Figure~\ref{fig:hyper-graph}.\ In the simplest setting, all cut-cost functions take value either 0 or 1 (e.g., the case when $\gamma_1=\gamma_2=1$ in Figure~\ref{fig:hyperedge-b}), we obtain a unit cut-cost hypergraph.\ In a slightly more general setting, the cut-costs are determined solely by the number of nodes in either side of the hyperedge cut (e.g., the case when $\gamma_1=1/2$ and $\gamma_2=1$ in Figure~\ref{fig:hyperedge-b}), we obtain a cardinality-based hypergraph.\ We refer to hypergraphs associated with arbitrary submodular cut-cost functions (e.g., the case when $\gamma_1=1/2$ and $0 \le \gamma_2 \le 1$ in Figure~\ref{fig:hyperedge-b}) as general submodular hypergraphs.

Hypergraphs that arise from data science applications consist of interesting small-scale local structure such as local communities~\cite{LVHLG20,takai2020hypergraph}.\ Exploiting this structure is central to the above mentioned applications on hypergraphs and related applications in machine learning and applied mathematics~\cite{BGH2021}.\ We consider local hypergraph clustering as the task of finding a community-like cluster around a given set of seed nodes, where nodes in the cluster are densely connected to each other while relatively isolated to the rest of the graph.\ One of the most powerful primitives for the local hypergraph clustering task is the graph diffusion.\ Diffusion on a graph is the process of spreading a given initial mass from some seed node(s) to neighbor nodes using the edges of the graph. Graph diffusions have been successfully employed in the industry, for example, both Pinterest and Twitter use diffusion methods for their recommendation systems~\cite{EJLLSSUL2018,ELSSU2020,GGLSWZ13}. Google uses diffusion methods to perform clustering query refinements~\cite{SMWH10}. Let us not forget PageRank~\cite{BP98,PBMW99}, Google's model for their search engine.

\begin{figure}[t]
        \centering
	\begin{subfigure}{0.15\textwidth}\end{subfigure}
	\begin{subfigure}{0.3\textwidth}
		\centering
		\includegraphics[height=3.5cm]{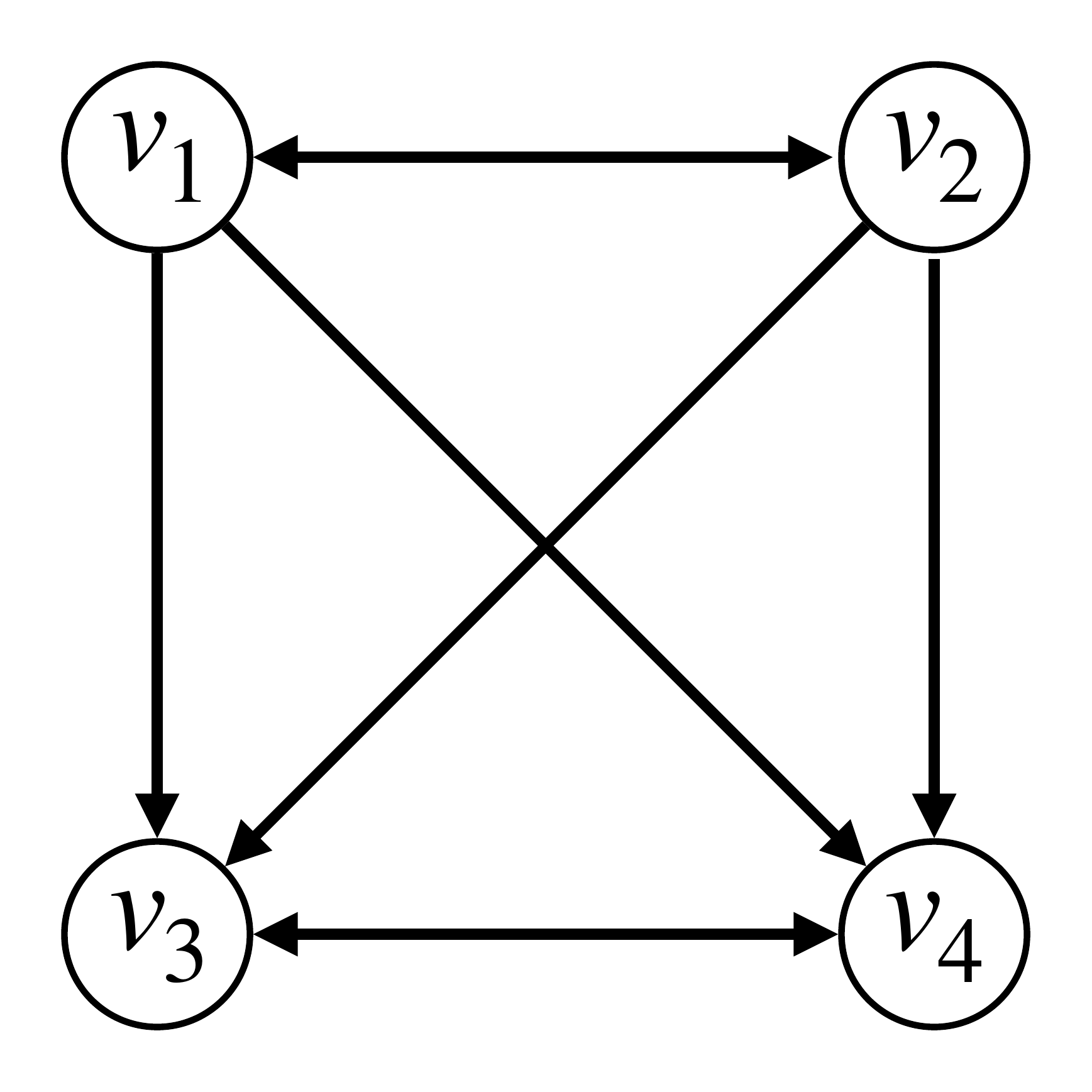}
		\caption{Network motif}\label{fig:hyperedge-a}
	\end{subfigure}%
	\begin{subfigure}{0.4\textwidth}
		\centering
		\includegraphics[height=3.5cm]{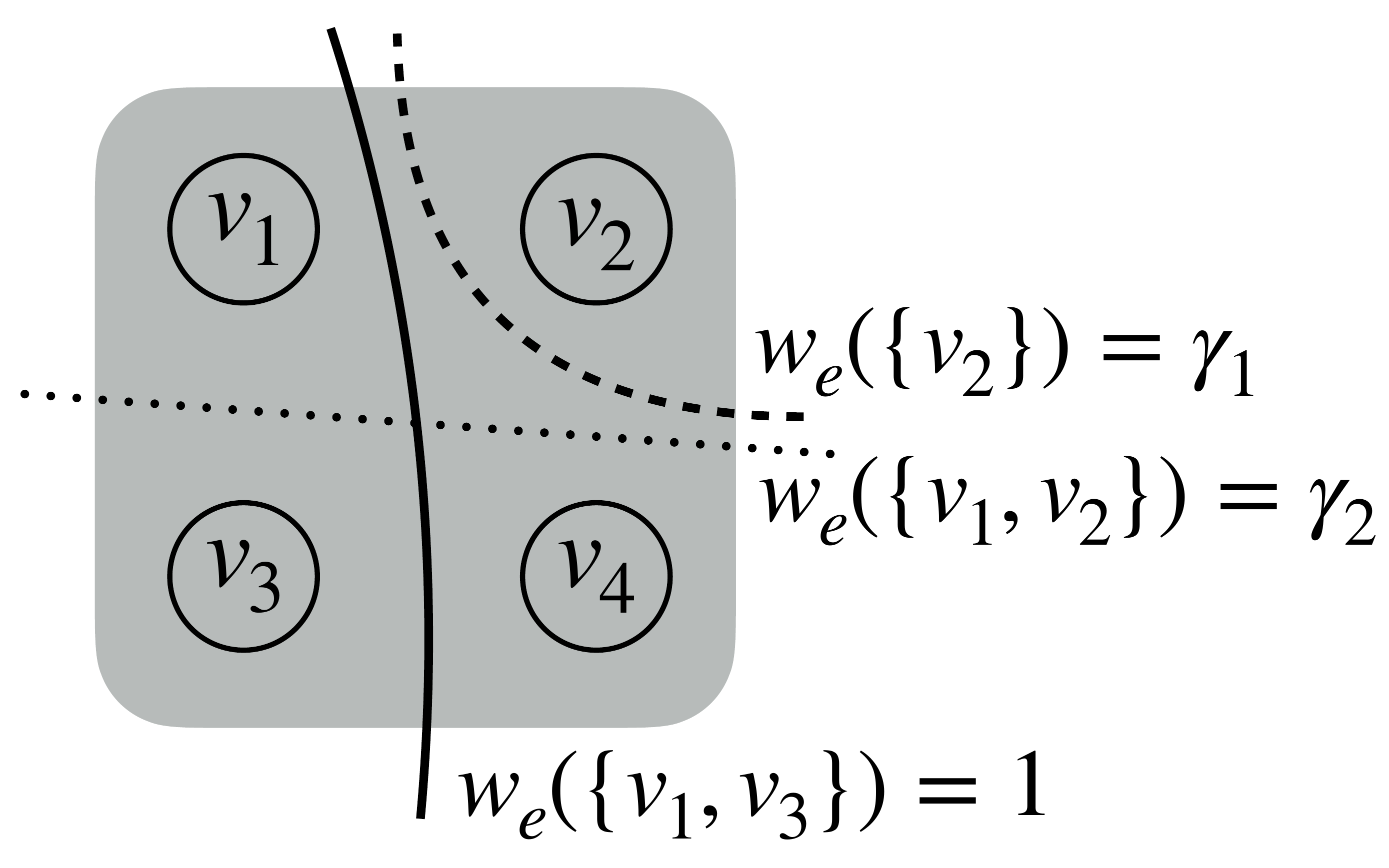}
		\caption{Hyperedge and cut-cost $w_e$}\label{fig:hyperedge-b}
	\end{subfigure}
        \caption{A food network can be mapped into a hypergraph by taking each network pattern in (a) as a hyperedge~\cite{LM17}.\ This network pattern captures carbon flow from two preys ($v_1, v_2$) to two predators ($v_3, v_4$).\ (b) is a hyperedge associated with cut-cost $w_e$ that models their relations: $w_e$ is a set function defined over the node set $e$ s.t. $w_e(\{v_i\})=\gamma_1$ for $i=1,2,3,4$, $w_e(\{v_1,v_2\})=\gamma_2$, $w_e(\{v_1,v_3\})=w_e(\{v_1,v_4\})=1$ and $w_{e}(S)=w_{e}(e\backslash S)$ for $S\subseteq e$. $w_e$ becomes the unit cut-cost when $\gamma_1=\gamma_2 = 1$; $w_e$ is cardinality-based if $\gamma_1=1/2$ and $\gamma_2=1$; more generally, $w_e$ is submodular if $\gamma_1=1/2$ and $0 \le \gamma_2 \le 1$. The specific choices depend on the application.}
        \label{fig:hyper-graph}
\end{figure}

Empirical and theoretical performance of local diffusion methods is often measured on the problem of local hypergraph clustering~\cite{WFHM2017,SWF20,LVHLG20}.\ 
Existing local diffusion methods only directly apply to hypergraphs with the unit cut-cost~\cite{IG20,takai2020hypergraph}. For the slightly more general cardinality-based cut-cost, they rely on graph reduction techniques which result in a rather pessimistic edge-size-dependent approximation error~\cite{yin2017local,IG20,LVHLG20,VBK20}.\ 
Moreover, none of the existing methods is capable of processing general submodular cut-costs.\ 
In this work, we are interested in designing a diffusion framework that (i) achieves stronger theoretical guarantees for the problem of local hypergraph clustering, (ii) is flexible enough to work with general submodular hypergraphs, and (iii) permits computationally efficient algorithms.\ We propose the first local diffusion method that simultaneously accomplishes these goals. 

In what follows we describe our main contributions and previous work.\ In Section~\ref{sec:notation} we provide preliminaries and notations.\ In Section~\ref{sec:hfd} we introduce our diffusion model from a combinatorial flow perspective.\ In Section~\ref{sec:clustering} we discuss the local hypergraph clustering problem and Cheeger-type quadratic approximation error.\ In Section~\ref{sec:experiments} we perform experiments using both synthetic and real datasets.\

\subsection{Our main contributions}

In this work we propose a generic local diffusion model that applies to hypergraphs characterized by a rich class of cut-cost functions that covers many previously studied special cases, e.g., unit, cardinality-based and submodular cut-costs.\ We provide the first edge-size-independent Cheeger-type approximation error for the problem of local hypergraph clustering using any of these cut-costs.\ In particular, assume that there exists a cluster $C$ with conductance $\Phi(C)$, and assume that we are given a set of seed nodes that reasonably overlaps with $C$, then the proposed diffusion model can be used to find a cluster $\hat{C}$ with conductance at most $O(\sqrt{\Phi(C)})$ (in the appendix we show that an $\ell_p$-norm version of the proposed model can achieve $O(\Phi(C))$ asymptotically).\ Our hypergraph diffusion model is formulated as a convex optimization problem.\ It has a natural combinatorial flow interpretation that generalizes the notion of network flows over hyperedges.\ We show that the optimization problem can be solved efficiently by an alternating minimization method.\ In addition, we prove that the number of nonzero nodes in the optimal solution is independent of the size of the hypergraph, and it only depends on the size of the initial mass.\ This key property ensures that our algorithm scales well in practice for large datasets.\ We evaluate our method using both synthetic and real-world data. We show that our method improves accuracy significantly for hypergraphs with unit, cardinarlity-based and general submodular cut-costs for local clustering.\

\subsection{Previous work}

Recently, clustering methods on hypergraphs received renewed interest.\ Different methods require different assumptions about the hyperedge cut-cost, which can be roughly categorized into unit cut-cost, cardinality-based (and submodular) cut-cost and general submodular cut-cost.\ Moreover, existing methods can be either global, where the output is not localized around a given set of seed nodes, or local, where the output is a tiny cluster around a set of seed nodes.\ Local algorithms are the only scalable ones for large hypergraphs, which is our main focus.\ Many works propose global methods and thus they are not scalable to large hypergraphs \cite{zhou2006learning,ARV2009,hein2013total,Louis2015,BGL16,zhang2017re,CLTZ2018,LM17,LM18,CLM2019,yoshida2019cheeger,li2020quadratic,CR2019,HAPP2020,takai2020hypergraph}.\ 
Local diffusion-based methods are more relevant to our work~\cite{IG20,LVHLG20,VBK20}.\ In particular, iterative hypergraph min-cut methods for the local hypergraph clustering problem can be adopted~\cite{VBK20}.\ However, these methods require in theory and in practice a large seed set, i.e., they are not expansive and thus cannot work with one seed node.\ The expansive combinatorial diffusion~\cite{WFHM2017} is generalized for hypergraphs~\cite{IG20}, which can detect a target cluster using only one seed node.\ However, combinatorial methods have a large bias towards low conductance clusters as opposed to finding the target cluster~\cite{FLDM2020}.\ The most relevant paper to our work is~\cite{LVHLG20}.\ However, the proposed methods in \cite{LVHLG20} depend on a reduction from hypergraphs to directed graphs.\ This results in an approximation error for clustering that is proportional to the size of hyperedges and induces performance degeneration when the hyperedges are large.\ In fact, none of the above approaches (including global and local ones) has an edge-size-independent approximation error bound for even simple cardinality-based hypergraphs.\ Moreover, existing local approaches do not work for general submodular hypergraphs.\

\section{Preliminaries and Notations}\label{sec:notation}

{\bf Submodular function.} Given a set $S$, we denote $2^S$ the power set of $S$ and $|S|$ the cardinality of $S$. A submodular function $F : 2^S \rightarrow \RR$ is a set function such that $F(A) + F(B) \ge F(A\cup B) + F(A \cap B)$ for any $A,B \subseteq S$.

{\bf Submodular hypergaph.} A hypergraph $H = (V,E)$ is defined by a set of nodes $V$ and a set of hyperedges $E \subseteq 2^V$, i.e., each hyperedge $e \in E$ is a subset of $V$. A hypergraph is termed {\em submodular} if every $e \in E$ is associated with a submodular function $w_e : 2^e \rightarrow \RR_{\geq0}$~\cite{LM18}.\ The weight $w_e(S)$ indicates the cut-cost of splitting the hyperedge $e$ into two subsets, $S$ and $e\setminus S$. This general form allows us to describe the potentially complex higher-order relation among multiple nodes (Figure~\ref{fig:hyper-graph}).\ A proper hyperedge weight $w_e$ should satisfy that $w_e(\emptyset) = w_e(e) = 0$. To ease notation we extend the domain of $w_e$ to $2^V$ by setting $w_e(S) := w_e(S \cap e)$ for any $S \subseteq V$. We assume without loss of generality that $w_e$ is normalized by $\vartheta_e := \max_{S \subseteq e} w_e(S)$, so that $w_e(S) \in [0,1]$ for any $S \subseteq V$. For the sake of simplicity in presentation, we assume that $\vartheta_e = 1$ for all $e$.\footnote{This is without loss of generality. In the appendix we show that our method works with arbitrary $\vartheta_e > 0$.} A submodular hypergraph is written as $H = (V,E,\WW)$ where $\WW := \{w_e, \vartheta_e\}_{e \in E}$. Note that when $w_e(S)=1$ for any $S\in2^e\backslash \{\emptyset,e\}$, the definition reduces to \emph{unit cut-cost hypergraphs}. When $w_e(S)$ only depends on $|S|$, it reduces to \emph{cardinality-based cut-cost hypergraphs}.

{\bf Vector/Function on $V$ or $E$.} For a set of nodes $S \subseteq V$, we denote $\one_S$ the indicator vector of $S$, i.e., $[\one_S]_v =1$ if $v \in S$ and 0 otherwise. For a vector $x \in \RR^{|V|}$, we write $x(S) := \sum_{v \in S}x_v$, where $x_v$ in the entry in $x$ that corresponds to $v \in V$. We define the support of $x$ as $\supp(x) := \{v \in V | x_v \neq 0\}$. The support of a vector in $\RR^{|E|}$ is defined analogously. We refer to a function over nodes $x:V\rightarrow\RR$ and its explicit representation as a $|V|$-dimensional vector interchangeably.

{\bf Volume, cut, conductance.} Given a submodular hypergraph $H = (V,E,\WW)$, the {\em degree} of a node $v$ is defined as $d_v := |\{e \in E: v \in e\}|$. We reserve $d$ for the vector of node degrees and $D = \mbox{diag}(d)$. We refer to $\vol(S) := d(S)$ as the {\em volume} of $S \subseteq V$. A {\em cut} is treated as a proper subset $C \subset V$, or a partition $(C,\bar{C})$ where $\bar{C} := V \setminus C$. The {\em cut-set} of $C$ is defined as $\partial C := \{e \in E | e \cap C \neq \emptyset, e \cap \bar{C} \neq \emptyset\}$; the {\em cut-size} of $C$ is defined as $\vol(\partial C) := \sum_{e \in \partial C} \vartheta_e w_e(C) = \sum_{e \in E}\vartheta_e w_e(C)$. The {\em conductance} of a cut $C$ in $H$ is $\Phi(C) := \frac{\vol(\partial C)}{\min\{\vol(C),\vol(V\setminus C)\}}$.

{\bf Flow.} A flow {\em routing} over a hyperedge $e$ is a function $r_e : e \rightarrow \RR$ where $r_e(v)$ specifies the amount of mass that flows from $\{v\}$ to $e \setminus \{v\}$ over $e$. To ease notation we extend the domain of $r_e$ to $V$ by identifying $r_e(v) = 0$ for $v \not\in e$, so $r_e$ is treated as a function $r_e: V \rightarrow \RR$ or equivalently a $|V|$-dimensional vector.\ The net (out)flow at a node $v$ is given by $\sum_{e \in E}r_e(v)$.\ Given a routing function $r_e$ and a set of nodes $S \subseteq V$, a {\em directional routing} on $e$ with direction $S \rightarrow e \setminus S$ is represented by $r_e(S)$, which specifies the net amount of mass that flows from $S$ to $e \setminus S$.\  A routing $r_e$ is called {\em proper} if it obeys flow conservation, i.e., $r_e^T\one_e = 0$.\ Our flow definition generalizes the notion of network flows to hypergraphs. We provide concrete illustrations in Figure~\ref{fig:flowexamples}.

\begin{figure}[ht!]
	\centering
	\begin{subfigure}{0.245\textwidth}
		\centering
		\includegraphics[height=3cm]{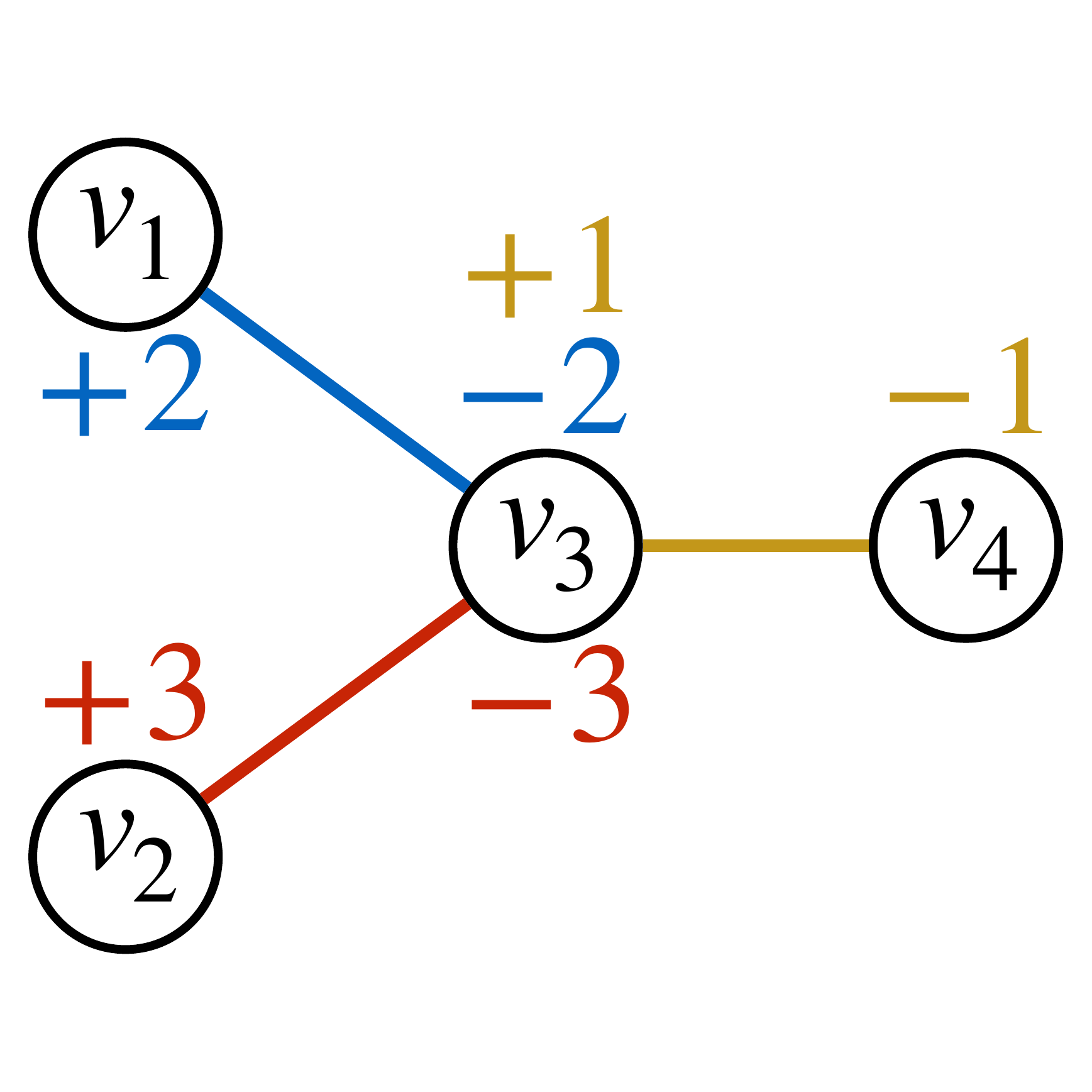}
		\caption{Flows on graph}\label{fig:flowexamples-a}
	\end{subfigure}%
	\hspace{.5mm}
	\begin{subfigure}{0.245\textwidth}
		\centering
		\includegraphics[height=3cm]{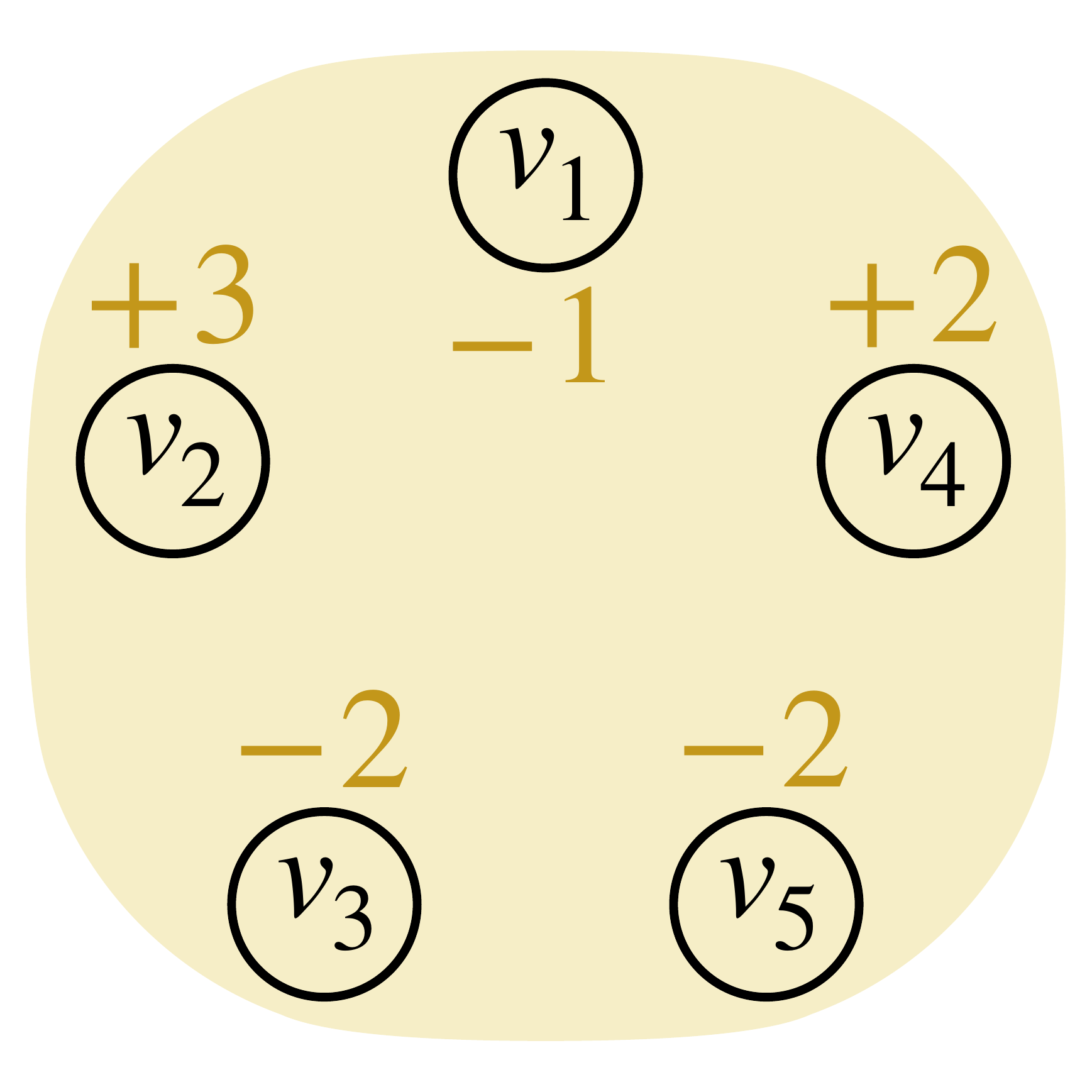}
		\caption{Hyperedge routing}\label{fig:flowexamples-b}
	\end{subfigure}
	\begin{subfigure}{0.49\textwidth}
		\centering
		\includegraphics[height=3cm]{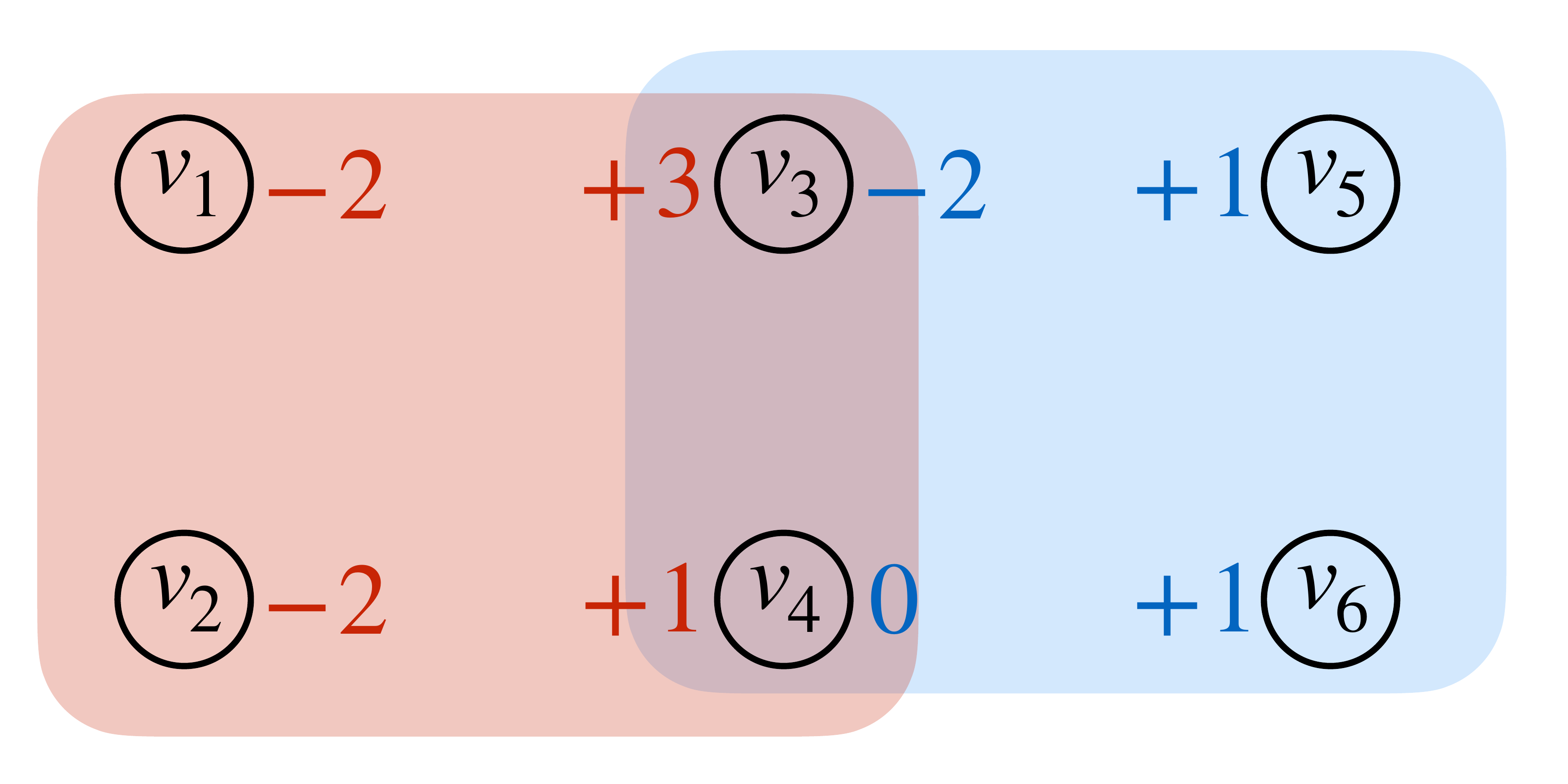}
		\caption{Flows on hypergraph}\label{fig:flowexamples-c}
	\end{subfigure}
	\caption{Illustration of proper flow routings. The numbers next to each node correspond to entries in the flow routing $r_e$ over a (hyper)edge $e$. We assign the same color to a (hyper)edge and its associated flow values. Our flow definition is a natural generalization of graph edge flow where $r_e(v) = \pm f$ if and only if $v \in e$, i.e., $v$ is incident to $e$, where $f$ and the sign determine the amplitude and direction of the flow over $e$. In Figure~\ref{fig:flowexamples-a}, the net (out)flow at node $v_3$ is given by $\sum_{e \in E}r_e(v_3) = 1-2-3=-4$. In Figure~\ref{fig:flowexamples-b}, the directional flow from $\{v_1\}$ to $\{v_2,v_3,v_4, v_5\}$ over this hyperedge equals $-1$; similarly, the directional flow from $\{v_1,v_2,v_4\}$ to $\{v_3,v_5\}$ equals $3+2-1=4$, etc. In Figure~\ref{fig:flowexamples-c}, the net (out)flow at node $v_3$ is given by $\sum_{e \in E}r_e(v_3) = 3 - 2 = 1$.} 
	\label{fig:flowexamples}
\end{figure}

\section{Diffusion as an Optimization Problem}\label{sec:hfd}

In this section we provide details of the proposed local diffusion method.\ We consider diffusion process as the task of spreading mass from a small set of seed nodes to a larger set of nodes.\ More precisely, given a hypergraph $H = (V,E,\WW)$,
we assign each node a sink capacity specified by a sink function $T$, i.e., node $v$ is allowed to hold at most $T(v)$ amount of mass. In this work we focus on the setting where $T(v) = d_v$, so that a high-degree node that is part of many hyperedges can hold more mass than a low-degree node that is part of few hyperedges.\ Moreover, we assign each node some initial mass specified by a source function $\Delta$, i.e., node $v$ holds $\Delta(v)$ amount of mass at the start of the diffusion.\ In order to encourage the spread of mass in the hypergraph, the initial mass on the seed nodes is larger than their capacity.\ This forces the seed nodes to diffuse mass to neighbor nodes to remove their excess mass.\ In Section~\ref{sec:clustering} we will discuss the choice of $\Delta$ to obtain good theoretical guarantees for the problem of local hypergraph clustering.\ Details about the local hypergraph clustering problem are provided in Section~\ref{sec:clustering}.\ 

Given a set of proper flow routings $r_e$ for $e \in E$, recall that $\sum_{e \in E} r_e(v)$ specifies the amount of net (out)flow at node $v$.\ Therefore, the vector $m = \Delta - \sum_{e \in E} r_e$ gives the amount of net mass at each node after routing.\ The {\em excess mass} at a node $v$ is $\ex(v) := \max\{m_v-d_v, 0\}$.\ In order to force the diffusion of initial mass we could simply require that $\ex(v) = 0$ for all $v \in V$, or equivalently, $\Delta - \sum_{e \in E} r_e \le d$.\ But to provide additional flexibility in the diffusion dynamics, we introduce a hyper-parameter $\sigma\ge0$ and we impose a softer constraint $\Delta - \sum_{e \in E} r_e \le d + \sigma D z$, where $z \in \RR^{|V|}$ is an optimization variable that controls how much excess mass is allowed on each node.\ In the context of numerical optimization, we show in Section~\ref{sec:alg} that $\sigma$ allows a reformulation which makes the optimization problem amenable to efficient alternating minimization schemes.

Note that so far we have not yet talked about how specific higher-order relations among nodes within a hyperedge would affect the flow routings over it. Apparently, simply requiring that the $r_e$'s obey flow conservation (i.e., $r_e^T\one_e = 0$) similar to the standard graph setting is not enough for hypergraphs. An important difference between hyperedge flows and graph edge flows is that additional constraints on $r_e$ are in need. To this end, we consider $r_e = \phi_e\rho_e$ for some $\phi_e \in \RR_+$ and $\rho_e \in B_e$, where
\[
B_e := \{\rho_e \in \mathbb{R}^{|V|} ~|~ \rho_e(S) \le w_e(S),\forall S \subseteq V,  \ \mbox{and} \ \rho_e(V) = w_e(V)\}
\]
is the {\em base polytope}~\cite{Bach2011a} for the submodular cut-cost $w_e$ associated with hyperedge $e$. It is straightforward to see that $r_e(v) = 0$ for every $v \not\in e$ and $r_e^T\one_e = 0$, so $r_e$ defines a proper flow routing over $e$.\ Moreover, for any $e \subseteq V$, recall that $r_e(S)$ represents the net amount of mass that moves from $S$ to $e\setminus S$ over hyperedge $e$. Therefore, the constraints $\rho_e(S) \le w_e(S)$ for $S \subseteq e$ mean that the directional flows $r_e(S)$ are upper bounded by a submodular function $\phi_ew_e(S)$. Intuitively, one may think of $\phi_e$ and $\rho_e$ as the {\em scale} and the {\em shape} of $r_e$, respectively.

The goal of our diffusion problem is to find low cost routings $r_e \in \phi_eB_e$ for $e \in E$ such that the capacity constraint $\Delta - \sum_{e\in E}r_e \le d + \sigma Dz$ is satisfied.\ We consider the (weighted) $\ell_2$-norm of $\phi$ and $z$ as the cost of diffusion.\ In the appendix we show that one readily extends the $\ell_2$-norm to $\ell_p$-norm for any $p\ge2$.\ Formally, we arrive at the following convex optimization formulation (input: the source function $\Delta$, the hypergraph $H=(V, E, \WW)$, and a hyper-parameter $\sigma$):
\begin{equation}
\label{eq:primalL2}
\min_{\phi\in\mathbb{R}^{|E|}_+,z\in\mathbb{R}^{|V|}_+}~\frac{1}{2}\sum_{e \in E} \phi_e^2 + \frac{\sigma}{2}\sum_{v \in V}d_vz_v^2, \ \ \mbox{s.t.}\ \ \Delta - \sum_{e \in E} r_e \le d + \sigma D z, \ r_e \in \phi_e B_e, \forall e \in E.
\end{equation}
We name problem~\eqref{eq:primalL2} Hyper-Flow Diffusion (HFD) for its combinatorial flow interpretation we discussed above. The dual problem of \eqref{eq:primalL2} is:
\begin{equation}
\label{eq:dualL2}
	\max_{x \in \mathbb{R}^{|V|}_+}~(\Delta - d)^Tx - \frac{1}{2}\sum_{e \in E} f_e(x)^2 - \frac{\sigma}{2}\sum_{v \in V}d_v x_v^2,
\end{equation}
where $f_e$ in \eqref{eq:dualL2} is the support function of the base polytope $B_e$ given by $f_e(x) := \max_{\rho_e \in B_e} \rho_e^T x$. $f_e$ is also known as the {\em Lov\'{a}sz extension} of the submodular function $w_e$.

We provide a combinatorial interpretation for \eqref{eq:dualL2} and leave algebraic derivations to the appendix.\ For the dual problem, one can view the solution $x$ as assigning heights to nodes, and the goal is to separate/cut the nodes with source mass from the rest of the hypergraph.\ Observe that the linear term in the dual objective function encourages raising $x$ higher on the seed nodes and setting it lower on others.\ The cost $f_e(x)$ captures the discrepancy in node heights over a hyperedge $e$ and encourages smooth height transition over adjacent nodes.\ The dual solution embeds nodes into the nonnegative real line, and this embedding is what we actually use for local clustering and node ranking.\

\section{Local Hypergraph Clustering}\label{sec:clustering}

In this section we discuss the performance of the primal-dual pair \eqref{eq:primalL2}-\eqref{eq:dualL2}, respectively, in the context of local hypergraph clustering.\ We consider a generic hypergraph $H = (V,E,\WW)$ with submodular hyperedge weights $\WW = \{w_e,\vartheta_e\}_{e\in E}$.\ Given a set of seed nodes $S \subset V$, the goal of local hypergraph clustering is to identify a target cluster $C\subset V$ that contains or overlaps well with $S$.\ This generalizes the definition of local clustering over graphs~\cite{FKSCM2017}.\ To the best of our knowledge, we are the first one to consider this problem for general submodular hypergraphs.\ We consider a subset of nodes having low conductance as a good cluster, i.e., these nodes are well-connected internally and well-separated from the rest of the hypergraph.\ Following prior work on local hypergraph clustering, we assume the existence of an unknown target cluster $C$ with conductance $\Phi(C)$.\ We prove that applying sweep-cut to an optimal solution $\hat{x}$ of \eqref{eq:dualL2} returns a cluster $\hat{C}$ whose conductance is at most quadratically worse than $\Phi(C)$.\ Note that this result resembles Cheeger-type approximation guarantees of spectral clustering in the graph setting~\cite{ACL06}, and it is the first result that is independent of hyperedge size for general hypergraphs.\ We keep the discussion at high level and defer details to the appendix, where we prove a more general, and stronger, i.e., constant approximation error result when the primal problem~\eqref{eq:primalL2} is penalized by the $\ell_p$-norm for any $p\ge2$. 

In order to start a diffusion process we need to provide the source mass $\Delta$.\ Similar to the $p$-norm flow diffusion in the graph setting~\cite{SWF20}, we let
\begin{equation}
\label{eq:source}
\Delta(v) = \left\{\begin{array}{ll} \delta d_v & \mbox{if}~ v \in S, \\ 0 & \mbox{otherwise,} \end{array}\right. 
\end{equation}
where $S$ is a set of seed nodes and $\delta\ge1$. Below, we make the assumptions that the seed set $S$ and the target cluster $C$ have some overlap, there is a constant factor of $\vol(C)$ amount of mass trapped in $C$ initially, and the hyper-parameter $\sigma$ is not too large.\ Note that Assumption~\ref{assum:delta} is without loss of generality: if the right value of $\delta$ is not known apriori, we can always employ binary search to find a good choice.\ Assumption~\ref{assum:sigma} is very weak as it allows $\sigma$ to reside in an interval containing 0.
\begin{assumption}
\label{assum:overlap}
$\vol(S \cap C) \ge \alpha \vol(C)$ and $\vol(S \cap C) \ge \beta \vol(S)$ for some $\alpha,\beta \in (0,1]$.
\end{assumption}
\begin{assumption}
\label{assum:delta}
The source mass $\Delta$ as specified in \eqref{eq:source} satisfies $\delta = 3/\alpha$, so $\Delta(C) \ge 3\vol(C)$.
\end{assumption}
\begin{assumption}
\label{assum:sigma}
$\sigma$ satisfies $0 \le \sigma \le \beta\Phi(C)/3$.
\end{assumption}

Let $\hat{x}$ be an optimal solution for the dual problem~\eqref{eq:dualL2}. For $h>0$ define the sweep sets $S_h := \{v \in V | \hat{x}_v \ge h\}$. We state the approximation property in Theorem~\ref{thm:conductance}.
\begin{theorem}
\label{thm:conductance}
Under Assumptions~\ref{assum:overlap},~\ref{assum:delta},~\ref{assum:sigma}, there exists $h>0$ such that
$
	\Phi(S_h) \le O(\sqrt{\Phi(C)}/\alpha\beta).
$
\end{theorem}

One of the challenges we face in establishing the result in Theorem~\ref{thm:conductance} is making sure that our diffusion model enjoys both good clustering guarantees and practical algorithmic advantages at the same time.\ This is achieved by introducing the hyper-parameter $\sigma$ to our diffusion problem.\ We will demonstrate how $\sigma$ helps with algorithmic development in Section 5, but from a clustering perspective, the additional flexibility given by $\sigma>0$ complicates the underlying diffusion dynamics, making it more difficult to analyze.\ Another challenge is connecting the Lov\'{a}sz extension $f_e(x)$ in \eqref{eq:dualL2} with the conductance of a cluster.\ We resolve all these problems by combining a generalized Rayleigh quotient result for submodular hypergraphs~\cite{LM18}, primal-dual convex conjugate relations between \eqref{eq:primalL2} and \eqref{eq:dualL2}, and a classical property of the Choquet integral/Lov\'{a}sz extension.

Let $(\hat{\phi},\hat{r},\hat{z})$ be an optimal solution for the primal problem~\eqref{eq:primalL2}. We state the following lemma on the locality (i.e., sparsity) of the optimal solutions, which justifies why HFD is a local diffusion method.

\begin{lemma}
\label{lem:support}
$|\supp(\hat{\phi})| \le \vol(\supp(\hat{x})) \le \|\Delta\|_1$; moreover, $\vol(\supp(\hat{z})) = \vol(\supp(\hat{x}))$ if $\sigma > 0$.
\end{lemma}

\section{Optimization algorithm for HFD}\label{sec:alg}

We use a simple Alternating Minimization (AM)~\cite{Beck2015} method that efficiently solves the primal diffusion problem~\eqref{eq:primalL2}.\ For $e \in E$, we define a diagonal matrix $A_e \in \mathbb{R}^{|V| \times |V|}$ such that $[A_e]_{v,v} = 1$ if $v \in e$ and 0 otherwise.\ Denote $\mathcal{C} := \{(\phi,r) : r_e \in \phi_eB_e, ~\forall e \in E\}$.\ The following Lemma~\ref{lem:sep} allows us to cast problem~\eqref{eq:primalL2} to an equivalent separable formulation amenable to the AM method.\
\begin{lemma}
\label{lem:sep}
The following problem is equivalent to \eqref{eq:primalL2} for any $\sigma > 0$, in the sense that $(\hat{\phi},\hat{r},\hat{z})$ is optimal in \eqref{eq:primalL2} for some $\hat{z} \in \RR^{|V|}$ if and only if $(\hat{\phi},\hat{r},\hat{s})$ is optimal in \eqref{eq:sepL2} for some $\hat{s} \in \bigotimes_{e\in E}\mathbb{R}^{|V|}$.
\begin{equation}
\label{eq:sepL2}
	\min_{\phi, r, s}~\frac{1}{2}\sum_{e \in E} \left(\phi_e^2 + \frac{1}{\sigma}\left\|s_e - r_e\right\|_2^2\right),
	~\mbox{s.t.}~(\phi,r) \in \mathcal{C}, ~\Delta - \sum_{e \in E}s_e \le d, ~s_{e,v} = 0, \forall v \not\in e.
\end{equation}
\end{lemma}

The AM method for problem~\eqref{eq:sepL2} is given in Algorithm~\ref{alg:AM}. The first sub-problem corresponds to computing projections to a group of cones, where all the projections can be computed in parallel.\ The computation of each projection depends on the choice of base polytope $B_e$.\ If the hyperedge weight $w_e$ is unit cut-cost, $B_e$ holds special structures and projection can be computed with $O(|e|\log|e|)$~\cite{li2020quadratic}.\ For general $B_e$, a conic Fujishige-Wolfe minimum norm algorithm can be adopted to obtain the projection~\cite{li2020quadratic}.\ The second sub-problem in Algorithm~\ref{alg:AM} can be easily computed in closed-form.\ We provide more information about Algorithm~\ref{alg:AM} and its convergence properties in the appendix.

{\centering
\begin{minipage}{.6\linewidth}
\begin{algorithm}[H]
	\caption{Alternating Minimization for problem \eqref{eq:sepL2}}
	\label{alg:AM}
	{\bf Initialization:} 
	\[
	\phi^{(0)} := 0, r^{(0)} := 0, s^{(0)}_e := D^{-1}A_e\left[\Delta - d\right]_+, \forall e \in E.
	\]
	{\bf For $k = 0,1,2,\ldots$ do:}
	\begin{align*}
	&(\phi^{(k+1)}, r^{(k+1)}) := \argmin\limits_{(\phi,r)\in\mathcal{C}} \sum\limits_{e \in E}(\phi_e^2 + \tfrac{1}{\sigma}\|s_e^{(k)} - r_e\|_2^2)\\
	&s^{(k+1)} := \argmin\limits_{s} \sum\limits_{e \in E} \|s_e - r_e^{(k+1)}\|_2^2\\
	&\hspace{22mm}\mbox{s.t.}~\Delta-\sum\limits_{e\in E}s_e \le d, \ s_{e,v}=0,\forall v\not\in e.
	\end{align*}
\end{algorithm}
\end{minipage}
\par
}
\vspace{5pt}

We remark that the reformulation \eqref{eq:sepL2} for $\sigma>0$ is crucial from an algorithmic point of view.\ If $\sigma = 0$, then the primal problem~\eqref{eq:primalL2} has complicated coupling constraints that are hard to deal with.\ In this case, one has to resort to the dual problem~\eqref{eq:dualL2}.\ However, problem~\eqref{eq:dualL2} has a nonsmooth objective function, which prohibits applicability of optimization methods for smooth objective functions.\ Even though subgradient method may be applied, we have observed empirically that its convergence rate is extremely slow for our problem, and early stopping results in a bad quality output.

Lastly, as noted in Lemma~\ref{lem:support}, the number of nonzeros in the optimal solution is upper bounded by $\|\Delta\|_1$.\ In Figure~\ref{fig:sol-nnz} we plot the number of nodes having positive excess (which equals the number of nonzeros in the dual solution $\hat{x}$) at every iteration of Algorithm~\ref{alg:AM}.\ Figure~\ref{fig:sol-nnz} indicates that Algorithm~\ref{alg:AM} is strongly local, meaning that it works only on a small fraction of nodes (and their incident hyperedges) as opposed to producing dense iterates.\ This key empirical observation has enabled our algorithm to scale to large datasets by simply keeping track of all active nodes and hyperedges.\ Proving that the worst-case running time of AM depends only on the number of nonzero nodes at optimality as opposed to size of the whole hypergraph is an open problem, which we leave for future work.\

\begin{figure}[ht!]
\centering
	\begin{subfigure}{0.33\textwidth}
	\centering
	\includegraphics[height=2.6cm]{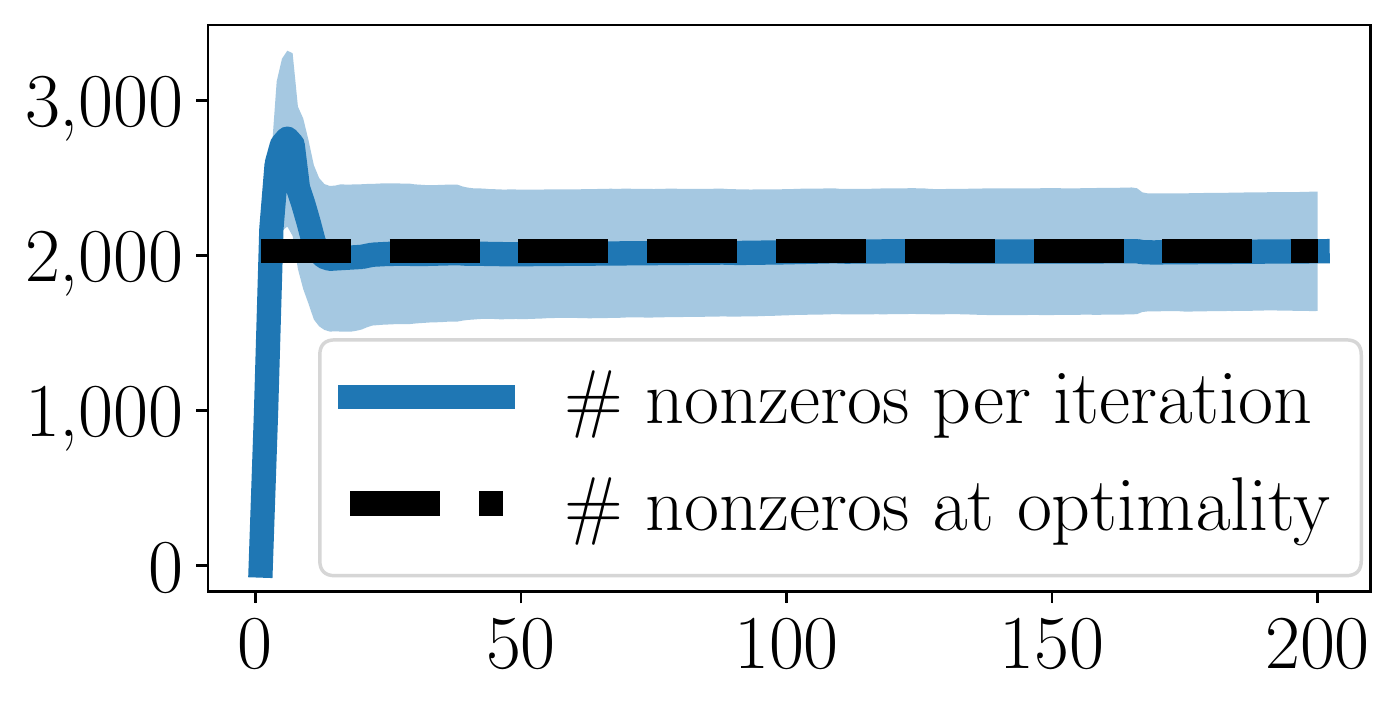}
	\caption{Cluster 12 - Gift Cards}\label{fig:sol-nnz-a}
	\end{subfigure}%
	\begin{subfigure}{0.33\textwidth}
	\centering
	\includegraphics[height=2.6cm]{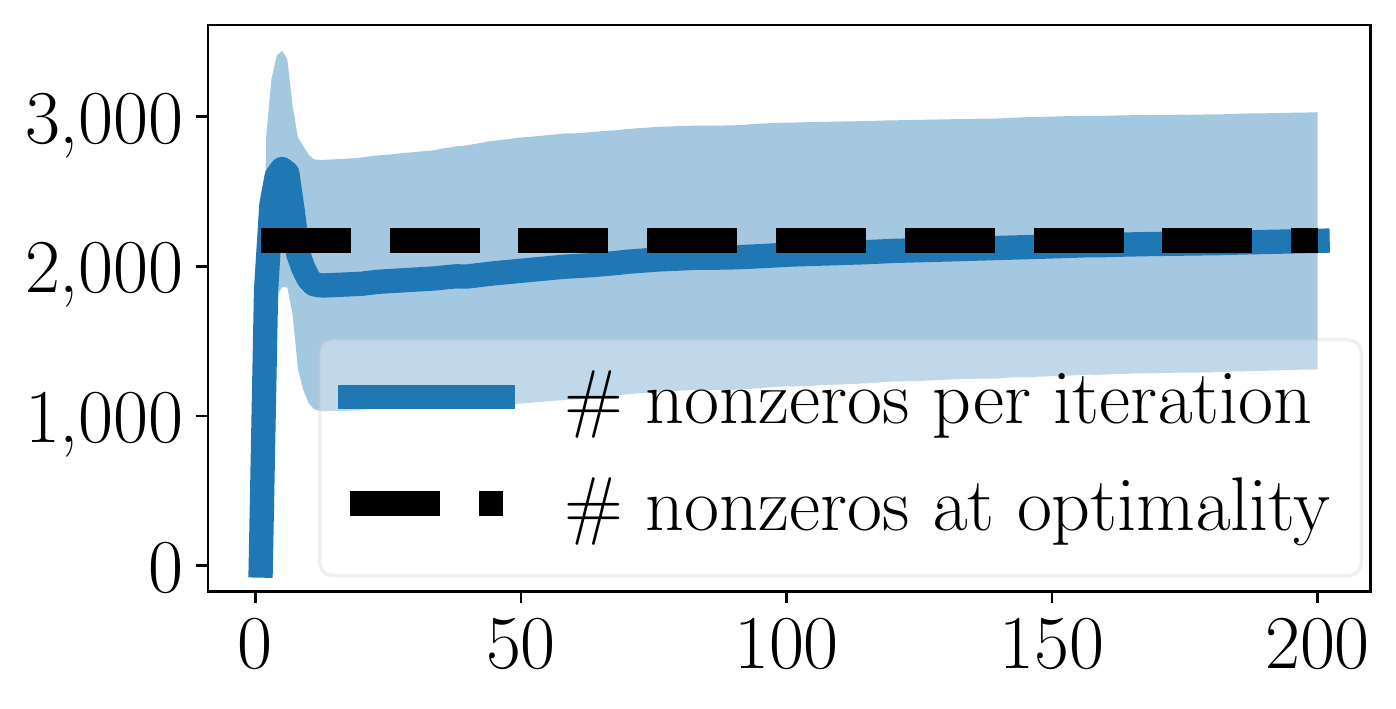}
	\caption{Cluster 18 - Magazine Subs.}\label{fig:sol-nnz-b}
	\end{subfigure}%
	\begin{subfigure}{0.33\textwidth}
	\centering
	\includegraphics[height=2.6cm]{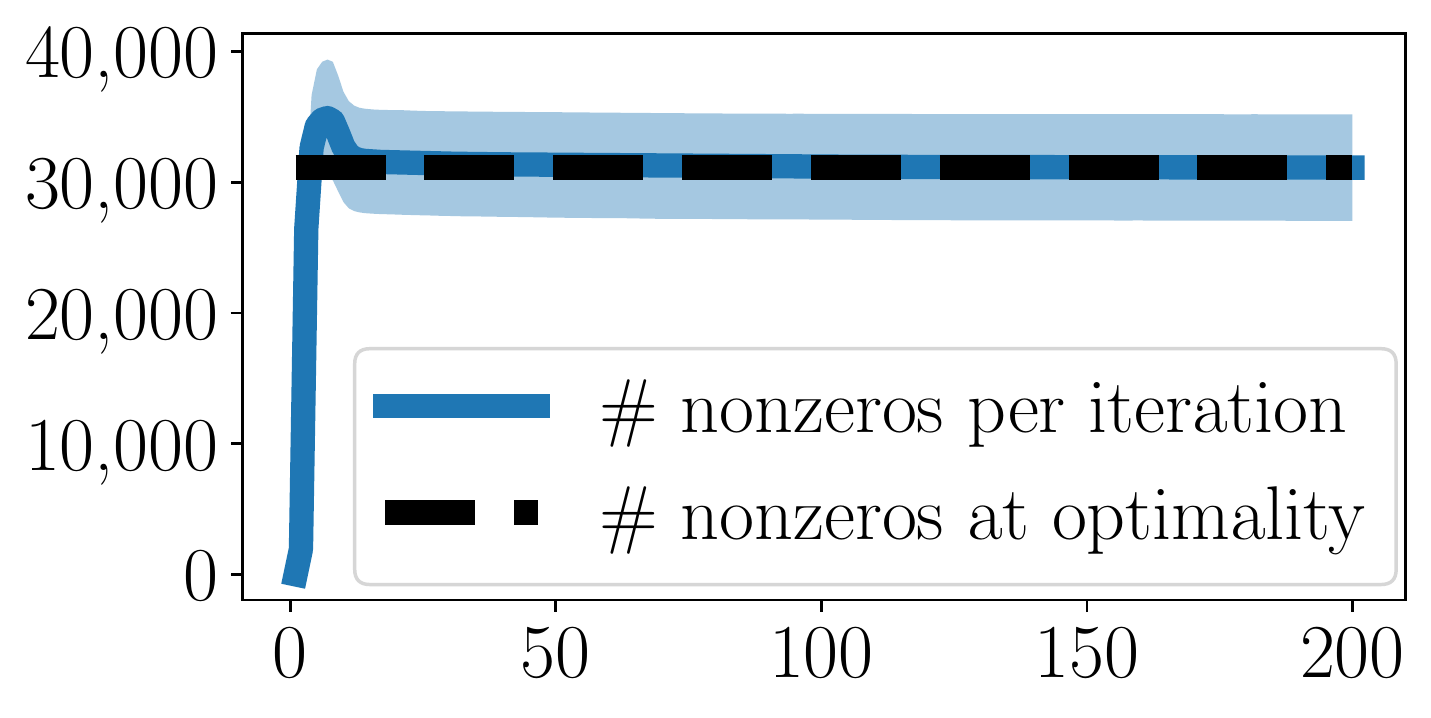}
	\caption{Cluster 24 - Prime Pantry}\label{fig:sol-nnz-c}
	\end{subfigure}
\caption{The blue solid line plots the number of nonzeros in the dual solution $x$ over 200 iterations of Algorithm~\ref{alg:AM}, when it is applied to solve HFD on the Amazon-reviews hypergraph for local clustering.\ See Section~\ref{sec:experiments-real} for details about the dataset.\ The error bars show standard deviation over 10 trials. In each trial we pick a different seed node and set the same amount of initial mass.\ The black dashed line shows the average number of nonzeros at optimality.\ The algorithm touches only a small fraction of nodes out the total 2,268,264 nodes in the Amazon-reviews dataset.}
\label{fig:sol-nnz}
\end{figure}

\section{Empirical Results}\label{sec:experiments}

In this section we evaluate the performance of HFD for local clustering.\ First, we carry out experiments on synthetic hypergraphs with varying target cluster conductances and varying hyperedge sizes.\ For the unit cut-cost setting, we show that HFD is more robust and has better performance when the target cluster is noisy; for a cardinality-based cut-cost setting, we show that the edge-size-independent approximation guarantee is important for obtaining good recovery results.\ Second, we carry out experiments using real-world data.\ We show that HFD significantly outperforms existing state-of-the-art diffusion methods for both unit and cardinality-based cut-costs.\ Moreover, we provide a compelling example where specialized submodular cut-cost is necessary for obtaining good results.\ Code that reproduces all results is available at \url{https://github.com/s-h-yang/HFD}.

\subsection{Synthetic experiments using hypergraph stochastic block model (HSBM)}

{\bf The generative model.} We generalize the standard $k$-uniform hypergraph stochastic block model ($k$HSBM)~\cite{GD2014} to allow different types of inter-cluster hyperedges appear with possibly different probabilities according to the cardinality of hyperedge cut. Let $V = \{1,2,\ldots,n\}$ be a set of nodes and let $k\ge2$ be the required constant hyperedge size. We consider $k$HSBM with parameters $k$, $n$, $p$, $q_j$, $j = 1,2,\ldots,\lfloor k/2 \rfloor$. The model samples a $k$-uniform hypergraph according to the following rules: (i) The community label $\sigma_i \in \{0,1\}$ is chosen uniformly at random for $i \in V$;\footnote{We consider two blocks for simplicity. In general the model applies to any number of blocks.} (ii) Each size $k$ subset $e = \{v_1, v_2, \ldots, v_k\}$ of $V$ appears independently as a hyperedge with probability
\[
	\mathbb{P}(e \in E) = 
	\left\{
	\begin{array}{ll}
		p & \mbox{if}~\sigma_{v_1} = \sigma_{v_2} = \cdots = \sigma_{v_k}, \\
		q_j & \mbox{if}~ \min\{k - \sum_{i=1}^k \sigma_{v_i}, \sum_{i=1}^k \sigma_{v_i}\} = j.
	\end{array}
	\right.
\]
If $k=3$ or all $q_j$'s are the same, then we obtain the standard two-block $k$HSBM. We use this setting to evaluate HFD for unit cut-cost. If $q_j$'s are different, then we obtain a cardinality-based $k$HSBM. In particular, when $q_1 \ge q_2 \ge \cdots \ge q_{\lfloor k/2 \rfloor}$, it models the scenario where hyperedges containing similar numbers of nodes from each block are rare, while small noises (e.g., hyperedges that have one or two nodes in one block and all the rest in the other block) are more frequent. We use $q_1 \gg q_j$, $j \ge 2$, to evaluate HFD for cardinality-based cut-cost. There are other random hypergraph models, for example the Poisson degree-corrected HSBM~\cite{CVB2021} that deals with degree heterogeneity and edge size heterogeneity. In our experiments we focus on $k$HSBM because it allows stronger control over hyperedge sizes. We provide details on data generation in the appendix.

{\bf Task and methods.} We consider the local hypergraph clustering problem. We assume that we are given a single labelled node and the goal is to recover all nodes having the same label. Using a single seed node the most common (and sought-after) practice for local graph clustering tasks. We test the performance of HFD with two other methods: (i) Localized Quadratic Hypergraph Diffusions (LH)~\cite{LVHLG20}, which can be seen as a hypergraph analogue of Approximate Personalized PageRank (APPR); (ii) ACL~\cite{ACL06}, which is used to compute APPR vectors on a standard graph obtained from reducing a hypergraph through star expansion~\cite{ZSC1999}.\footnote{There are other heuristic methods which first reduce a hypergraph to a graph by clique expansion~\cite{BGL16} and then apply diffusion methods for standard graphs. We did not compare with this approach because clique expansion often results in a dense graph and consequently makes the computation slow. Moreover, it has been shown in \cite{LVHLG20} that clique expansion did not offer significant performance improvement over star expansion.}

{\bf Cut-costs and parameters.} We consider both unit cut-cost, i.e., $w_e(S) = 1$ if $S \cap e \neq \emptyset$ and $e \setminus S \neq \emptyset$, and {\it cardinality cut-cost} $w_e(S) = \min\{|S \cap e|, |e \setminus S|\}/\lfloor |e|/2 \rfloor$. HFD that uses unit and cardinality cut-costs are denoted by U-HFD and C-HFD, respectively. LH also works with both unit and cardinality cut-costs and we specify them by U-LH and C-LH, respectively.
For HFD, we initialize the seed mass so that $\|\Delta\|_1$ is a constant factor times the volume of the target cluster. We set $\sigma=0.01$. We highly tune LH by performing binary search over its parameters $\kappa$ and $\delta$ and pick the output cluster having the lowest conductance. For ACL we use the same parameter choices as in~\cite{LVHLG20}. Details on parameter setting are provided in the appendix.

{\bf Results.} For each hypergraph, we randomly pick a block as the target cluster. We run the methods 50 times. Each time we choose a different node from the target cluster as the single seed node.\\
\underline{{\it Unit cut-cost results.}} Figure~\ref{fig:varycond} shows local clustering results when we fix $k=3$ but vary the conductance of the target cluster (i.e., constant $p$ but varying $q_1$). Observe that the performances of all methods become worse as the target cluster becomes more noisy, but U-HFD has significantly better performance than both U-LH and ACL when the conductance of the target cluster is between 0.2 and 0.4. The reason that U-HFD performs better is in part because it requires much weaker conditions for the theoretical conductance guarantee to hold. On the contrary, LH assumes an upper bound on the conductance of the target cluster~\cite{LVHLG20}. This upper bound is dataset-dependent and could become very small in many cases, leading to poor practical performances. We provide more details in this perspective in the appendix. ACL with star expansion is a heuristic method that has no performance guarantee.\\ 
\underline{{\it Cardinality cut-cost results.}} Figure~\ref{fig:edgesize} shows the median (markers) and 25-75 percentiles (lower-upper bars) of conductance ratios (i.e., the ratio between output conductance and ground-truth conductance, lower is better) and F1 scores for different methods for $k\in\{3,4,5,6\}$. The target cluster for each $k$ has conductance around 0.3.\footnote{See the appendix for similar results when we fix the target cluster conductances around 0.2 and 0.25, respectively. These cover a reasonably wide range of scenarios in terms of the target conductance and illustrate the performance of algorithms for different levels of noise.} For $k=3$, unit and cardinality cut-costs are equivalent, therefore all methods have similar performances. As $k$ increases, cardinality cut-cost provides better performance than unit cut-cost in both conductance and F1. However, since the theoretical approximation guarantee of C-LH depends on hyperedge size~\cite{LVHLG20}, there is a noticeable performance degradation for C-LH when we increase $k=3$ to $k=4$. On the other hand, the performance of C-HFD appears to be independent from $k$, which aligns with our conductance bound in Theorem~\ref{thm:conductance}. 

\begin{figure}[ht!]
	\centering
	\begin{subfigure}{0.45\textwidth}
		\centering
		\includegraphics[width=.95\textwidth]{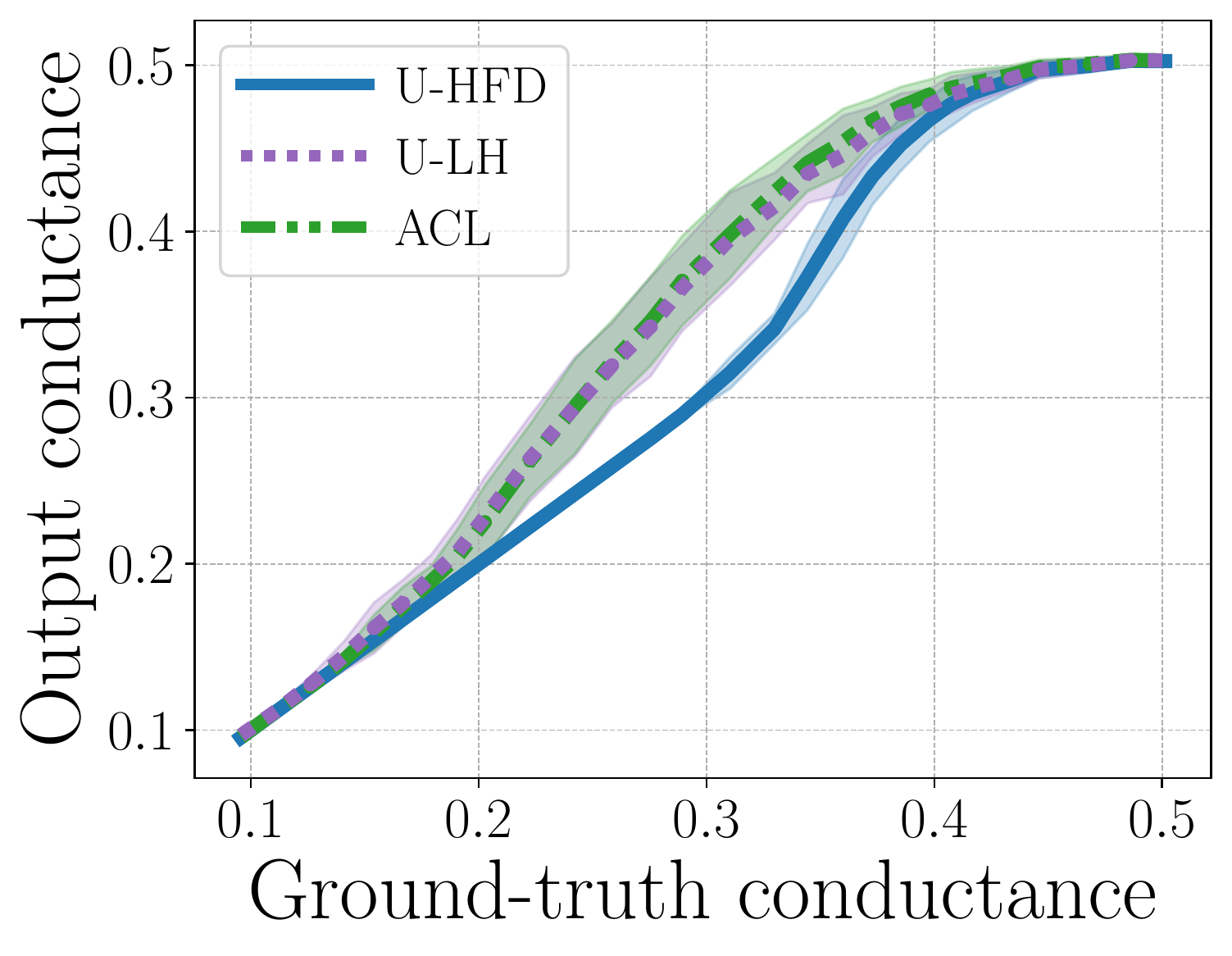}
	\end{subfigure}%
	\begin{subfigure}{0.45\textwidth}
		\centering
		\includegraphics[width=.95\textwidth]{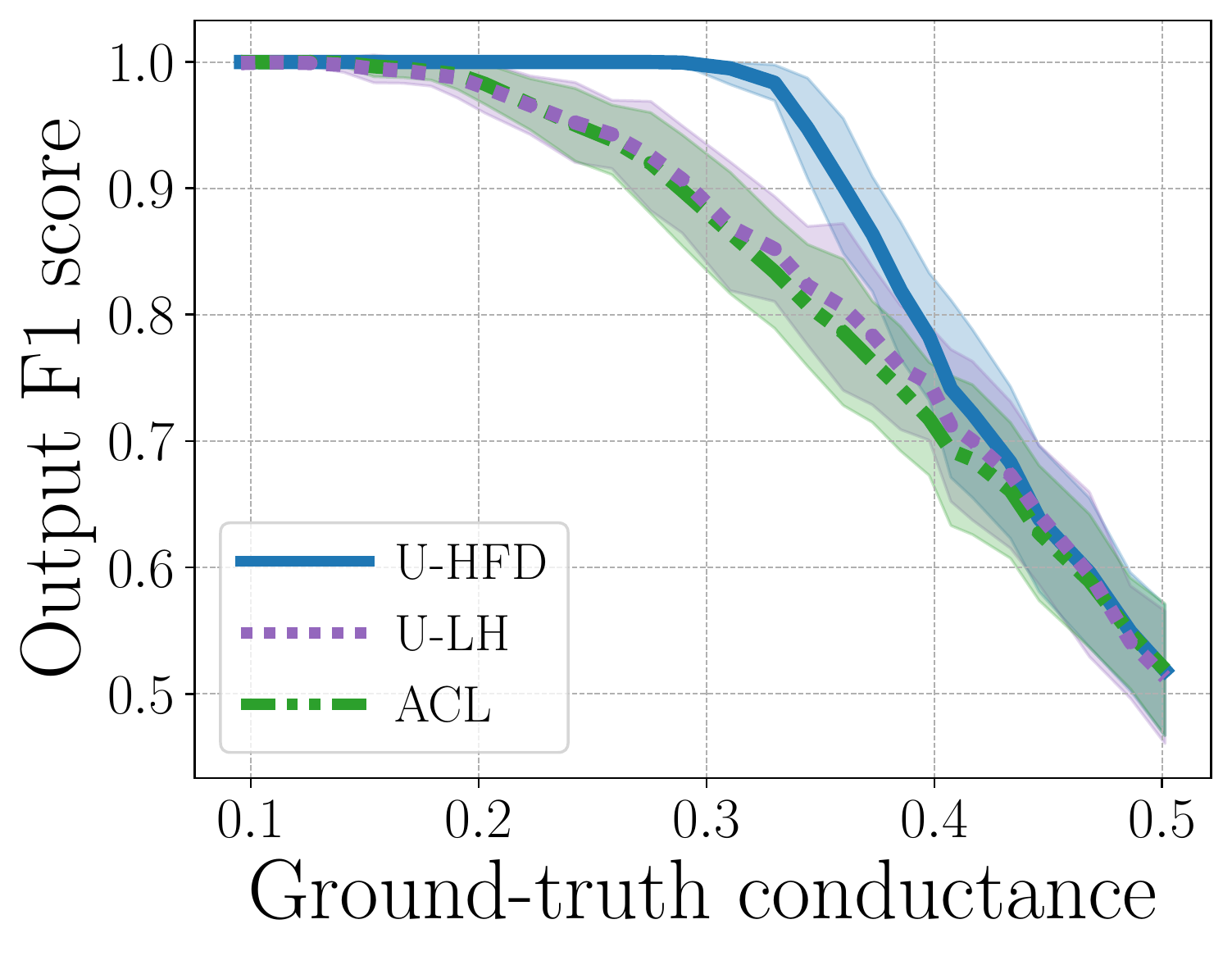}
	\end{subfigure}
	\caption{Output conductance and F1 against ground-truth conductance}
	\label{fig:varycond}
\end{figure}

\begin{figure}[ht!]
	\centering
	\begin{subfigure}{0.45\textwidth}
		\centering
		\includegraphics[width=.95\textwidth]{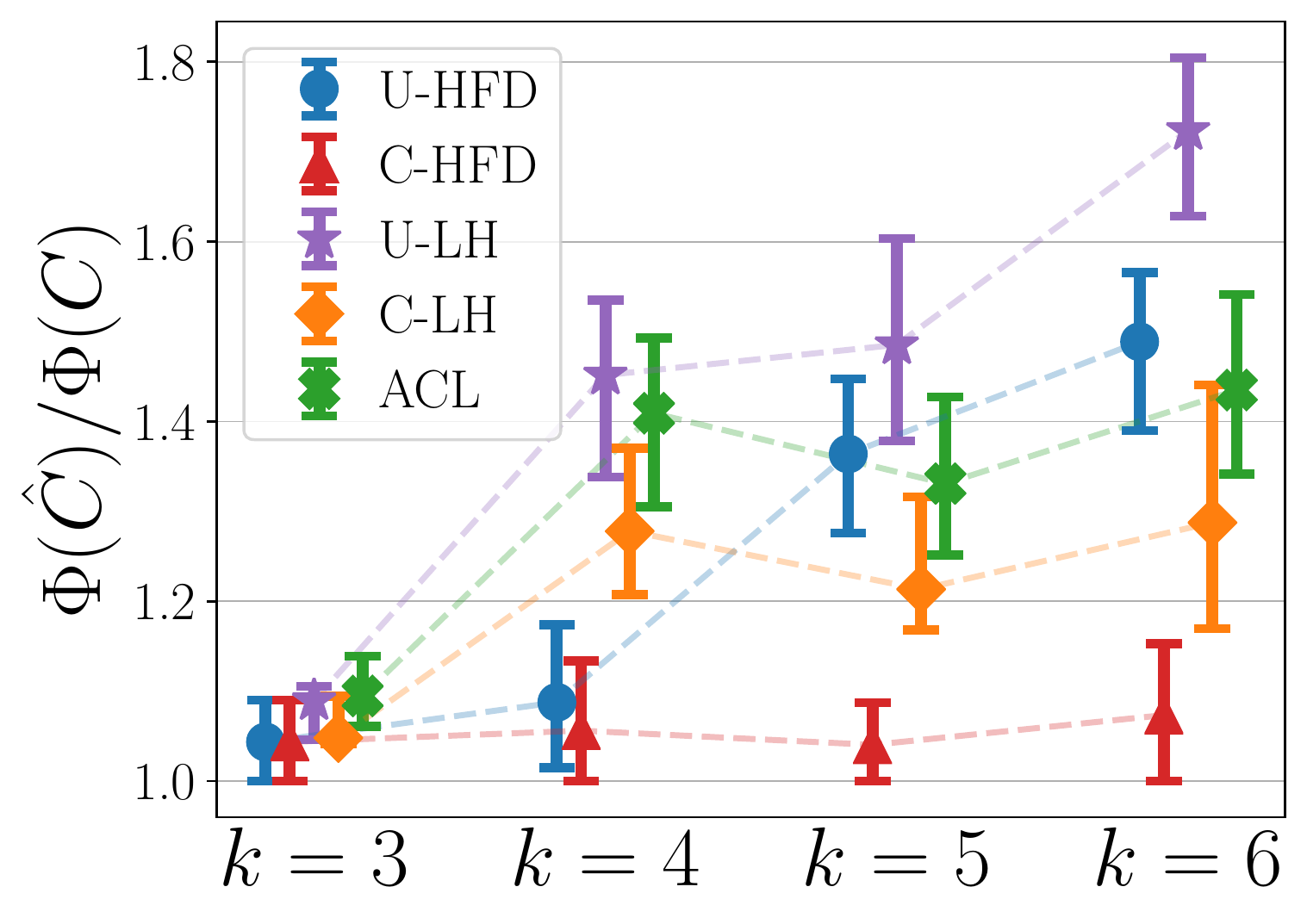}
	\end{subfigure}%
	\begin{subfigure}{0.45\textwidth}
		\centering
		\includegraphics[width=.95\textwidth]{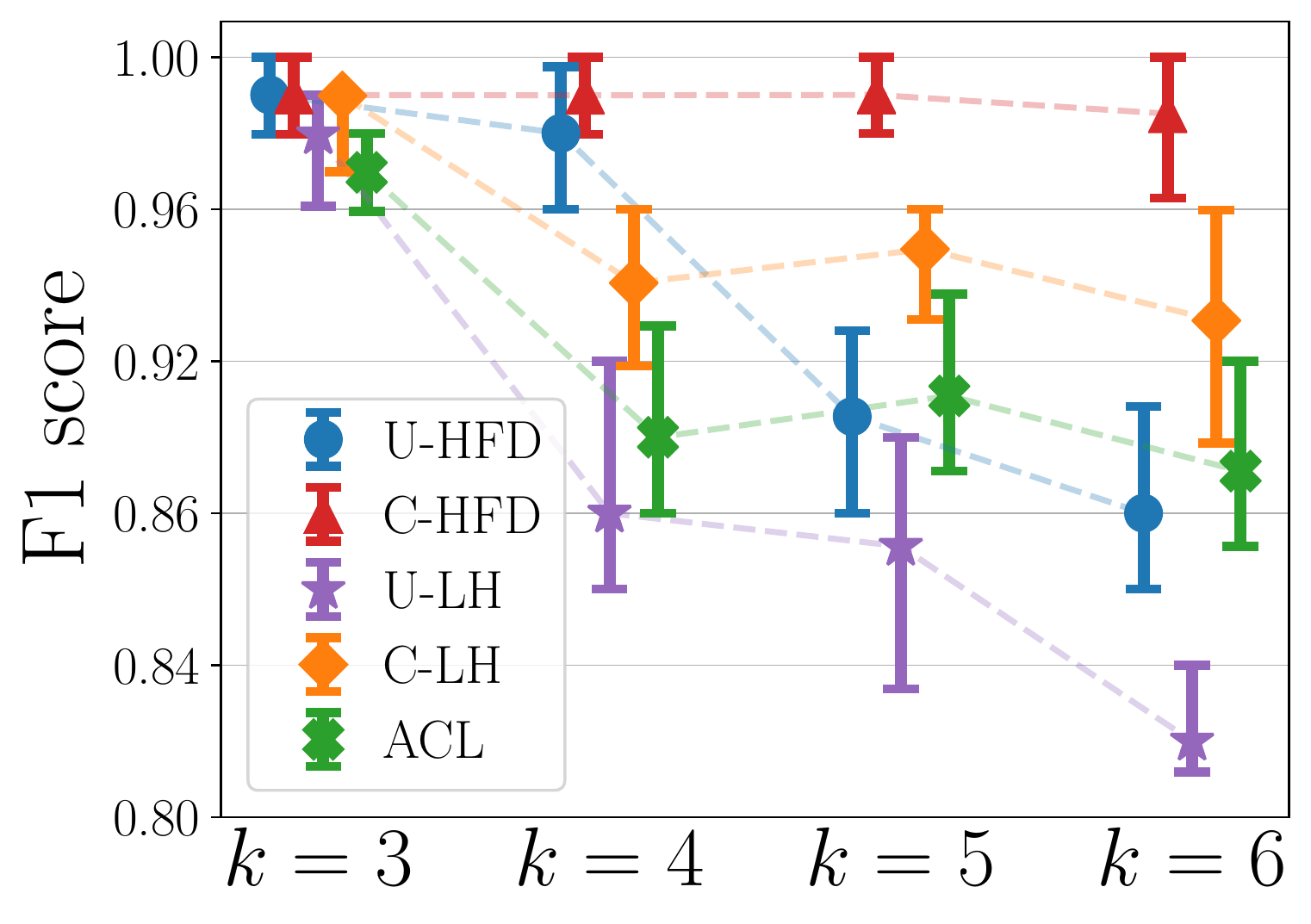}
	\end{subfigure}
	\caption{Conductance ratio and F1 on $k$-uniform hypergraphs}
	\label{fig:edgesize}
\end{figure}

\subsection{Experiments using real-world data}\label{sec:experiments-real}

We conduct extensive experiments using real-world data. First, we show that HFD has superior local clustering performances than existing methods for both unit and cardinality-based cut-costs. Then, we show that general submodular cut-cost (recall that HFD is the only method that applies to this setting) can be necessary for capturing complex high-order relations in the data, improving F1 scores by up to 20\% for local clustering and providing the only meaningful result for node ranking. In the appendix we show additional local clustering experiments on two additional datasets, where our method improves F1 scores by 8\% on average for 13 different target clusters.

{\bf Datasets.} We provide basic information about the datasets used in our experiments.\ Complete descriptions are provided in the appendix.\\
\underline{\it Amazon-reviews}~($|V|$ = 2,268,264, $|E|$ = 4,285,363)~\cite{ni2019justifying,VBK20}.\ In this hypergraph each node represents a product.\ A set of products are connected by a hyperedge if they are reviewed by the same person.\ We use product category labels as ground truth cluster identities.\ We consider all clusters of less than 10,000 nodes.\\
\underline{\it  Trivago-clicks}~($|V|$ = 172,738, $|E|$ = 233,202)~\cite{CVB2021}. The nodes in this hypergraph are accommodations/hotels.\ A set of nodes are connected by a hyperedge if a user performed ``click-out'' action during the same browsing session.\ We use geographical locations as ground truth cluster identities.\ There are 160 such clusters, and we filter them using cluster size and conductance.\\
\underline{\it Florida Bay food network}~($|V|$ = 128, $|E|$ = 141,233)~\cite{LM17}.\ Nodes in this hypergraph correspond to different species or organisms that live in the Bay, and hyperedges correspond to transformed motifs (Figure~\ref{fig:hyper-graph}) of the original dataset.\ Each species is labelled according its role in the food chain: producers, low-level consumers, high-level consumers.

{\bf Methods and parameters.} We compare HFD with LH and ACL.\footnote{We also tried a flow-improve method for hypergraphs~\cite{VBK20}, but the method was very slow in our experiments, so we only used it for small datasets. See appendix for results. The flow-improve method did not improve the performance of existing methods, therefore, we omitted it from comparisons on larger datasets.}\ There is a heuristic nonlinear variant of LH which is shown to outperform linear LH in some cases~\cite{LVHLG20}.\ Therefore we also compare with the same nonlinear variant considered in \cite{LVHLG20}.\ We denote the linear and nonlinear versions by LH-2.0 and LH-1.4, respectively.\ We set $\sigma=0.0001$ for HFD and we set the parameters for LH-2.0, LH-1.4 and ACL as suggested by the authors~\cite{LVHLG20}. More details on parameter choices appear in the appendix.\ We prefix methods that use unit and cardinality-based cut-costs by U- and C-, respectively.\

{\bf Experiments for unit and cardinality cut-costs.} For each target cluster in Amazon-reviews and Trivago-clicks, we run the methods multiple times, each time we use a different node as the singe seed node.\footnote{We show additional results using seed sets of more than one node in the appendix.} We report the median F1 scores of the output clusters in Table~\ref{tab:amazon-results} and Table~\ref{tab:trivago-results}. For Amazon-reviews, we only compare the unit cut-cost because it is both shown in \cite{LVHLG20} and verified by our experiments that unit cut-cost is more suitable for this dataset. Observe that U-HFD obtains the highest F1 scores for nearly all clusters. In particular, U-HFD significantly outperforms other methods for clusters 12, 18, 24, where we see an increase in F1 score by up to 52\%. For Trivago-clicks, C-HFD has the best performance for all but one clusters. Among the rest of all other methods, U-HFD has the second highest F1 scores for nearly all clusters. Moreover, observe that for each method (i.e., HFD, LH-2.0, LH-1.4), cardinality cut-cost leads to higher F1 than its unit cut-cost counterpart.

\begin{table}[ht!]
\caption{F1 results for Amazon-reviews network}\label{tab:amazon-results}
  \centering
  \small
  \setlength\tabcolsep{2.5pt}
  \begin{tabular}{lccccccccc}
    \toprule
    Method & 1 & 2  & 3 & 12 & 15 & 17 & 18 & 24 & 25 \\
    \midrule
    U-HFD & \bf 0.45 & \bf 0.09 & \bf 0.65 & \bf 0.92 & 0.04 & \bf 0.10 & \bf 0.80 & \bf  0.81 & \bf 0.09 \\
    U-LH-2.0 & 0.23 & 0.07 & 0.23 & 0.29 & \bf 0.05 & 0.06 & 0.21 & 0.28 & 0.05 \\
    U-LH-1.4 & 0.23 & \bf 0.09 & 0.35 & 0.40 & 0.00 & 0.07 & 0.31 & 0.35 & 0.06 \\
    ACL &  0.23 & 0.07 & 0.22 & 0.25 & 0.04 & 0.05 & 0.17 & 0.20 & 0.04  \\
    \bottomrule
  \end{tabular}
\end{table}

\begin{table}[ht!]
\caption{F1 results for Trivago-clicks network}\label{tab:trivago-results}
  \centering
  \setlength\tabcolsep{2.5pt}
  \small
  \begin{tabular}{lcccccccccc}
  \toprule
    Method & KOR & ISL & PRI & UA-43 & VNM & HKG & MLT & GTM & UKR & EST\\
    \midrule
    U-HFD 		& 0.75 	& \bf 0.99 	& 0.89 	& 0.85 	& 0.28 	& 0.82 	& \bf 0.98 	& 0.94 	& 0.60 	& \bf 0.94 \\
    C-HFD		& \bf 0.76 	& \bf 0.99 	& \bf 0.95 	& \bf 0.94 	& \bf 0.32 & 0.80 	& \bf 0.98 	& \bf 0.97 	& \bf 0.68 	& \bf 0.94 \\
    U-LH-2.0 	& 0.70 	& 0.86 	& 0.79 	& 0.70 	& 0.24 	& 0.92 	& 0.88 	& 0.82 	& 0.50 	& 0.90 \\
    C-LH-2.0 	& 0.73 	& 0.90 	& 0.84 	& 0.78 	& 0.27 	& \bf 0.94 	& 0.96 	& 0.88 	& 0.51 	& 0.83 \\
    U-LH-1.4  	& 0.69 	& 0.84 	& 0.80 	& 0.75 	& 0.28 	& 0.87 	& 0.92 	& 0.83 	& 0.47 	& 0.90 \\
    C-LH-1.4 	& 0.71 	& 0.88 	& 0.84 	& 0.78 	& 0.27 	& 0.88 	& 0.93 	& 0.85 	& 0.50 	& 0.85 \\
    ACL 		& 0.65 	& 0.84 	& 0.75 	& 0.68 	& 0.23 	& 0.90 	& 0.83 	& 0.69 	& 0.50 	& 0.88 \\
   \bottomrule
\end{tabular}
\end{table} 

{\bf Experiments for general submodular cut-cost.} In order to understand the importance of specialized general submodular hypergraphs we study the node-ranking problem for the Florida Bay food network using hypergraph modelling shown in Figure~\ref{fig:hyper-graph}.\ We compare HFD using unit (U-HFD, $\gamma_1=\gamma_2=1$), cardinality-based (C-HFD, $\gamma_1=1/2$ and $\gamma_2=1$) and submodular (S-HFD, $\gamma_1=1/2$ and $\gamma_2=0$) cut-costs.\ Our goal is to search the most similar species of a queried species based on the food-network structure.\ Table~\ref{tab:fw-results} shows that S-HFD provides the only meaningful node ranking results.\ Intuitively, when $\gamma_2=0$, separating the preys $v_1, v_2$ from the predators $v_3, v_4$ incurs 0 cost.\ This encourages S-HFD to diffuse mass among preys or predators only and not to cross from a predator to a prey or vice versa.\ As a result, similar species receive similar amount of mass and thus are ranked similarly.\ In the local clustering setting, Table~\ref{tab:fw-results} compares HFD using different cut-costs.\ 
By exploiting specialized higher-order relations, S-HFD further improves F1 scores by up to 20\% over U-HFD and C-HFD.\ This is not surprising, given the poor node-ranking results of other cut-costs.\ In the appendix we show another application of submodular cut-cost for node-ranking in an international oil trade network.

\begin{table}[ht!]
\caption{Node-ranking and local clustering in Florida Bay food network using different cut-costs}\label{tab:fw-results}
  \centering
  \small
  \setlength\tabcolsep{2.5pt}
  \begin{tabular}{clllll}
    \toprule
    & \multicolumn{2}{c}{Top-2 node-ranking results} & \multicolumn{3}{c}{Clustering F1} \\
    \cmidrule(l{2pt}r{2pt}){2-3} \cmidrule(l{2pt}r{2pt}){4-6}
    Method & Query: Raptors &  Query: Gray Snapper & Prod. & Low & High \\
    \midrule
    U-HFD & Epiphytic Gastropods, Detriti. Gastropods & Meiofauna, Epiphytic Gastropods & \bf 0.69 & 0.47 & 0.64 \\
    C-HFD & Epiphytic Gastropods, Detriti. Gastropods & Meiofauna, Epiphytic Gastropods & 0.67 & 0.47 &  0.64 \\
    S-HFD & Gruiformes, Small Shorebirds & Snook, Mackerel & \bf 0.69 & \bf 0.62 & \bf 0.84 \\
    \bottomrule
  \end{tabular}
\end{table}

\bibliography{references}
\bibliographystyle{plain}

\clearpage

\appendix 

\begin{center}
\huge Appendices for: Local Hyper-Flow Diffusion
\end{center}

\vspace{10pt}

\numberwithin{equation}{section}
\numberwithin{algorithm}{section}
\numberwithin{table}{section}
\numberwithin{figure}{section}
\numberwithin{claim}{section}
\numberwithin{assumption}{section}

Outline of the Appendix:

\vspace{5pt}

\begin{itemize}
	\item Appendix~\ref{sec:A01} contains supplementary material to Section 3 and Section 4 of the paper:

	\begin{itemize}
		\item mathematical derivation of the dual diffusion problem;
		\item proofs of Theorem 1 and Lemma 2.
	\end{itemize}
           
	\vspace{5pt}
        
	\item Appendix~\ref{sec:A02} contains supplementary material to Section 5 of the paper:
        
	\begin{itemize}
		\item proof of Lemma 3;	
		\item convergence properties of Algorithm 1;
		\item specialized algorithms for alternating minimization sub-problems of Algorithm 1.
	\end{itemize}
        
	\vspace{5pt}
        
	\item Appendix~\ref{sec:A03} contains supplementary material to Section 6 of the paper:
        
	\begin{itemize}
		\item additional synthetic experiments using $k$-uniform hypergraph stochastic block model;
		\item complete information about the real datasets considered in Section 6 of the paper;
		\item experiments for local clustering using seed sets that contain more than one node;
		\item experiments using 3 additional real datasets that are not discussed in the main paper;
		\item parameter settings and implementation details.
	\end{itemize}
	
\end{itemize}

\section{Approximation guarantee for local hypergraph clustering}\label{sec:A01}

In this section we prove a generalized and stronger version of Theorem 1 in the main paper, where the primal and dual diffusion problems are penalized by $\ell_p$-norm and $\ell_q$-norm, respectively, where $p\ge2$ and $1/p + 1/q = 1$.\ Moreover, we consider a generic hypergraph $H = (V,E,\WW)$ with general submodular weights $\WW = \{w_e,\vartheta_e\}_{e \in E}$ for any nonzero $\vartheta_e := \max_{S \subseteq e} w_e(S)$.\ {\em All claims in the main paper are therefore immediate special cases when $p = q = 2$ and $\vartheta_e = 1$ for all $e \in E$.}

Unless otherwise stated, we use the same notation as in the main paper.\ We generalize the definition of the degree of a node $v \in V$ as
\[
	d_v := \sum_{e \in E : v \in e} \vartheta_e.
\]
Note that when $\vartheta_e = 1$ for all $e$, the above definition reduces to $d_v = |\{e \in E : v \in e\}|$, which is the number of hyperedges to which $v$ belongs to.

Given $H = (V,E,\WW)$ where $\WW = \{w_e,\vartheta_e\}_{e \in E}$, $p \ge 2$, and a hyperparameter $\sigma \ge 0$, our primal Hyper-Flow Diffusion (HFD) problem is written as
\begin{equation}
\label{eq:primalLp}
\begin{split}
	\min_{\phi\in\RR^{|E|}_+,z\in\RR^{|V|}_+}~& \frac{1}{p}\sum_{e \in E}\vartheta_e \phi_e^p + \frac{\sigma}{p}\sum_{v \in V}d_vz_v^p\\
	\mbox{s.t.} \hspace{9mm} &\Delta - \sum_{e \in E}\vartheta_e r_e \le d + \sigma D z\\
	&r_e \in \phi_e B_e, ~\forall e \in E
\end{split}
\end{equation}
where
\[
	B_e := \{\rho_e \in \RR^{|V|} ~|~ \rho_e(S) \le w_e(S), \forall S \subseteq V, ~\mbox{and}~ \rho_e(V) = w_e(V)\}
\]
is the base polytope of $w_e$. The vector $m = \Delta - \sum_{e \in E}\vartheta_e r_e$ gives the net amount of mass after routing. Note that we multiply $r_e$ by $\vartheta_e$ because we have normalized $w_e$ by $\vartheta_e$ in its definition. 

\begin{lemma}
\label{lem:dual}
The following optimization problem is dual to \eqref{eq:primalLp}:
\begin{equation}
\label{eq:dualLq}
	\max_{x \in \RR^{|V|}_+}~(\Delta - d)^Tx - \frac{1}{q}\sum_{e \in E} \vartheta_e f_e(x)^q - \frac{\sigma}{q}\sum_{v \in V}d_v x_v^q
\end{equation}
where $f_e(x) := \max_{\rho_e \in B_e} \rho_e^Tx$ is the support function of base polytope $B_e$.
\end{lemma}

\begin{proof}
Using convex conjugates, for $x \in \RR^{|V|}_+$, we have
\begin{subequations}
\label{eq:conjugate}
\begin{align}
	\frac{1}{q}f_e(x)^q &= \max_{\phi_e \ge 0}~\phi_e f_e(x) - \frac{1}{p}\phi_e^p, ~\forall e \in E, \label{eq:conjugate-a}\\
	\frac{1}{q}x_v^q &= \max_{z_v \ge 0}~z_vx_v - \frac{1}{p}z_v^p, ~\forall v \in V.
\end{align}
\end{subequations}
Apply the definition of $f_e(x)$, we can write \eqref{eq:conjugate-a} as
\[
\frac{1}{q}f_e(x)^q = \max_{\phi_e\ge0}~\phi_e f_e(x) - \frac{1}{p}\phi_e^p \ = \max_{\phi_e\ge0, r_e\in\phi_eB_e} r_e^Tx - \frac{1}{p}\phi_e^p.
\]
Therefore,
\[
\begin{split}
	&~\max_{x \in \RR^{|V|}_+}~(\Delta-d)^Tx - \frac{1}{q}\sum_{e \in E}\vartheta_e f_e(x)^q - \frac{\sigma}{q}\sum_{v \in V}d_vx_v^q\\
	 =&~\max_{x \in \RR^{|V|}_+}~(\Delta-d)^Tx - \sum_{e \in E}\vartheta_e \left(\max_{\phi_e\ge0, r_e\in\phi_eB_e}~r_e^Tx - \frac{1}{p}\phi_e^p \right) - \sigma\sum_{v \in V} d_v\left(\max_{z_v \ge 0}~z_vx_v - \frac{1}{p}z_v^p\right)\\
	 =&~\max_{x \in \RR^{|V|}_+}~(\Delta-d)^Tx + \min_{\substack{\phi\in\RR^{|E|}_+\\r_e\in\phi_eB_e,\forall e\in E}}\sum_{e \in E} \left(\frac{1}{p}\vartheta_e\phi_e^p - \vartheta_er_e^Tx\right) + \min_{z\in\RR^{|V|}_+} \sigma \sum_{v \in V}\left(\frac{1}{p}d_vz_v^p - d_vz_vx_v\right)\\
	 =&~\min_{\substack{\phi\in\RR^{|E|}_+,z\in\RR^{|V|}_+\\r_e\in\phi_eB_e,\forall e\in E}}\frac{1}{p}\sum_{e \in E}\vartheta_e \phi_e^p + \frac{\sigma}{p}\sum_{v \in V} d_vz_v^p + \max_{x\in\RR^{|V|}_+} \left((\Delta-d)^Tx - \sum_{e \in E}\vartheta_er_e^Tx - \sigma \sum_{v \in V} d_vz_vx_v\right)\\
	 =&~\min_{\substack{\phi\in\RR^{|E|}_+,z\in\RR^{|V|}_+\\r_e\in\phi_eB_e,\forall e\in E}}\frac{1}{p}\sum_{e \in E}\vartheta_e \phi_e^p + \frac{\sigma}{p}\sum_{v \in V} d_vz_v^p \quad \mbox{s.t.}\quad \Delta - d - \sum_{e \in E}\vartheta_e r_e - \sigma Dz \le 0.
\end{split}
\]
In the above derivations, we may exchange the order of minimization and maximization and arrive at the second last equality, due to Proposition 2.2, Chapter VI, in~\cite{ET1999}. The last equality follows from
\[
 \max_{x\in\RR^{|V|}_+} \bigg((\Delta-d)^Tx - \sum_{e \in E}\vartheta_er_e^Tx - \sigma \sum_{v \in V} d_vz_vx_v\bigg) = \left\{\begin{array}{ll} 0, & \mbox{if}~\Delta - d - \sum\limits_{e \in E}\vartheta_e r_e - \sigma Dz \le 0, \\ +\infty, & \mbox{otherwise}. \end{array} \right.
\]
\end{proof}

{\bf Notation.} For the rest of this section, we reserve the notation $(\hat{\phi},\hat{z})$ and $\hat{x}$ for optimal solutions of \eqref{eq:primalLp} and \eqref{eq:dualLq} respectively. If $\sigma = 0$, we simply treat $\hat{z} = 0$.

The next lemma relates primal and dual optimal solutions. We make frequent use of this relation throughout our discussion.

\begin{lemma}
\label{lem:conjugate}
We have that $\hat{\phi}_e^p = f_e(\hat{x})^q$ for all $e \in E$. Moreover, if $\sigma > 0$, then $\hat{z}_v^p = \hat{x}_v^q$ for all $v \in V$.
\end{lemma}

\begin{proof}
Given $\hat{x}$ an optimal solution to \eqref{eq:dualLq}, it follows directly from \eqref{eq:conjugate} and strong duality that $(\hat{\phi}, \hat{z})$ must satisfy, for each $e \in E$ and $v \in V$,
\[
\hat{\phi}_e = f(\hat{x})^{q-1} = \argmax_{\phi_e\ge0}~\phi_e f_e(\hat{x}) - \frac{1}{p}\phi_e^p \quad \mbox{and} \quad
\hat{z}_v = \hat{x}_v^{q-1} = \argmax_{z_v \ge 0}~z_v\hat{x}_v - \frac{1}{p}z_v^p.
\]
\end{proof}

{\bf Diffusion setup.} Recall that we pick a scalar $\delta$ and set the source $\Delta$ as
\begin{equation}
\label{eq:source_a}
	\Delta_v = \left\{
	\begin{array}{ll}
	\delta d_v, & \mbox{if}~ v \in S,\\
	0, & \mbox{otherwise}.
	\end{array}
	\right.
\end{equation}

For convenience we restate the assumptions below.

\begin{assumption}
\label{assum:overlap_a}
$\vol(S \cap C) \ge \alpha \vol(C)$ and $\vol(S \cap C) \ge \beta \vol(S)$ for some $\alpha,\beta \in (0,1]$.
\end{assumption}
\begin{assumption}
\label{assum:delta_a}
The source mass $\Delta$ as specified in \eqref{eq:source_a} satisfies $\delta = 3/\alpha$, which gives $\Delta(C) \ge 3\vol(C)$.
\end{assumption}
\begin{assumption}
\label{assum:sigma_a}
$\sigma$ satisfies $0 \le \sigma \le \beta\Phi(C)/3$.
\end{assumption}

\subsection{Technical lemmas}

In this subsection we state and prove some technical lemmas that will be used for the main proof in the next subsection. 

The following lemma characterizes the maximizers of the support function for a base polytope.

\begin{lemma}[Proposition 4.2 in \cite{Bach2011a}]
\label{lem:maximizers}
Let $w$ be a submodular function such that $w(\emptyset) = 0$. Let $x \in \RR^{|V|}$, with unique values $a_1 > \cdots > a_m$, taken at sets $A_1, \ldots, A_m$ (i.e., $V = A_1 \cup \cdots \cup A_m$ and $\forall i \in \{1,\ldots,m\}$, $\forall v \in A_i$, $x_v = a_v$). Let $B$ be the associated base polytope. Then $\rho \in B$ is optimal for $\max_{\rho \in B} \rho^Tx$ if and only if for all $i = 1,\ldots,m$, $\rho(A_1 \cup \cdots \cup A_i) = w(A_1 \cup \cdots \cup A_i)$.
\end{lemma}

Recall that $(\hat{\phi},\hat{z})$ and $\hat{x}$ denote the optimal solutions of \eqref{eq:primalLp} and \eqref{eq:dualLq} respectively. We start with a lemma on the locality of the optimal solutions.

\begin{lemma}[Lemma 2 in the main paper]
\label{lem:support_a}
We have
\[
	\sum_{e \in \supp(\hat{\phi})} \vartheta_e  ~=~ \vol(\supp(\hat{x})) ~\le~ \|\Delta\|_1.
\]
Moreover, if $\sigma > 0$, then $\vol(\supp(\hat{z})) = \vol(\supp(\hat{x}))$.
\end{lemma}

\begin{proof}
To see the first inequality, note that if $\hat{x}_v = 0$ for every $v \in e$ for some $e$, then $f_e(\hat{x}) = 0$. By Lemma~\ref{lem:conjugate}, this means $\hat{\phi}_e = 0$. Thus, $\hat{\phi}_e \neq 0$ only if there is some $v \in e$ such that $\hat{x}_v \neq 0$. Therefore, we have that
\[
	\sum_{e \in \supp(\hat{\phi})}\vartheta_e \le \sum_{v \in \supp(\hat{x})} \sum_{e \in E : v \in e} \vartheta_e =  \sum_{v \in \supp(\hat{x})} d_v = \vol(\supp(\hat{x})).
\]
To see the last inequality, note that, by the first order optimality condition of \eqref{eq:dualLq}, if $\hat{x}_v \neq 0$ then we must have
\begin{equation}
\label{eq:dualstationarity}
	\Delta_v - d_v = \sum_{e \in E}\vartheta_ef_e(\hat{x})^{q-1}\hat{\rho}_{e,v} + \sigma d_v\hat{x}_v^{q-1}, ~~ \mbox{for some} ~~ \hat{\rho}_e \in \partial f_e(\hat{x}) = \argmax_{\rho_e \in B_e} \rho_e^T\hat{x}.
\end{equation}
Denote $N := \supp(\hat{x})$ and $E[N] := \{e \in E ~|~ v \in N ~\mbox{for all}~ v \in e\}$. Note that $E[N] \cap \partial N = \emptyset$, and $E[N] \cup \partial N = \{e\in E ~|~ v \in N ~\mbox{for some}~v \in e\}$, that is,  $E[N] \cup \partial N$ contain all hyperedges that are incident to some node in $N$. Moreover, we have that for any $\hat{\rho}_e \in \argmax_{\rho_e \in B_e} \rho_e^T\hat{x}$,
\[
\sum_{v \in N} \hat{\rho}_{e,v} = \hat{\rho}_e(N) = \left\{
\begin{array}{ll}
w_e(N), & \mbox{if}~ e \in \partial N,\\
0, & \mbox{if}~ e \in E[N],
\end{array}
\right.
\]
where $\hat{\rho}_e(N) = w_e(N)$ for $e \in \partial N$ follows from Lemma~\ref{lem:maximizers}, since $\hat{x}_v > 0 $ for $v \in N$ and $\hat{x}_v = 0$ for $v \not\in N$. The equality $\hat{\rho}_e(N) = 0$ for $e \in E[N]$ follows from $\hat{\rho}_e(N) = \hat{\rho}_e(e) = 0$ because $e \subseteq N$ and $\hat{\rho}_{e,v} = 0$ for all $v \not\in e$.

Taking sums over $v \in N$ on both sides of equation~\eqref{eq:dualstationarity} we obtain
\[
\begin{split}
\Delta(N) - \vol(N)~
&=~\sum_{v \in N}\sum_{e \in E} \vartheta_e f_e(\hat{x})^{q-1} \hat{\rho}_{e,v} + \sum_{v \in N} \sigma d_v\hat{x}_v^{q-1}\\
&=~\sum_{v \in N}\sum_{e \in E[N]} \vartheta_e f_e(\hat{x})^{q-1} \hat{\rho}_{e,v} + \sum_{v \in N}\sum_{e \in \partial N} \vartheta_e f_e(\hat{x})^{q-1} \hat{\rho}_{e,v} + \sum_{v \in N} \sigma d_v\hat{x}_v^{q-1}\\
&=~\sum_{e \in E[N]}\vartheta_e f_e(\hat{x})^{q-1} \sum_{v \in N} \hat{\rho}_{e,v} + \sum_{e \in \partial N}\vartheta_e f_e(\hat{x})^{q-1} \sum_{v \in N} \hat{\rho}_{e,v} + \sum_{v \in N} \sigma d_v\hat{x}_v^{q-1}\\
&=~ 0 + \sum_{e \in \partial N}\vartheta_e f_e(\hat{x})^{q-1}w_e(N) + \sum_{v \in N} \sigma d_v\hat{x}_v^{q-1}\\
&\ge~ 0.
\end{split}
\]
The second equality follows from $\hat{\rho}_{e,v} = 0$ for all $v \not\in e$. This proves $\vol(\supp(\hat{x})) \le \Delta(\supp(\hat{x})) \le \|\Delta\|_1$. 

Finally, if $\sigma > 0$, then $\vol(\supp(\hat{z})) = \vol(\supp(\hat{x}))$ follows from Lemma~\ref{lem:conjugate} that $\hat{z}^p = \hat{x}^q$ for all $v \in V$. 
\end{proof}

The following inequality is a special case of H\"{o}lder's inequality for degree-weighted norms. It will become useful later.
\begin{lemma}
\label{lem:holder}
For $x \in \RR^{|V|}$ and $p > 1$ we have that
\[
	\Bigg(\sum_{v \in V} d_v |x_v|\Bigg)^p \le \vol(\supp(x))^{p-1} \sum_{v \in V} d_v |x_v|^p.
\]
\end{lemma}

\begin{proof}
Let $q = p/(p-1)$. Apply H\"{o}lder's inequality we have
\[
\begin{split}
\sum_{v \in V} d_v|x_v| = \sum_{v \in \supp(x)} |d_v^{1/q}| |d_v^{1/p}x_v| 
&\le \Bigg(\sum_{v \in \supp(x)} d_v\Bigg)^{1/q} \Bigg(\sum_{v \in \supp(x)}d_v |x_v|^p\Bigg)^{1/p} \\
&= \vol(\supp(x))^{1/q} \Bigg(\sum_{v \in V}d_v |x_v|^p\Bigg)^{1/p}.
\end{split}
\]
\end{proof}

\begin{lemma}[Lemma I.2 in \cite{LM18}]
\label{lem:rayleigh}
For any $x \in \RR^{|V|}_+ \setminus \{0\}$ and $q \ge 1$, one has
\[
    \frac{\sum_{e \in E}\vartheta_ef_e(x)^q}{\sum_{v \in V}d_vx_v^q} \ge \frac{c(x)^q}{q^q},
\]
where
\[
    c(x) := \min_{h \ge 0}\frac{\vol(\partial\{v\in V | x_v^q > h\})}{\vol(\{v\in V | x_v^q > h\})} = \min_{h \ge 0}\frac{\vol(\partial\{v\in V | x_v > h\})}{\vol(\{v\in V | x_v > h\})}.
\]
\end{lemma}

Recall that the objective function of our primal diffusion problem~\eqref{eq:primalLp} consists of two parts. The first part is $\sum_{e \in E}\vartheta_e \phi_e^p$ and it penalizes the cost of flow routing, the second part is $\sum_{v \in V} d_v z_v^p$ and it penalizes the cost of excess mass. An immediate consequence of Lemma~\ref{lem:rayleigh} is the inequality in Lemma~\ref{lem:lower} that relates the cost of optimal flow routing $\sum_{e \in E}\vartheta_e \hat{\phi}_e^p$ and the cost of excess mass $\sum_{v \in V} d_v\hat{z}_v^p$ at optimality.

For $h>0$, recall that the sweep sets are defined as $S_h := \{v \in V | \hat{x}_v \ge h\}$. 

Let $\hat{h} \in \argmin_{h > 0} \Phi(S_h)$ and denote $\hat{S} = S_{\hat{h}}$. That is, $\hat{S} = S_h$ for some $h > 0$ and $\Phi(\hat{S}) \le \Phi(S_h)$ for all $h > 0$. 

\begin{lemma}
\label{lem:lower}
For $p > 1$ and $q = p/(p-1)$ we have that
\[
	 \sum_{e \in E}\vartheta_e \hat{\phi}_e^p \ge \left(\frac{\Phi(\hat{S})}{q}\right)^q\sum_{v \in V} d_v \hat{z}_v^p.
\]
\end{lemma}

\begin{proof}
By Lemma~\ref{lem:conjugate},
\[
	\sum_{e \in E}\vartheta_e \hat\phi_e^p = \sum_{e \in E}\vartheta_e f_e(\hat{x})^q \quad\mbox{and}\quad \sum_{v \in V} d_v \hat{z}_v^p = \sum_{v \in V} d_v \hat{x}_v^q,
\]
and the result follows from applying Lemma~\ref{lem:rayleigh}.
\end{proof}

Given a vector $a \in \RR^{|V|}$ and a set $S \subseteq V$, recall that we write $a(S) = \sum_{v \in S} a_v$. This actually defines a modular set-function $a$ taking input on subsets of $V$. The Lov\'{a}sz extension of modular function $a$ is simply $f(x) = a^Tx$~\cite{Bach2011a}. Since all modular functions are also submodular, we arrive at the following lemma that follows from a classical property of the Choquet integral/Lov\'{a}sz extension.

\begin{lemma}
\label{lem:lovasz}
We have that  
\[
\Delta^T\hat{x} = \int_{h=0}^{+\infty} \Delta(S_h) dh, \quad
d^T\hat{x} = \int_{h=0}^{+\infty} \vol(S_h) dh, \quad
f_e(\hat{x}) = \int_{h=0}^{+\infty} w_e(S_h) dh.
\]
\end{lemma}

\begin{proof}
Recall that, by definition, $\vol(S) = d(S)$ where $d$ is the degree vector. $\Delta$ and $d$ are modular functions on $2^V$ and $w_e$ is a submodular function on $2^V$. The Lov\'{a}sz extension of $\Delta$ and $d$ are $\Delta^Tx$ and $d^Tx$, respectively. The Lov\'{a}sz extension of $w_e$ is $f_e(x)$. The results then follow immediately from representing the Lov\'{a}sz extensions using Choquet integrals. See, e.g., Proposition 3.1 in \cite{Bach2011a}.
\end{proof}

\subsection{Proof of Theorem 1 in the main paper}

We restate the theorem below with respect to the general formulations \eqref{eq:primalLp} and \eqref{eq:dualLq} for any $p \ge 2$ and $q = p/(p-1)$.

Let us recall that the sweep sets are defined as $S_h := \{v \in V | \hat{x}_v \ge h\}$. 

\begin{theorem}
\label{thm:conductance_lp}
Under Assumptions~\ref{assum:overlap_a},~\ref{assum:delta_a},~\ref{assum:sigma_a}, for some $h>0$ we have that
\[
	\Phi(S_h) \le O\left(\frac{\Phi(C)^{1/q}}{\alpha\beta}\right).
\]
\end{theorem}

Recall that $\hat{S}$ is such that $\hat{S} = S_h$ for some $h > 0$ and $\Phi(\hat{S}) \le \Phi(S_h)$ for all $h > 0$. We will assume without loss of generality that $\Phi(C) \le (\Phi(\hat{S})/q)^q$, as otherwise $\Phi(\hat{S}) < q\Phi(C)^{1/q}$ and the statement in Theorem~\ref{thm:conductance_lp} already holds.

Denote $\hat{\nu} := \sum_{e \in E}\vartheta_e\hat{\phi}_e^p$, the cost of optimal flow routing. The following claim states that $\hat{\nu}$ must be large.

\begin{claim}
\label{claim:flows}
$\hat\nu \ge \vol(C)^p/\vol(\partial C)^{p-1}.$
\end{claim}

\begin{proof}
The proof of this claim follows from a case analysis on the total amount of excess mass $\sigma\sum_{v \in V}d_v\hat{z}_v$ at optimality. Intuitively, if the excess is small, then naturally there must be a large amount of flow in order to satisfy the primal constraint; if the excess is large, then Lemma~\ref{lem:lower} and Lemma~\ref{lem:holder} guarantee that flow is also large. We give details below.

Suppose that $\sigma \sum_{v \in V}d_v\hat{z}_v < \vol(C)$. Note that this also includes the case where $\sigma = 0$. By Assumption~\ref{assum:delta_a} there is at least $\Delta(C) \ge 3\vol(C)$ amount of source mass trapped in $C$ at the beginning. Moreover, the primal constraint enforces the nodes in $C$ can settle at most $\sum_{v\in C}(d_v + \sigma d_v \hat{z}_v) \le \vol(C) + \sum_{v \in V}\sigma d_v\hat{z}_v  < 2\vol(C)$ amount of mass. Therefore, the remaining at least $\vol(C)$ amount of mass needs to get out of $C$ using the hyperedges in $\partial C$. That is, the net amount of mass that moves from $C$ to $V\setminus C$ satisfies $\sum_{e \in \partial C}\vartheta_e \hat{r}_e(C) \ge \vol(C)$. We focus on the cost of $\hat{\phi}$ restricted to these hyperedges along. It is easy to see that
\begin{subequations}
\begin{align}
	\sum_{e \in \partial C} \vartheta_e \hat{\phi}_e^p ~
  	&\ge~\min_{\phi\in\RR^{|\partial C|}_+}~\sum_{e \in \partial C} \vartheta_e \phi_e^p ~~\mbox{subject to}~~\hat{r}_e \in \phi_eB_e, ~\forall e \in \partial C \label{eq:flowcut_a}\\
  	&\ge~\min_{\phi\in\RR^{|\partial C|}_+}~\sum_{e \in \partial C} \vartheta_e \phi_e^p ~~\mbox{subject to}~\sum_{e \in \partial C}\vartheta_e \hat{r}_e(C) \le \sum_{e \in \partial C} \vartheta_e \phi_e w_e(C) \label{eq:flowcut_b}\\
  	&\ge~\min_{\phi\in\RR^{|\partial C|}_+}~\sum_{e \in \partial C} \vartheta_e \phi_e^p ~~\mbox{subject to}~~\vol(C) \le \sum_{e \in \partial C} \vartheta_e \phi_e w_e(C). \label{eq:flowcut_c}
\end{align}
\end{subequations}
The first inequality follows because $\hat{\phi}$ restricted to $\partial C$ is a feasible solution in problem~\eqref{eq:flowcut_a}. The second inequality follows because $\hat{r}_e \in \phi_eB_e$ implies $\hat{r}_e(C) \le \phi_e w_e(C)$, therefore every feasible solution for \eqref{eq:flowcut_a} is also a feasible solution for \eqref{eq:flowcut_b}. The third inequality follows because $\vol(C) \le \sum_{e\in E}\vartheta_e \hat{r}_e(C)$. Let $\bar{\phi} \in \RR^{|\partial C|}_+$ be an optimal solution of problem \eqref{eq:flowcut_c}. The optimality condition of \eqref{eq:flowcut_c} is given by (we may assume the $p$ factor in the gradient of $\sum_{e \in \partial C} \vartheta_e \phi_e^p$ is absorbed into multipliers $\lambda$ and $\eta_e$)
\begin{equation}
\label{eq:flowcut_optcond}
\begin{split}
&\vartheta_e \phi_e^{p-1} - \lambda\vartheta_e w_e(C) - \eta_e = 0, ~ \forall e \in \partial C\\
& \phi_e \ge 0, ~ \eta_e \ge 0, ~ \phi_e\eta_e = 0, ~ \forall e \in \partial C\\
&\vol(C) \le \sum_{e \in \partial C}\vartheta_e \phi_e w_e(C)\\
& \lambda \ge 0, ~ \lambda\bigg(\vol(C)-\sum_{e \in \partial C}\vartheta_e \phi_e w_e(C)\bigg) = 0.
\end{split}
\end{equation}
If $\lambda = 0$, then the conditions in \eqref{eq:flowcut_optcond} imply that $\vartheta_e \phi_e^{p-1} = \eta_e$, but then by complimentary slackness we would obtain $\phi_e = \eta_e = 0$ for all $e \in \partial C$ which will violate feasibility. Therefore we must have $\lambda>0$, and consequently, we have that
\begin{equation}
\sum_{e \in \partial C} \vartheta_e \bar{\phi}_e w_e(C) = \vol(C). \label{eq:flowcond2}
\end{equation}
Moreover, the conditions in \eqref{eq:flowcut_optcond} imply that for $e \in \partial C$, $\bar{\phi}_e = 0$ if and only if $w_e(C) = 0$, and hence we have that
\begin{equation}
\label{eq:flowcond1}
\vartheta_e \bar{\phi}_e^{p-1} = \lambda \vartheta_e w_e(C), ~\forall e \in \partial C.
\end{equation}
Rearrange \eqref{eq:flowcond1} we get
\[
	\bar\phi_e w_e(C) = \lambda^{1/(p-1)}w_e(C)^{p/(p-1)}, ~\forall e \in \partial C.
\]
Substitute the above into \eqref{eq:flowcond2},
\[
	\vol(C) = \sum_{e \in \partial C} \vartheta_e \bar\phi_e w_e(C) = \sum_{e \in \partial C}\vartheta_e \lambda^{1/(p-1)} w_e(C)^{p/(p-1)},
\]
this gives
\[
	\lambda^{1/(p-1)} = \frac{\vol(C)}{\sum_{e\in\partial C}\vartheta_e w_e(C)^{p/(p-1)}}.
\]
Therefore, the solution $\bar\phi$ for \eqref{eq:flowcut_c} is give by
\[
	\bar\phi_e = \lambda^{1/(p-1)} w_e(C)^{1/(p-1)} = \frac{\vol(C)w_e(C)^{1/(p-1)}}{\sum_{e' \in \partial C}\vartheta_{e'}w_{e'}(C)^{p/(p-1)}}, \quad \forall e \in \partial C,
\]
and hence,
\[
\begin{split}
	\hat\nu = \sum_{e \in E} \vartheta_e \hat\phi_e^p
	\ge \sum_{e \in \partial C} \vartheta_e\hat\phi_e^p
	\ge \sum_{e \in \partial C} \vartheta_e\bar\phi_e^p
	&=\sum_{e \in \partial C} \vartheta_e \frac{\vol(C)^p w_e(C)^{p/(p-1)}}{\left(\sum_{e' \in \partial C}\vartheta_{e'}w_{e'}(C)^{p/(p-1)}\right)^p} \\
	&=\frac{\vol(C)^p \sum_{e \in \partial C} \vartheta_e w_e(C)^{p/(p-1)}}{\left(\sum_{e' \in \partial C}\vartheta_{e'}w_{e'}(C)^{p/(p-1)}\right)^p} \\
	&=\frac{\vol(C)^p}{\left(\sum_{e' \in \partial C}\vartheta_{e'}w_{e'}(C)^{p/(p-1)}\right)^{p-1}}\\
	&\ge\frac{\vol(C)^p}{\big(\sum_{e' \in \partial C}\vartheta_{e'} w_{e'}(C)\big)^{p-1}}
\end{split}
\]
where the last inequality follows because $w_e(C) \in [0,1]$ and $p\ge1$.

Suppose now that $\sigma \sum_{v \in V}d_v\hat{z}_v \ge \vol(C)$. Becase $\Phi(C) \le (\Phi(\hat{S})/q)^q$ (recall that we assumed this without loss of generality), by Assumption~\ref{assum:sigma_a}, we know that $\sigma < (\phi(\hat{S})/q)^q$. Therefore,
\[
\begin{split}
	\hat\nu = \sum_{e \in E} \vartheta_e \hat\phi_e^p
	&\stackrel{(i)}{\ge} \sigma \sum_{v \in V} d_v \hat{z}_v^p\\
	&\stackrel{(ii)}{\ge} \frac{\sigma \left(\sum_{v \in V} d_v \hat{z}_v\right)^p}{\vol(\supp(\hat{z}))^{p-1}}\\
	&\stackrel{(iii)}{\ge} \frac{\sigma^p \left(\sum_{v \in V} d_v \hat{z}_v\right)^p}{\sigma^{p-1}(3\vol(C)/\beta)^{p-1}}\\
	&\stackrel{(iv)}{\ge} \frac{\sigma^p \left(\sum_{v \in V} d_v \hat{z}_v\right)^p}{\vol(\partial C)^{p-1}}\\
	&\stackrel{(v)}{\ge} \frac{\vol(C)^p}{\vol(\partial C)^{p-1}}.
\end{split}
\]
$(i)$ is due to Lemma~\ref{lem:rayleigh}. $(ii)$ is due to Lemma~\ref{lem:holder}. $(iii)$ is due to Lemma~\ref{lem:support_a} that $\vol(\supp(\hat{z})) \le \|\Delta\|_1$ and Assumption~\ref{assum:delta_a} that $\|\Delta\|_1 \le 3\vol(c)/\beta$, so $\vol(\supp(\hat{z}))^{p-1} \le (3\vol(C)/\beta)^{p-1}$ for $p \ge 1$. $(iv)$ is due to Assumption~\ref{assum:sigma_a} that $\sigma \le \frac{\beta\vol(\partial C)}{3\vol(C)}$, so $(3\sigma\vol(C)/\beta)^{p-1} \le \vol(\partial C)^{p-1}$ for $p\ge1$. $(v)$ is due to the assumption that $\sigma \sum_{v \in V}d_v\hat{z}_v \ge \vol(C)$.
\end{proof}

To connect $\Phi(S_h)$ with $\Phi(C)$, we define the {\em length} of a hyperedge $e \in E$ as
\[
	\hat{l}(e) := \left\{
	\begin{array}{ll}
	\max(1/\vol(C)^{1/q}, f_e(\hat{x})/\hat{\nu}^{1/q}), & \mbox{if}~f_e(\hat{x}) > 0,\\
	0, & \mbox{otherwise}.
	\end{array}
	\right.
\]

The next claim follows from simple algebraic computations and the locality of solutions in Lemma~\ref{lem:support_a}.

\begin{claim}
\label{claim:top}
$\sum_{e \in E} \vartheta_e f_e(\hat{x})\hat{l}(e)^{q-1} \le 4\hat\nu^{1/q}/\beta$.
\end{claim}

\begin{proof}
For $e \in E$, define $l(e) := f_e(\hat{x})/\hat{\nu}^{1/q}$. Then $l(e) \le \hat{l}(e)$. Moreover,
\[
\sum_{e:l(e)<\hat{l}(e)} \vartheta_e \le \sum_{e \in \supp(\hat\phi)}\vartheta_e \le \vol(\supp(\hat{x})) \le \|\Delta\|_1 = \frac{3}{\alpha} \vol(S) \le \frac{3}{\beta}\vol(C).
\]
The first inequality follows from that $l(e) < \hat{l}(e)$ only if $l(e) \neq 0$, and by Lemma~\ref{lem:conjugate}, $l(e) \neq 0$ if and only if $\hat\phi_e \neq 0$. The second and the third inequalities are due to Lemma~\ref{lem:support_a}. The second to last equality follows from the diffusion setting \eqref{eq:source_a} and Assumption~\ref{assum:delta_a} that $\delta = 3/\alpha$. The last inequality follows from Assumption~\ref{assum:overlap_a}. Therefore,
\[
\begin{split}
	\sum_{e \in E}\vartheta_e f_e(\hat{x})\hat{l}(e)^{q-1} 
	& = \sum_{e:l(e)=\hat{l}(e)} \vartheta_e f_e(\hat{x}) \frac{f_e(\hat{x})^{q-1}}{\hat\nu^{(q-1)/q}} + \sum_{e:l(e)<\hat{l}(e)} \vartheta_e f_e(\hat{x})\frac{1}{\vol(C)^{(q-1)/q}}\\
	&\le \sum_{e:l(e)=\hat{l}(e)} \vartheta_e f_e(\hat{x}) \frac{f_e(\hat{x})^{q-1}}{\hat\nu^{(q-1)/q}} + \sum_{e:l(e)<\hat{l}(e)} \vartheta_e \frac{\hat\nu^{1/q}}{\vol(C)^{1/q}}\frac{1}{\vol(C)^{(q-1)/q}}\\
	&= \frac{1}{\hat\nu^{(q-1)/q}}\sum_{e:l(e)=\hat{l}(e)} \vartheta_e f_e(\hat{x})^q + \frac{\hat\nu^{1/q}}{\vol(C)}\sum_{e:l(e)<\hat{l}(e)} \vartheta_e\\
	&\le \frac{1}{\hat\nu^{(q-1)/q}}\sum_{e \in E} \vartheta_e f_e(\hat{x})^q + \frac{\hat\nu^{1/q}}{\vol(C)}\frac{3\vol(C)}{\beta}\\
	&= \frac{\hat\nu}{\hat\nu^{(q-1)/q}} + \frac{3\hat\nu^{1/q}}{\beta}\\
	&\le \frac{4\hat\nu^{1/q}}{\beta}
\end{split}
\]
where the last equality follows from Lemma~\ref{lem:conjugate} that $\hat\nu = \sum_{e\in E}\vartheta_e \hat\phi_e^p = \sum_{e \in E}\vartheta_e f_e(\hat{x})^q$.
\end{proof}

By the strong duality between \eqref{eq:primalLp} and \eqref{eq:dualLq}, we know that
\[
	(\Delta - d)^T\hat{x} - \frac{1}{q}\sum_{e \in E}\vartheta_e f_e(\hat{x})^q - \frac{\sigma}{q}\sum_{v \in V} d_v \hat{x}_v^q = \frac{1}{p}\sum_{e \in E}\vartheta_e \hat\phi_e^p + \frac{\sigma}{p}\sum_{v \in V}d_v\hat{z}_v^p.
\]
Hence, by Lemma~\ref{lem:conjugate}, we get
\[
	(\Delta - d)^T\hat{x} \ge \frac{1}{q}\sum_{e \in E}\vartheta_e f_e(\hat{x})^q + \frac{1}{p}\sum_{e \in E}\vartheta_e \hat\phi_e^p = \sum_{e\in E}\vartheta_e\hat\phi_e^p = \hat\nu.
\] 
It then follows that
\begin{equation}
\label{eq:ratio}
	\frac{\sum_{e \in E} \vartheta_e f_e(\hat{x})\hat{l}(e)^{q-1}}{(\Delta - d)^T \hat{x}} \le \frac{\sum_{e \in E} \vartheta_e f_e(\hat{x})\hat{l}(e)^{q-1}}{\hat\nu} \stackrel{(i)}{\le} \frac{4\hat\nu^{1/q}}{\beta\hat\nu} = \frac{4}{\beta\hat\nu^{1/p}} \stackrel{(ii)}{\le} \frac{4\vol(\partial C)^{1/q}}{\beta\vol(C)},
\end{equation}
where $(i)$ is follows from Claim~\ref{claim:top} and $(ii)$ follows from Claim~\ref{claim:flows}.

We can write the left-most ratio in \eqref{eq:ratio} in its integral form, as follows. By Lemma~\ref{lem:lovasz}, we have 
\[
	(\Delta - d)^T \hat{x} = \int_{h=0}^\infty (\Delta(S_h) - \vol(S_h)) dh,
\]
and
\[
\begin{split}
	\sum_{e \in E} \vartheta_e f_e(\hat{x}) \hat{l}(e)^{q-1} 
	&= \sum_{e \in E} \vartheta_e \int_{h=0}^\infty w_e(S_h) dh~\hat{l}(e)^{q-1}\\
	&= \int_{h=0}^\infty \sum_{e \in E} \vartheta_e w_e(S_h)\hat{l}(e)^{q-1}dh\\
	&= \int_{h=0}^\infty \sum_{e \in \partial S_h} \vartheta_e w_e(S_h)\hat{l}(e)^{q-1}dh,
\end{split}
\]
where the last equality follows from the fact that $w_e(S_h) = 0$ for $e \not\in \partial S_h$. Therefore, we get
\[
	\int_{h=0}^\infty \frac{\sum_{e \in \partial S_h} \vartheta_e w_e(S_h)\hat{l}(e)^{q-1}}{\Delta(S_h)-\vol(S_h)} dh \le \frac{4\vol(\partial C)^{1/q}}{\beta\vol(C)},
\]
which means that there exists $h > 0$ such that
\begin{equation}
\label{eq:first}
	\frac{\sum_{e\in\partial S_h}\vartheta_e w_e(S_h) \hat{l}(e)^{q-1}}{\Delta(S_h) - \vol(S_h)} \le \frac{4\vol(\partial C)^{1/q}}{\beta\vol(C)}.
\end{equation}

Finally, we connect the left hand side in inequality~\eqref{eq:first} to the conductance of $S_h$.
For the denominator, by Assumption~\ref{assum:delta_a}, we have
\begin{equation}
\label{eq:second}
	\Delta(S_h) - \vol(S_h) \le \frac{3}{\alpha}\vol(S_h).
\end{equation}
For the numerator, every hyperedge $e \in \partial S_h$ must contain some $u,v \in e$ such that $\hat{x}_u \neq \hat{x}_v$, thus $f_e(\hat{x}) > 0$, which means $\hat{l}(e) \ge 1/\vol(C)^{1/q}$. This gives
\begin{equation}
\label{eq:third}
\sum_{e \in \partial S_h} \vartheta_e w_e(S_h) \hat{l}(e)^{q-1} \ge \frac{\sum_{e \in \partial S_h} \vartheta_e w_e(S_h)}{\vol(C)^{(q-1)/q}} = \frac{\vol(\partial S_h)}{\vol(C)^{(q-1)/q}}.
\end{equation}

Put \eqref{eq:first}, \eqref{eq:second} and \eqref{eq:third} together, there exists $h > 0$ such that
\[
	\Phi(S_h) = \frac{\vol(\partial S_h)}{\vol(S_h)} \le \frac{12\vol(\partial C)^{1/q}}{\alpha\beta\vol(C)^{1/q}} = \frac{12\Phi(C)^{1/q}}{\alpha\beta}.
\]

\section{Optimization algorithm for HFD}\label{sec:A02}

In this section we give details on an Alternating Minimization (AM) algorithm~\cite{Beck2015} that solves the primal problem~\eqref{eq:primalLp}. In Algorithm~\ref{alg:lpAM} we write the basic AM steps in a slightly more general form than what is given by Algorithm 1 in the main paper. The key observation is that the AM method provides a unified framework to solve HFD, when the objective function of the primal problem~\eqref{eq:primalLp} is penalized by any $\ell_p$-norm for $p \ge 2$.

{\centering
\begin{minipage}{.7\linewidth}
\begin{algorithm}[H]
	\caption{Alternating Minimization for HFD}
	\label{alg:lpAM}
	{\bf Initialization:} 
	\[
	\phi^{(0)} := 0, r^{(0)} := 0, s^{(0)}_e := D^{-1}A_e\left[\Delta - d\right]_+, \forall e \in E.
	\]
	{\bf For $k = 0,1,2,\ldots$ do:}
	\begin{align*}
	&(\phi^{(k+1)}, r^{(k+1)}) := \argmin\limits_{(\phi,r)\in\mathcal{C}} \sum\limits_{e \in E}\vartheta_e\left(\phi_e^p + \frac{1}{\sigma^{p-1}}\|s_e^{(k)} - r_e\|_p^p\right)\\
	&s^{(k+1)} := \argmin\limits_{s} \sum\limits_{e \in E}\vartheta_e \|s_e - r_e^{(k+1)}\|_p^p, 
	\hspace{2mm}\mbox{s.t.}~\Delta-\sum\limits_{e\in E}\vartheta_es_e \le d, \ s_{e,v}=0,\forall v\not\in e.
	\end{align*}
\end{algorithm}
\end{minipage}
\par
}
\vspace{5pt}

Let us remind the reader the definitions and notation that we will use. We consider a generic hypergraph $H = (V,E,\WW)$ where $\WW = \{w_e,\vartheta_e\}_{e \in E}$ are submodular hyperedge weights. \ For each $e \in E$, we define a diagonal matrix $A_e \in \mathbb{R}^{|V| \times |V|}$ such that $[A_e]_{v,v} = 1$ if $v \in e$ and 0 otherwise.\ We use the notation $r \in \bigotimes_{e\in E}\mathbb{R}^{|V|}$ to represent a vector in the space $\RR^{|V||E|}$, where each $r_e \in \RR^{|V|}$ corresponds to a block in $r$ indexed by $e \in E$. \ For a vector $r_e \in \RR^{|V|}$, $r_{e,v}$ is the entry in $r_e$ that corresponds to $v \in V$. \ For a vector $x \in \RR^{|V|}$, $[x]_+ := \max\{x,0\}$ where the maximum is taken entry-wise. 

We denote $\mathcal{C} := \{(\phi,r) \in \RR^{|E|}_+ \times (\bigotimes_{e\in E}\RR^{|V|}) ~|~  r_e \in \phi_eB_e, ~\forall e \in E\}$.

We will prove the equivalence between the primal diffusion problem~\eqref{eq:primalLp} and its separable reformulation shortly, but let us start with a simple lemma that gives closed-form solution for one of the AM sub-problems.

\begin{lemma}
\label{lem:s-step}
The optimal solution to the following problem
\begin{equation}
\label{prob:s-step}
\min_{s\in\bigotimes_{e\in E}\mathbb{R}^{|V|}} \sum_{e \in E} \vartheta_e \|s_e-r_e\|_p^p, ~\mbox{s.t.}~\Delta - \sum_{e\in E}\vartheta_e s_e \le d, ~s_{e,v} = 0, \forall v \not\in e.
\end{equation}
is given by
\begin{equation}
\label{eq:s-step}
s_e^* = r_e + A_eD^{-1}\Big[\Delta - \sum_{e' \in E}\vartheta_{e'} r_{e'} - d \Big]_+, ~\forall e \in E.
\end{equation}
\end{lemma}

\begin{proof}
Rewrite \eqref{prob:s-step} as
\begin{align*}
\min_{s\in\bigotimes_{e\in E}\mathbb{R}^{|V|}} &\sum_{v\in V}\sum_{e\in E} \vartheta_e |s_{e,v} - r_{e,v}|^p\\
\mbox{s.t.}\hspace{6mm}&\Delta_v - \sum_{e \in E} \vartheta_e s_{e,v} \le d_v, ~\forall v \in V\\
&s_{e,v} = 0, ~\forall v \not\in e.
\end{align*}
Then it is immediate to see that \eqref{prob:s-step} decomposes into $|V|$ sub-problems indexed by $v \in V$,
\begin{equation}
\label{prob:s-step-sub}
\min_{\xi_v \in \mathbb{R}^{|E_v|}} \sum_{e \in E_v} \vartheta_e |\xi_{v,e} - r_{e,v}|^p, ~\mbox{s.t.}~ \Delta_v - \sum_{e \in E_v} \vartheta_e \xi_{v,e} \le d_v,
\end{equation}
where $E_v := \{e \in E ~|~ v \in e\}$ is the set of hyperedges incident to $v$, and we use $\xi_{v,e}$ for the entry in $\xi_v$ that corresponds to $e \in E_v$. Let $\xi_v^*$ denote the optimal solution for \eqref{prob:s-step-sub}. We have that $s^*_{e,v} = \xi^*_{v,e}$ if $v \in e$ and $s^*_{e,v}=0$ otherwise. Therefore, it suffices to find $\xi_v^*$ for $v \in V$. The optimality condition of \eqref{prob:s-step-sub} is given by
\begin{align*}
&p\vartheta_e|\xi_{v,e} - r_{e,v}|^{p-1}\sign(\xi_{v,e} - r_{e,v}) - \vartheta_e \lambda\ni 0, ~\forall e \in E_v,\\
&\lambda \ge 0, ~\Delta_v - \sum_{e \in E_v}\vartheta_e \xi_{v,e} \le d_v,
~\lambda\Big(\Delta_v - \sum_{e \in E_v}\vartheta_e \xi_{v,e} - d_v\Big) = 0,
\end{align*}
where
\[
	\sign(a) := \left\{\begin{array}{ll} \{-1\}, & \mbox{if} ~ a < 0, \\ \{1\}, & \mbox{if} ~ a > 0, \\ \mbox{$[-1,1]$} & \mbox{if} ~ a = 0. \end{array}\right.
\]
There are two cases about $\lambda$. We show that in both cases the solution given by \eqref{eq:s-step} is optimal.

{\em Case 1.} If $\lambda > 0$, then we must have that $p\vartheta_e|\xi_{v,e} - r_{e,v}|^{p-1}>0$ for all $e \in E_v$ (otherwise, the stationarity condition would be violated). This means that $p|\xi_{v,e} - r_{e,v}|^{p-1} = \lambda$ for all $e \in E_v$, that is, $\xi_{v,e_1} - r_{e_1,v} = \xi_{v,e_2} - r_{e_2,v} > 0$ for every $e_1,e_2 \in E_v$. Denote $t_v :=\xi_{v,e} - r_{e,v}$. Because $\lambda>0$, by complementarity we have
\[
	\Delta_v - \sum_{e \in E_v}\vartheta_e (t_v+r_{e,v}) = \Delta_v - \sum_{e \in E_v}\vartheta_e \xi_{v,e} = d_v,
\]
which implies that $t_v = (\sum_{e \in E_v} \vartheta_e)^{-1}(\Delta_v - \sum_{e \in E_v}\vartheta_e r_{e,v} - d_v)$. Note that $\Delta_v - \sum_{e \in E_v}\vartheta_e r_{e,v} - d_v > 0$ because $\Delta_v - \sum_{e \in E_v}\vartheta_e \xi_{v,e} - d_v = 0$ and $\xi_{v,e} > r_{e,v}$ for all $e \in E_v$. Therefore we have that
\[
	s^*_{e,v} = \xi^*_{v,e} = r_{e,v} + d_v^{-1}\Big[\Delta_v - \sum_{e \in E_v}\vartheta_e r_{e,v} - d_v\Big]_+.
\]
{\em Case 2.} If $\lambda = 0$, then we have that $p\vartheta_e|\xi_{v,e} - r_{e,v}|^{p-1}\sign(\xi_{v,e} - r_{e,v}) \ni 0$ for all $e \in E_v$, which implies $\xi_{v,e} - r_{e,v}= 0$ for all $e \in E_v$. Then we must have
\[
	\Delta_v - \sum_{e \in E_v}\vartheta_e r_{e,v} = \Delta_v - \sum_{e \in E_v}\vartheta_e \xi_{v,e} \le d_v.
\]
Therefore we still have that
\[
	s^*_{e,v} = \xi^*_{v,e} = r_{e,v} = r_{e,v} + d_v^{-1}\Big[\Delta_v - \sum_{e \in E_v}\vartheta_e r_{e,v} - d_v\Big]_+.
\]
The required result then follows from the definition of $A_e$ and $D$.
\end{proof}

We are now ready to show that the primal problem \eqref{eq:primalLp} can be cast into an equivalent separable formulation, which can then be solved by the AM method in Algorithm~\ref{alg:lpAM}. We give the reformulation under general $\ell_p$-norm penalty and arbitrary $\vartheta_e > 0$.

\begin{lemma}[Lemma 3 in the main paper]
\label{lem:sep_a}
The following problem is equivalent to \eqref{eq:primalLp} for any $\sigma > 0$, in the sense that $(\hat{\phi},\hat{r},\hat{z})$ is optimal in \eqref{eq:primalLp} for some $\hat{z} \in \RR^{|V|}$ if and only if $(\hat{\phi},\hat{r},\hat{s})$ is optimal in \eqref{eq:sep} for some $\hat{s} \in \bigotimes_{e\in E}\mathbb{R}^{|V|}$.
\begin{equation}
\label{eq:sep}
\begin{split}
	\min_{\phi, r, s}~&\frac{1}{p}\sum_{e \in E} \vartheta_e\left(\phi_e^p + \frac{1}{\sigma^{p-1}}\left\|s_e - r_e\right\|_p^p\right)\\
	\textnormal{s.t.}~&(\phi,r) \in \mathcal{C}, ~\Delta - \sum_{e \in E}\vartheta_e s_e \le d, ~s_{e,v} = 0, \forall v \not\in e.
\end{split}
\end{equation}
\end{lemma}

\begin{proof}
We will show the forward direction and the converse follows from exactly the same reasoning. Let $\hat\nu_1$ and $\hat\nu_2$ denote the optimal objective value of problems \eqref{eq:primalLp} and \eqref{eq:sep}, respectively. Let $(\hat{\phi},\hat{r},\hat{z})$ be an optimal solution for \eqref{eq:primalLp}. Define $\hat{s}_e := \hat{r}_e + \sigma A_e \hat{z}$ for $e \in E$. We show that $(\hat{\phi},\hat{r},\hat{s})$ is an optimal solution for \eqref{eq:sep}. 

Because $\hat{r}_{e,v} = 0$ for all $v \not\in e$, by the definition of $A_e$, we know that $\hat{s}_{e,v} = 0$ for all $v \not\in e$. Moreover,
\[
	\sigma D\hat{z} = \sigma \sum_{e \in E}\vartheta_e A_e\hat{z} = \sum_{e \in E}\vartheta_e(\hat{s}_e-\hat{r}_e),
\]
so
\[
	 \Delta - \sum_{e \in E}\vartheta_e \hat{s}_e = \Delta - \sum_{e \in E}\vartheta_e \hat{r}_e - \sigma D\hat{z} \le d.
\]
Therefore, $(\hat{\phi},\hat{r},\hat{s})$ is a feasible solution for \eqref{eq:sep}. Furthermore,
\[
\begin{split}
	\sigma \sum_{v \in V} d_v \hat{z}_v^p 
	&= \sigma \sum_{e \in E} \vartheta_e \sum_{v \in e} \hat{z}_v^p 
	= \sigma\sum_{e \in E}\vartheta_e \|A_e\hat{z}\|_p^p \\
	&= \frac{1}{\sigma^{p-1}}\sum_{e \in E} \vartheta_e  \left\|\sigma A_e\hat{z}\right\|_p^p 
	= \frac{1}{\sigma^{p-1}}\sum_{e \in E} \vartheta_e  \left\|\hat{s}_e - \hat{r}_e\right\|_p^p.
\end{split}
\]
This means that $(\hat{\phi},\hat{r},\hat{s})$ attains objective value $\hat\nu_1$ in \eqref{eq:sep}. Hence $\hat\nu_1 \ge \hat\nu_2$. 

In order to show that $(\hat{\phi},\hat{r},\hat{s})$ is indeed optimal for \eqref{eq:sep}, it left to show that $\hat\nu_2 \ge \hat\nu_1$. Let $(\phi',r',s')$ be an optimal solution for \eqref{eq:sep}. Then we know that
\begin{equation}
\label{eq:altmins}
	s' = \argmin_{s\in\bigotimes_{e\in E}\RR^{|V|}} \sum_{e \in E} \vartheta_e \|s_e-r'_e\|_p^p, 
	~\mbox{s.t.}~\Delta - \sum_{e\in E}\vartheta_e s_e \le d, ~s_{e,v} = 0 ~\forall v \not\in e.
\end{equation}
According to Lemma~\ref{lem:s-step}, we know that
\begin{equation}
\label{eq:s_update}
	s'_e = r'_e + A_eD^{-1}\Big[\Delta - \sum_{e' \in E}\vartheta_{e'} r'_{e'} - d \Big]_+, ~\forall e \in E.
\end{equation}
Define $z' := \frac{1}{\sigma}D^{-1}[\Delta - \sum_{e \in E}\vartheta_e r'_e - d]_+$. Then $z' \ge 0$. Moreover, we have that
\[
	\sum_{e \in E}\vartheta_e s'_e - \sum_{e\in E}\vartheta_e r'_e = \sum_{e \in E} \vartheta_e A_eD^{-1}\Big[\Delta - \sum_{e' \in E}\vartheta_{e'} r'_{e'} - d \Big]_+ = \Big[\Delta - \sum_{e' \in E}\vartheta_{e'} r'_{e'} - d \Big]_+ = \sigma D z',
\]
so
\[
	\Delta - \sum_{e \in E}\vartheta_e r'_e = \Delta - \sum_{e \in E}\vartheta_e s'_e + \sigma D z' \le d + \sigma D z'.
\]
Therefore, $(\phi', r', z')$ is a feasible solution for \eqref{eq:primalLp}. Furthermore,
\[
\begin{split}
	\frac{1}{\sigma^{p-1}}\sum_{e \in E} \vartheta_e  \left\|s'_e - r'_e\right\|_p^p
	&= \frac{1}{\sigma^{p-1}}\sum_{e \in E} \vartheta_e  \left\|\sigma A_e z'\right\|_p^p 
	= \sigma\sum_{e \in E}\vartheta_e \|A_e z'\|_p^p \\
	&= \sigma \sum_{e \in E} \vartheta_e \sum_{v \in e} {z'}_v^p 
	= \sigma \sum_{v \in V} d_v {z'}_v^p.
\end{split}
\]
This means that $(\phi', r', z')$ attains objective value $\hat\nu_2$ in \eqref{eq:primalLp}. Hence $\hat\nu_2 \ge \hat\nu_1$.
\end{proof}

{\bf Remark.} The constructive proof of Lemma~\ref{lem:sep_a} means that, given an optimal solution $(\hat{\phi},\hat{r},\hat{s})$ for problem~\eqref{eq:sep}, one can recover an optimal solution $(\hat{\phi},\hat{r},\hat{z})$ for our original primal formulation~\eqref{eq:primalLp} via $\hat{z} := \frac{1}{\sigma}D^{-1}[\Delta - \sum_{e \in E}\vartheta_e \hat{r}_e - d]_+$. It then follows from Lemma~\ref{lem:conjugate} that the dual optimal solution $\hat{x}$ is given by $\hat{x} = \hat{z}^{p-1}$. Therefore, a sweep cut rounding procedure readily applies to the solution $(\hat{\phi},\hat{r},\hat{s})$ of problem~\eqref{eq:sep}.

Let $g(\phi,r,s)$ denote the objective function of problem~\eqref{eq:sep} and let $g^*$ denote its optimal objective value.

The following theorem gives the convergence rate of Algorithm~\ref{alg:lpAM} applied to \eqref{eq:sep}, when its objective function is penalized by $\ell_p$-norm for $p \ge 2$.

\begin{theorem}[\cite{Beck2015}]
\label{thm:convergence}
Let $\{\phi^{(k)},r^{(k)},s^{(k)}\}_{k\ge0}$ be the sequence generated by Algorithm~\ref{alg:lpAM}. Then for any $k \ge 1$,
\[
	g(\phi^{(k)},r^{(k)},s^{(k)}) - g^* \le \frac{3\max\{g(\phi^{(0)},r^{(0)},s^{(0)}) - g^*, L_p R^2\}}{k},
\]
where
\[
\begin{split}
	R &= \max_{(\phi,r,s) \in \mathcal{F}} ~\max_{(\hat{\phi},\hat{r},\hat{s}) \in \mathcal{O}} \big\{\|\phi-\hat{\phi}\|_2^2+\|r-\hat{r}\|_2^2+\|s-\hat{s}\|_2^2~\big|~g(\phi,r,s)\le g(\phi^{(0)},r^{(0)},s^{(0)})\big\},\\
	L_p &= (p-1)\frac{\vartheta_{\max}^{2/p}\|\Delta\|_p^{p-2}}{d_{\min}^{(p-1)(p-2)/p}\sigma^{p-1}},
\end{split}
\]
where $\mathcal{F}$ and $\mathcal{O}$ denote the feasible set and set of optimal solutions, respectively, $\vartheta_{\max} := \max\limits_{e \in E}\vartheta_e$, and $d_{\min} := \min\limits_{v \in \supp(\Delta)}d_v$.
\end{theorem}

{\bf Remark.} When $p=2$, as considered in the main paper, the objective function $g(\phi,r,s)$ has Lipschitz continuous gradient with constant $L_2 = \vartheta_{\max}/\sigma$. When $p>2$, the gradient of $g(\phi,r,s)$ is not generally Lipschitz continuous. However, the sub-linear convergence rate in Theorem~\ref{thm:convergence} applies as long as $g(\phi,r,s)$ is block Lipschitz smooth in the sub-level sets containing the iterates generated by Algorithm~\ref{alg:lpAM}. We give more details in Subsection~\ref{sec:lipschitz}.

\subsection{Block Lipschitz smoothness over sub-level set}
\label{sec:lipschitz}

Recall that $g(\phi,r,s)$ denotes the objective function of problem~\eqref{eq:sep}. Lemma~\ref{lem:lipschitz} concerns specifically the setting when problem~\ref{eq:sep} is penalized by the $\ell_p$-norm for some $p > 2$.

\begin{lemma}[Block Lipschitz smoothness]
\label{lem:lipschitz}
The partial gradient $\nabla_{(\phi,r)} g(\phi,r,s)$ is Lipschitz continuous over the sub-level sets (given any fixed $s$)
\[
	U_{\phi,r}(s) := \{(\phi,r) \in \RR^{|V|}_+ \times (\mbox{$\bigotimes_{e\in E}$}\RR^{|V|}) ~|~ g(\phi,r,s) \le g(\phi^{(0)}, r^{(0)}, s^{(0)})\}
\]
with constant $L_{\phi,r}$ such that
\[
	L_{\phi,r} \le (p-1)\frac{\vartheta_{\max}^{2/p}\|\Delta\|_p^{p-2}}{d_{\min}^{(p-1)(p-2)/p}\sigma^{p-1}},
\]	
where $\vartheta_{\max} := \max_{e\in E} \vartheta_e$ and $d_{\min} := \min_{v \in \supp(\Delta)} d_v$. The partial gradient $\nabla_s g(\phi,r,s)$ is Lipschitz continuous over the sub-level sets (given any fixed $(\phi,r)$)
\[
	U_s(\phi,r) := \{s \in \mbox{$\bigotimes_{e\in E}$}\RR^{|V|} ~|~ g(\phi,r,s) \le g(\phi^{(0)}, r^{(0)}, s^{(0)})\}
\]
with constant $L_s \le L_{\phi,r}$.
\end{lemma}

\begin{proof}
Fix $s \in \bigotimes_{e \in E}\RR^{|V|}$ and consider
\[
	g_1(\phi,r) := g(\phi,r,s) = \frac{1}{p}\sum_{e \in E} \vartheta_e\phi_e^p + \frac{1}{p\sigma^{p-1}}\sum_{e \in E}\sum_{v \in V} \vartheta_e|r_{e,v}-s_{e,v}|^p.
\]
The function $g_1(\phi,r)$ is coordinate-wise separable and hence its second order derivative $\nabla^2 g_1(\phi,r)$ is a diagonal matrix. Therefore, the largest eigenvalue of $\nabla^2 g_1(\phi,r)$ is the largest coordinate-wise second order partial derivative, that is,
\[
	L_{\phi,r} = \max_{(\phi,r) \in U_{\phi,r}(s)}\lambda_{\max}(\nabla^2 g_1(\phi,r)) 
	=  \max_{(\phi,r) \in U_{\phi,r}(s)}\max_{e\in E, v\in V}\{\nabla^2_{\phi_e}g_1(\phi,r), \nabla^2_{r_{e,v}}g_1(\phi,r)\}.
\] 
So it suffices to upper bound $\nabla^2_{\phi_e}G(\phi,r)$ and $\nabla^2_{r_{e,v}}G(\phi,r)$ for all $(\phi,r) \in  U_{\phi,r}(s)$. We have that
\[
	g(\phi^{(0)}, r^{(0)}, s^{(0)}) 
	= \frac{1}{p\sigma^{p-1}}\sum_{e \in E} \vartheta_e \sum_{v \in e} \frac{[\Delta_v - d_v]_+^p}{d_v^p}
	= \frac{1}{p\sigma^{p-1}}\sum_{v \in V} \frac{[\Delta_v - d_v]_+^p}{d_v^{p-1}}
	\le \frac{\|\Delta\|_p^p}{p\sigma^{p-1}d_{\min}^{p-1}}
\]
where $d_{\min} = \min_{v \in \supp(\Delta)} d_v$. It follows that for all $(\phi,r) \in U_{\phi,r}(s)$,
\begin{align*}
	\nabla^2_{\phi_e}g_1(\phi,r)&= (p-1)\vartheta_e\phi_e^{p-2} \le \frac{(p-1)\vartheta_e^{2/p}\|\Delta\|_p^{p-2}}{d_{\min}^{(p-1)(p-2)/p}\sigma^{(p-1)(p-2)/p}} \le \frac{(p-1)\vartheta_e^{2/p}\|\Delta\|_p^{p-2}}{d_{\min}^{(p-1)(p-2)/p}\sigma^{p-1}}, ~\forall e \in E,\\
	\nabla^2_{r_{e,v}}g_1(\phi,r)&=(p-1)\frac{\vartheta_e}{\sigma^{p-1}}|s_{e,v}-r_{e,v}|^{p-2} \le \frac{(p-1)\vartheta_e^{2/p}\|\Delta\|_p^{p-2}}{d_{\min}^{(p-1)(p-2)/p}\sigma^{p-1}}, ~\forall e\in E, ~\forall v \in V,
\end{align*}
because otherwise we would have $g(\phi,r,s) > g(\phi^{(0)}, r^{(0)}, s^{(0)})$. Hence,
\[
	L_{\phi,r}
	\le \max_{e\in E} \frac{(p-1)\vartheta_e^{2/p}\|\Delta\|_p^{p-2}}{d_{\min}^{(p-1)(p-2)/p}\sigma^{p-1}} 
	= \frac{(p-1)\vartheta_{\max}^{2/p}\|\Delta\|_p^{p-2}}{d_{\min}^{(p-1)(p-2)/p}\sigma^{p-1}}.
\]
Finally, by the symmetry between $r$ and $s$ in $F(\phi,r,s)$, we know that $L_s \le L_{\phi,r}$.
\end{proof}

{\bf Remark.} Because the iterates generated by Algorithm~\ref{alg:lpAM} monotonically decrease the objective function value, in particular, we have that
\[
	g(\phi^{(0)},r^{(0)},s^{(0)}) \ge g(\phi^{(k+1)},r^{(k+1)},s^{(k)}) \ge g(\phi^{(k+1)},r^{(k+1)},s^{(k+1)})
\]
for any $k \ge 0$. Therefore, the sequence of iterates live in the sub-level sets. As a result, for any $p > 2$, the block Lipschitz smoothness within sub-level sets suffices to obtain the sub-linear convergence rate for the AM method~\cite{Beck2015}.

\subsection{Alternating minimization sub-problems}

We now discuss how to solve the sub-problems in Algorithm~\ref{alg:lpAM} efficiently. By Lemma~\ref{lem:s-step}, we know that the sub-problem with respect to $s$,
\[
	s^{(k+1)} := \argmin\limits_{s} \sum\limits_{e \in E}\vartheta_e \|s_e - r_e^{(k+1)}\|_p^p, ~\mbox{s.t.}~\Delta-\sum\limits_{e\in E}\vartheta_es_e \le d, \ s_{e,v}=0,\forall v\not\in e,
\]
has closed-form solution
\[
	s_e^{(k+1)} = r_e^{(k+1)} + A_eD^{-1}\Big[\Delta - \sum_{e' \in E}\vartheta_{e'} r_{e'}^{(k+1)} - d \Big]_+, ~\forall e \in E.
\]
For the sub-problem with respect to $(\phi,r)$,
\[
	(\phi^{(k+1)}, r^{(k+1)}) := \argmin\limits_{(\phi,r)\in\mathcal{C}} \sum\limits_{e \in E}\vartheta_e\left(\phi_e^p + \frac{1}{\sigma^{p-1}}\|s_e^{(k)} - r_e\|_p^p\right),
\]
note that it decomposes into $|E|$ independent problems that can be minimized separately. That is, for $e \in E$, we have
\begin{equation}
\label{eq:r-step-primal}
\begin{split}
	(\phi_e^{(k+1)}, r_e^{(k+1)})
	&= \argmin_{\phi_e\ge0,r_e\in\phi_eB_e} \vartheta_e\phi_e^p + \frac{1}{\sigma^{p-1}}\vartheta_e\|s_e^{(k)} - r_e\|_p^p \\
	& = \argmin_{\phi_e\ge0,r_e\in\phi_eB_e} \frac{1}{p}\phi_e^p + \frac{1}{p\sigma^{p-1}}\|s_e^{(k)} - r_e\|_p^p.
\end{split}
\end{equation}
The above problem~\eqref{eq:r-step-primal} is strictly convex so it has a unique minimizer.

We focus on $p = 2$ first. In this case, problem~\eqref{eq:r-step-primal} can be solved in sub-linear time using either the conic Frank-Wolfe algorithm or the conic Fujishige-Wolfe minimum norm algorithm studied in \cite{li2020quadratic}. Notice that the dimension of problem \eqref{eq:r-step-primal} is the size of the corresponding hyperedge. Therefore, as long as the hyperedge is not extremely large, we can easily obtain a good update $(\phi_e^{(k+1)}, r_e^{(k+1)})$.

If $B_e$ has a special structure, for example, if the hyperedge weight $w_e$ models unit cut-cost, then an exact solution for \eqref{eq:r-step-primal} can be computed in time $O(|e| \log |e|)$ \cite{li2020quadratic}. For completeness we transfer the algorithmic details in \cite{li2020quadratic} to our setting and list them in Algorithm~\ref{alg:unit}. The basic idea is to find optimal dual variables achieving dual optimality, and then recover primal optimal solution from the dual. We refer the reader to \cite{li2020quadratic} for detailed justifications. Given $e \in E$, $s_e \in \RR^{|V|}$, and $a,b \in \RR$, denote 
\[
	e_{\ge}(a) := \{v \in e ~|~ s_{e,v} \ge \sigma a\} ~~\mbox{and}~~ e_{\le}(b) := \{v \in e ~|~ s_{e,v}\le \sigma b\}.
\]
Define 
\[
	\gamma(a,b) := a - b + \sum_{v \in e_{\ge}(a)}\sigma\left(a - \frac{s_{e,v}}{\sigma}\right).
\]

{\centering
\begin{minipage}{.85\linewidth}
\begin{algorithm}[H]
\caption{An Exact Projection Algorithm for \eqref{eq:r-step-primal} ($p=2$, unit cut-cost) \cite{li2020quadratic}}
\label{alg:unit}
\begin{algorithmic}[1]
\STATE {\bf Input:} $e$, $s_e$. 
\STATE $a \gets \max_{v \in e} s_{e,v}/\sigma$, ~ $b \gets \min_{v \in e} s_{e,v}/\sigma$
\STATE {\bf While} true:
\STATE \hspace{4mm} $w_a \gets \sigma\,|e_{\ge}(a)|$, \ $w_b \gets \sigma\,|e_{\le}(b)|$
\STATE \hspace{4mm} $a_1 \gets \max_{v \in e \setminus e_{\ge}(a)} s_{e,v}/\sigma$, \ $b_1 \gets b + (a-a_1)w_a/w_b$
\STATE \hspace{4mm} $b_2 \gets \min_{v \in e \setminus e_{\le}(b)} s_{e,v}/\sigma$, \ $a_2 \gets a - (b_2-b)w_b/w_a$
\STATE \hspace{4mm} $i^* \gets \argmin_{i \in \{1,2\}} b_i$
\STATE \hspace{4mm} {\bf If} $a_{i^*} \le b_{i^*}$ {\bf or} $\gamma(a_{i^*}, b_{i^*}) \le 0$ {\bf break}
\STATE \hspace{4mm} $a \gets a_{i^*}$, \ $b \gets b_{i^*}$
\STATE $a \gets a - \gamma(a, b)w_b/(w_aw_b + w_a + w_b)$, \ $b \gets b + \gamma(a, b)w_a/(w_aw_b + w_a + w_b)$
\STATE {\bf For} $v \in e$ {\bf do}:
\STATE \hspace{4mm} {\bf If} $v \in e_{\ge}(a)$ {\bf then} \ $r_{e,v} \gets s_{e,v} - \sigma a$
\STATE \hspace{4mm} {\bf Else if} $v \in e_{\le}(b)$ {\bf then} \ $r_{e,v} \gets s_{e,v} - \sigma b$
\STATE \hspace{4mm} {\bf Else} \ $r_{e,v} \gets 0$
\STATE {\bf Return:} $r_e$
\end{algorithmic}
\end{algorithm}
\end{minipage}
\par
}
\vspace{5pt}

Now we discuss the case $p > 2$ in \eqref{eq:r-step-primal}. The dual of \eqref{eq:r-step-primal} is written as
\begin{equation}
\label{eq:r-step-dual}
	\min_{y_e} \frac{1}{q}f_e(y_e)^q + \frac{\sigma}{q}\|y_e\|_q^q - y_e^Ts_e^{(k)}.
\end{equation}
Let $(\phi_e^*, r_e^*)$ and $y_e^*$ be optimal solutions of \eqref{eq:r-step-primal} and \eqref{eq:r-step-dual}, respectively. Then one has
\[
	r_e^* = s_e^{(k)} - \sigma (y_e^*)^{q-1} ~\mbox{and}~ \phi_e^*= \big((r_e^*)^Ty_e^*\big)^{1/q}.
\]
Both the derivation of \eqref{eq:r-step-dual} and the above relations between $(\phi_e^*, r_e^*)$ and $y_e^*$ follow from similar reasoning and algebraic computations used in the proofs of Lemma~\ref{lem:dual} and Lemma~\ref{lem:conjugate}. Therefore, we can use subgradient method to compute $y_e^*$ first and then recover $\phi_e^*$ and $r_e^*$. For special cases like the unit cut-cost, a similar approach to Algorithm~\ref{alg:unit} can be adopted to obtain an almost (up to a binary search tolerance) exact solution, by modifying Steps 2-6 to work with general $\ell_p$-norm and replacing Step 10 with binary search. See Algorithm~\ref{alg:lpunit} for details. 

{\bf Caution.} To simplify notation in Algorithm~\ref{alg:lpunit}, for $c \in \RR$ and $p > 0$, $c^{p}$ is to be interpreted as $c^p := |c|^p\sign(c)$, where we treat $\sign(0) := 0$. For $q = p/(p-1)$, we define
\[
	\gamma_p(a,b) := (a-b)^{q-1} + \sum_{v \in e_{\ge}(a^{q-1})}\sigma\left(a^{q-1} - \frac{s_{e,v}}{\sigma}\right).
\]
{\centering
\begin{minipage}{.88\linewidth}
\begin{algorithm}[H]
\caption{An $\ell_p$-Projection Algorithm for \eqref{eq:r-step-primal} ($p>2$, unit cut-cost)}
\label{alg:lpunit}
\begin{algorithmic}[1]
\STATE {\bf Input:} $e$, $s_e$. 
\STATE $a \gets \max_{v \in e} (s_{e,v}/\sigma)^{p-1}$, ~ $b \gets \min_{v \in e} (s_{e,v}/\sigma)^{p-1}$, ~ $q \gets p/(p-1)$
\STATE {\bf While} true:
\STATE \hspace{4mm} $w_a \gets \sigma\,|e_{\ge}(a^{q-1})|$, \ $w_b \gets \sigma\,|e_{\le}(b^{q-1})|$
\STATE \hspace{4mm} $a_1 \gets \max_{v \in e \setminus e_{\ge}(a^{q-1})} (s_{e,v}/\sigma)^{p-1}$, ~$b_1 \gets (b^{q-1} + (a^{q-1}-a_1^{q-1})w_a/w_b)^{p-1}$
\STATE \hspace{4mm} $b_2 \gets \min_{v \in e \setminus e_{\le}(b^{q-1})} (s_{e,v}/\sigma)^{p-1}$, ~ $a_2 \gets (a^{q-1} - (b_2^{q-1}-b^{q-1})w_b/w_a)^{p-1}$
\STATE \hspace{4mm} $i^* \gets \argmin_{i \in \{1,2\}} b_i$
\STATE \hspace{4mm} {\bf If} $a_{i^*} \le b_{i^*}$ {\bf or} $\gamma_p(a_{i^*}, b_{i^*}) \le 0$ {\bf break}
\STATE \hspace{4mm} $a \gets a_{i^*}$, \ $b \gets b_{i^*}$
\STATE Employ binary search for $\hat{a} \in [b,a]$ such that $\gamma_p(\hat{a},\hat{b}) = 0$ while maintaining $\hat{b} = (b^{q-1} + (a^{q-1}-\hat{a}^{q-1})w_a/w_b)^{p-1}$ and $\hat{b} \le \hat{a}$
\STATE {\bf For} $v \in e$ {\bf do}:
\STATE \hspace{4mm} {\bf If} $v \in v \in e_{\ge}(\hat{a}^{q-1})$ {\bf then} \ $r_{e,v} \gets s_{e,v} - \sigma \hat{a}^{q-1}$
\STATE \hspace{4mm} {\bf Else if} $v \in e_{\le}(\hat{b}^{q-1})$ {\bf then} \ $r_{e,v} \gets s_{e,v} - \sigma \hat{b}^{q-1}$
\STATE \hspace{4mm} {\bf Else} \ $r_{e,v} \gets 0$
\STATE {\bf Return:} $r_e$
\end{algorithmic}
\end{algorithm}
\end{minipage}
\par
}

\section{Empirical set-up and results}\label{sec:A03}

\subsection{Experiments using synthetic data}

In this subsection we provide details aboue how we generate synthetic hypergraphs using $k$-uniform stochastic block model and how we set the parameters for the algorithms used in our experiments. Additional synthetic experiments that demonstrate or explain the robustness of our method are also provided.

{\bf Data generation.} We generate four sets of hypergraphs using the generalized $k$HSBM described in the main paper. All hypergraphs have $n=100$ nodes. For simplicity, we require that each block in the hypergraph has constant size 50.

\underline{\it 1st set of hypergraphs.} We generate the first set of hypergraphs with $k=3$, constant $p=0.0765$ and varying $q \in [0.0041,0.0735]$. Recall that for $k=3$ there is only one possible inter-cluster probability $q \equiv q_1$. We pick $p=0.0765$ so the expected number of intra-cluster hyperedges is 1500 for each block of size 50. We set a wide range for $q$ so that the interval covers both extremes, i.e., when the ground-truth target cluster is very clean or very noisy. These hypergraphs are used to evaluate the performance of algorithms for the unit cut-cost setting when the target cluster conductance varies. Figure~4 in the main paper uses the local clustering results on these hypergraphs.

\underline{\it 2nd set of hypergraphs.} For the second set of hypergraphs, we vary $k \in \{3,4,5,6\}$. Moreover, we set $q_2 = \cdots = q_{\lfloor k/2 \rfloor} = 0$, so every inter-cluster hyperedge contains a single node on one side and the rest on the other side. In this setting, separating the two ground-truth communities will incur a small penalty using the cardinality cut-cost, but a large penalty using the unit cut-cost. Therefore, methods that exploit appropriate cardinality-based cut-cost should perform better. The hypergraphs are sampled so that the conductance of a block stays the same across different $k$'s. We compute the conductance based on the unit cut-cost when generating the hypergraphs, because the scale of conductance based on the unit cut-cost is less affected by $k$ than the scale of conductance based on the cardinality cut-cost. See details below for how the scale of conductance based on the cardinality cut-cost is affected by $k$. The second set of hypergraphs is used to evaluate the performance of algorithms for both unit and cardinality cut-costs when the hyperedge size varies. Figure~5 in the main paper (and Figure~\ref{fig:edgesize_cond20} and Figure~\ref{fig:edgesize_cond25} in the appendix) uses the local clustering results on these hypergraphs.

\underline{\it 3rd set of hypergraphs.} For the third set of hypergraphs, we set $q_2 = \cdots = q_{\lfloor k/2 \rfloor} = 0$. We consider constant $k=4$ or $k=5$, constant $p$ and varying $q_1$. These hypergraphs are used to evaluate the performance of algorithms for both unit and cardinality cut-costs when the target cluster conductance varies. Figure~\ref{fig:varycond_k4} and Figure~\ref{fig:varycond_k5} in the appendix are based on the local clustering results on these hypergraphs.

\underline{\it 4th set of hypergraphs.} This set consists of two hypergraphs generated with $k=3$, $p=0.04$ and $q \in \{0.001, 0.011\}$. The ground-truth target cluster in the first hypergraph has conductance 0.05, while the ground-truth target cluster in the second hypergraph has conductance 0.3. These two hypergraphs are used to compare the performance of algorithms for the unit cut-cost setting, when the theoretical assumptions of LH holds (for the first hypergraph) or fails (for the second hypergraph).


{\bf Parameters.} For HFD, for all synthetic experiments, we initialize the seed mass so that $\|\Delta\|_1$ is three times the volume of the target cluster (recall from Assumption~\ref{assum:delta} this is without loss of generality). We set $\sigma=0.01$. We tune the parameters for LH as suggested by the authors~\cite{LVHLG20}. Specifically, LH has a regularization parameter $\kappa$ and we let $\kappa = c \cdot r$ where $r$ is the ratio between the number of seed node(s) and the size of the target cluster. We perform a binary search on $c$ and find that $c = 0.35$ gives good results for the synthetic hypergraphs. An important parameter for LH is $\delta$. When $\delta=1$ it models unit cut-cost and when $\delta\ge1$ it models cardinality-based cut-cost with an upper bound $\delta$~\cite{LVHLG20}. We consider both cases $\delta=1$ (U-LH) and $\delta\ge1$ (C-LH). In principle, for $k$-uniform hypergraphs LH should produce the same result for any $\delta\ge k$, so one could simply set $\delta=k$ for C-LH. However in our experiments we find that the $\delta$ value that gives the best clustering results can be much larger than $k$. In order to get the best performance out of C-LH, we run C-LH for $\delta = 2^i$, $i = 0,1,\ldots,12$. Among the 13 output clusters from C-LH we pick the one with the lowest conductance. For ACL, we use the same set of parameter values used in~\cite{LVHLG20} because that parameter setting also produces good results in our synthetic experiments.

{\bf Scale of cardinality-based conductance.} To see how ground-truth conductance scales (computed using the cardinality cut-cost) with hyperedge size $k \ge 2$, let us assume that a hypergraph $H = (V,E)$, having $|V|=100$ nodes and two blocks where each block contains 50 nodes, is generated from $p = 0$, $q_1 = 1$ and $q_2 = \ldots = q_{\lfloor k/2 \rfloor} = 0$. In this case, the hypergraph consists of all and only inter-cluster hyperedges. Let $C$ denote a target cluster, that is, $C$ is either one of the two ground-truth blocks. Since we have $|V|=100$ nodes and each of the two blocks contains $50$ nodes, the total number of hyperedges is 
\[
	|E| = 2\binom{50}{k-1}\binom{50}{1}.
\]
Let $w_e$ denote the cardinality-based cut-cost given by $w_e(S) = \min\{|S \cap e|, |e \setminus S|\}/\lfloor |e|/2 \rfloor$. Then for each $e \in E$ we have that $w_e(C) = \frac{1}{\lfloor k/2 \rfloor}$. Moreover, the volume of $C$ is
\[
	\vol(C) = (k-1)\binom{50}{k-1}\binom{50}{1} + \binom{50}{1}\binom{50}{k-1} = k\binom{50}{k-1}\binom{50}{1},
\]
and hence we have
\[
	\Phi(C) = \frac{\vol(\partial C)}{\vol(C)} = \frac{\sum_{e \in E}w_e(C)}{\vol(C)} =  \frac{\frac{1}{\lfloor k/2 \rfloor}|E|}{\vol(C)} = \frac{\frac{2}{\lfloor k/2 \rfloor}\binom{50}{k-1}\binom{50}{1}}{k\binom{50}{k-1}\binom{50}{1}} = \frac{2}{k\lfloor k/2 \rfloor}.
\]
This means that, for any $p\ge0$, $q_1\le1$, $q_2 = \cdots = q_{\lfloor k/2 \rfloor} = 0$, let $B$ be one of the two blocks in $H$, then $\Phi(B) \le 1$ for $k=2,3$, $\Phi(B) \le 1/4$ for $k = 4$, and $\Phi(B) \le 1/5$ for $k = 5$. This explains why the ranges of ground-truth conductance we consider in Figure~\ref{fig:varycond_k4} and Figure~\ref{fig:varycond_k5} are different from the range of ground-truth conductance in Figure~4 in the main paper. For each $k$ we try to make the range of conductance (i.e., $x$-axis) as wide as possible, but due to the different scales of cardinality-based conductance for different $k$, the ranges vary accordingly.

{\bf Additional results.} Figure~\ref{fig:varycond_k4} and Figure~\ref{fig:varycond_k5} show how the algorithms perform on $k$-uniform hypergraphs for $k=4,5$, respectively, as we vary the target cluster conductances. The plots show that as the target cluster becomes more noisy, the performance of all methods degrades. However, C-HFD is better in terms of both conductance and F1 score, especially when the target cluster is noisy but not complete noise (i.e., the ground-truth conductance is high but not too high). For $k=5$ and high-conductance regime, methods that use unit cut-cost, e.g., U-HFD, have poor performance because they find low-conductance clusters based on the unit cut-cost as opposed to the cardinality cut-cost. In general, lower unit cut-cost conductance does not necessarily translates to lower cardinality-based conductance or higher F1 score. For both Figure~\ref{fig:varycond_k4} and Figure~\ref{fig:varycond_k5}, the ground-truth conductance is computed using cardinality-based cut-cost, therefore the ground-truth conductances (on the $x$-axes) have different scales and ranges. Figure~\ref{fig:edgesize_cond20} and Figure~\ref{fig:edgesize_cond25} show the median (markers) and 25-75 percentiles (lower-upper bars) of conductance ratios and F1 scores for $k = 3, 4, 5, 6$. The target clusters have unit cut-cost conductances around 0.2 for Figure~\ref{fig:edgesize_cond20} and 0.25 for Figure~\ref{fig:edgesize_cond25}. Notice that, when the target clusters are less noisy (cf. Figure~5 in the main paper where target clusters are more noisy, having unit conductance around 0.3), U-HFD and C-HFD are significantly better than other methods. The performance of U-HFD is slightly affected by the hyperedge size when the target clusters have unit conductance around 0.25, while the performance of C-HFD stays the same across all $k$'s.

\begin{figure}[ht!]
	\centering
	\begin{subfigure}{0.45\textwidth}
		\centering
		\includegraphics[width=.95\textwidth]{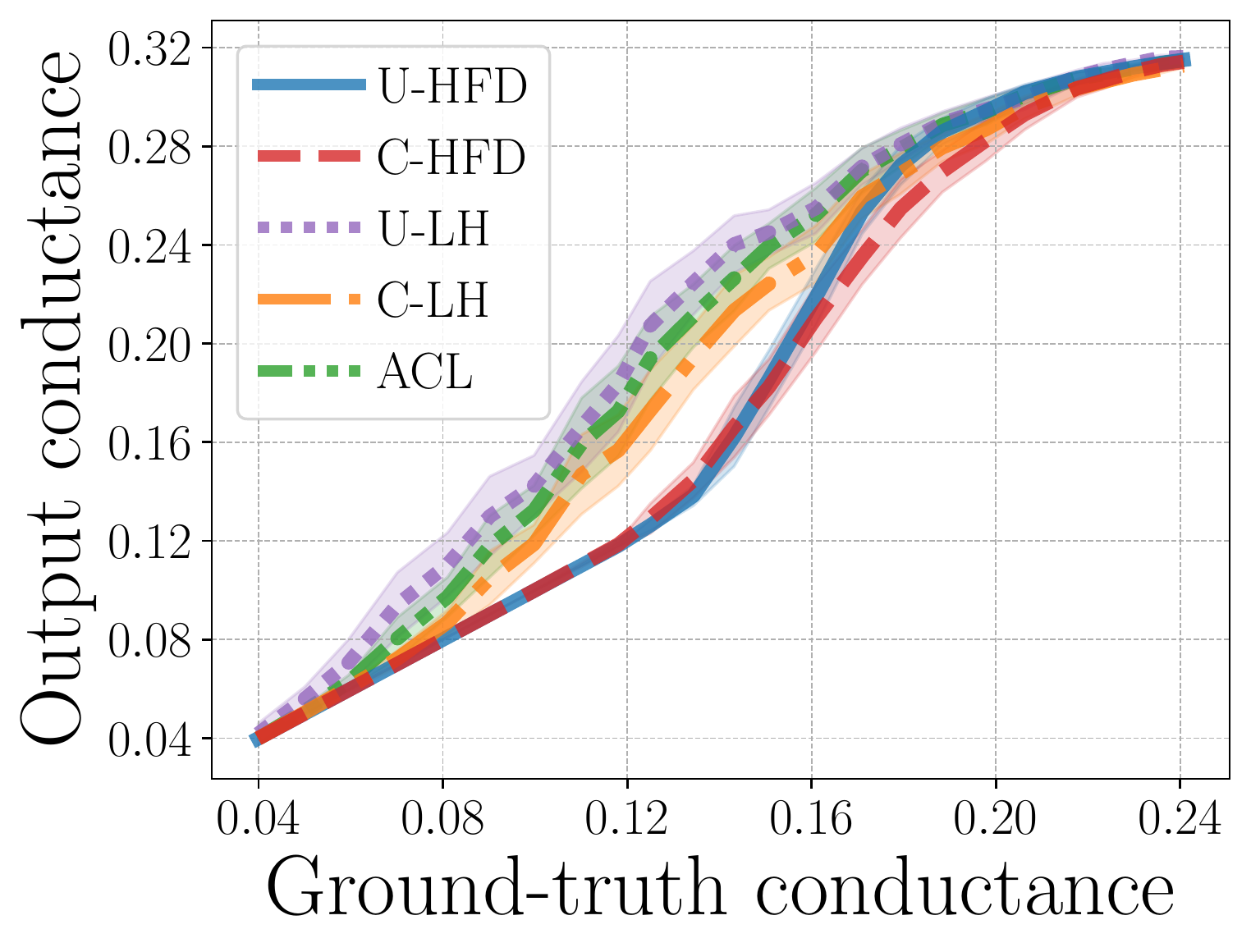}
	\end{subfigure}%
	\begin{subfigure}{0.45\textwidth}
		\centering
		\includegraphics[width=.95\textwidth]{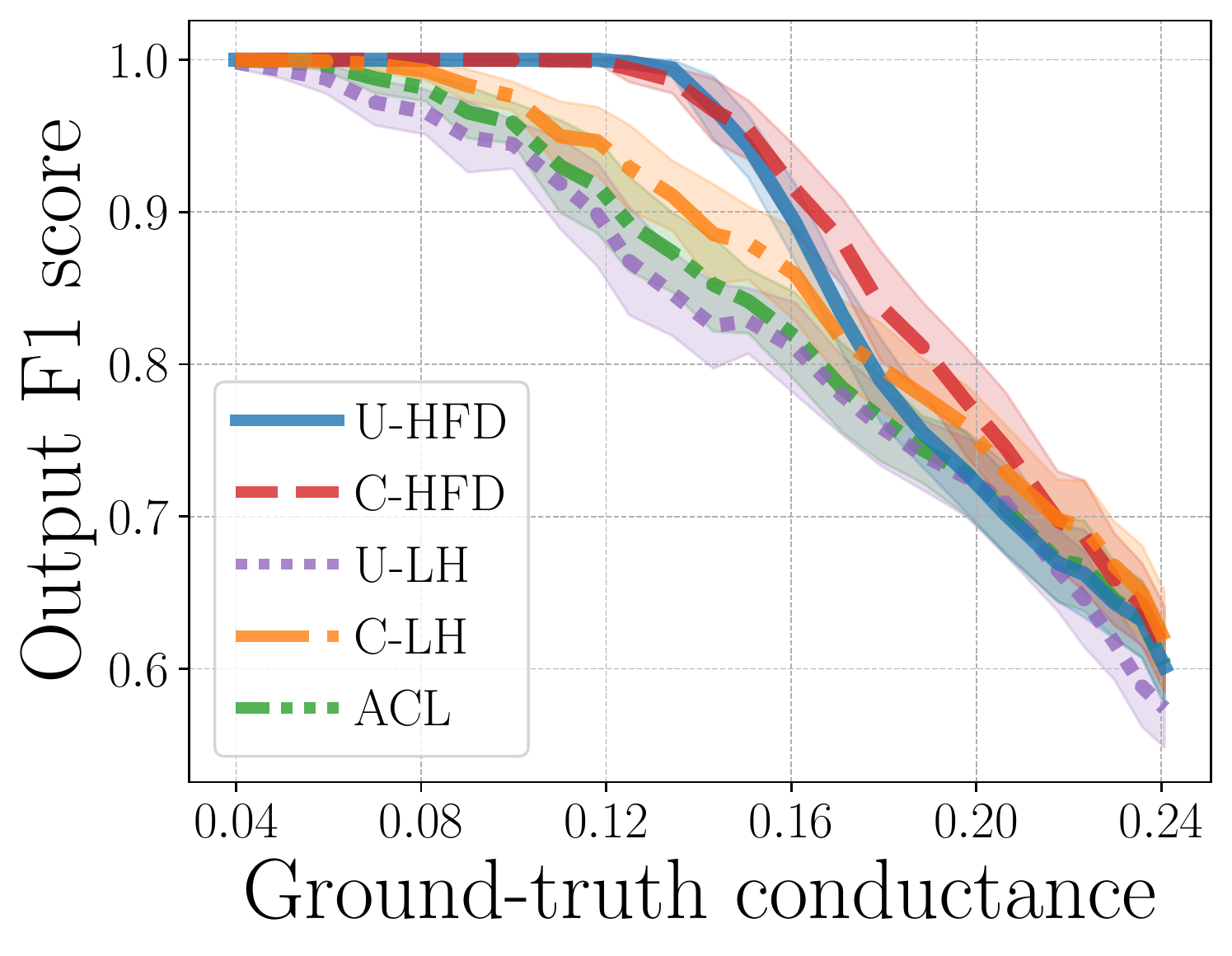}
	\end{subfigure}
	\caption{Average output conductance and F1 score against ground-truth conductance, on $k$-uniform hypergraphs with $k=4$. The error bars show variation over 50 runs using different seed nodes. Both the ground-truth and the target conductances are computed using cardinality-based cut-cost.}
	\label{fig:varycond_k4}
\end{figure}

\begin{figure}[ht!]
	\centering
	\begin{subfigure}{0.45\textwidth}
		\centering
		\includegraphics[width=.95\textwidth]{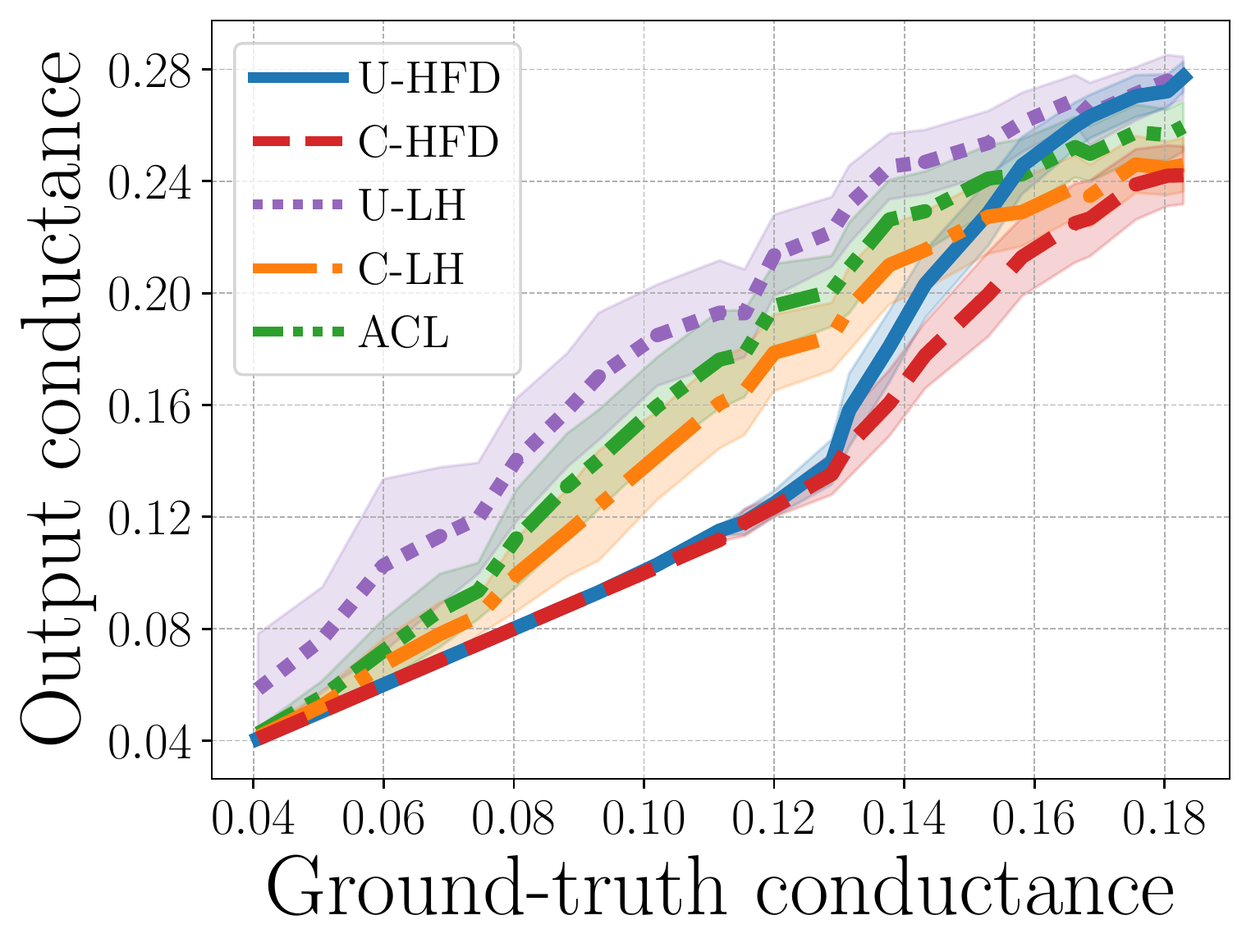}
	\end{subfigure}%
	\begin{subfigure}{0.45\textwidth}
		\centering
		\includegraphics[width=.95\textwidth]{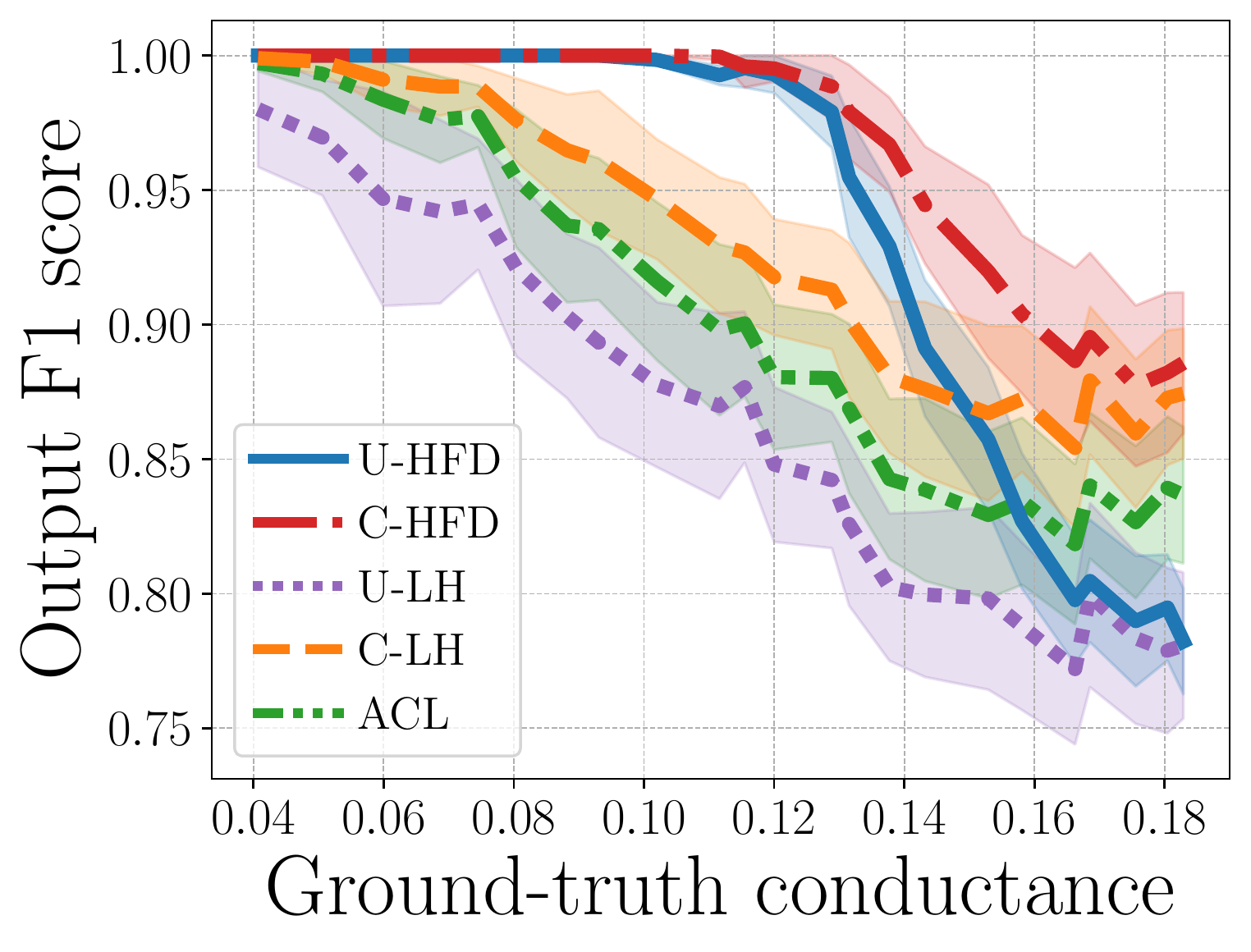}
	\end{subfigure}
	\caption{Average output conductance and F1 score against ground-truth conductance, on $k$-uniform hypergraphs with $k=5$. The error bars show variation over 50 runs using different seed nodes. Both the ground-truth and the target conductances are computed using cardinality-based cut-cost.}
	\label{fig:varycond_k5}
\end{figure}

\begin{figure}[ht!]
	\centering
	\begin{subfigure}{0.45\textwidth}
		\centering
		\includegraphics[width=.95\textwidth]{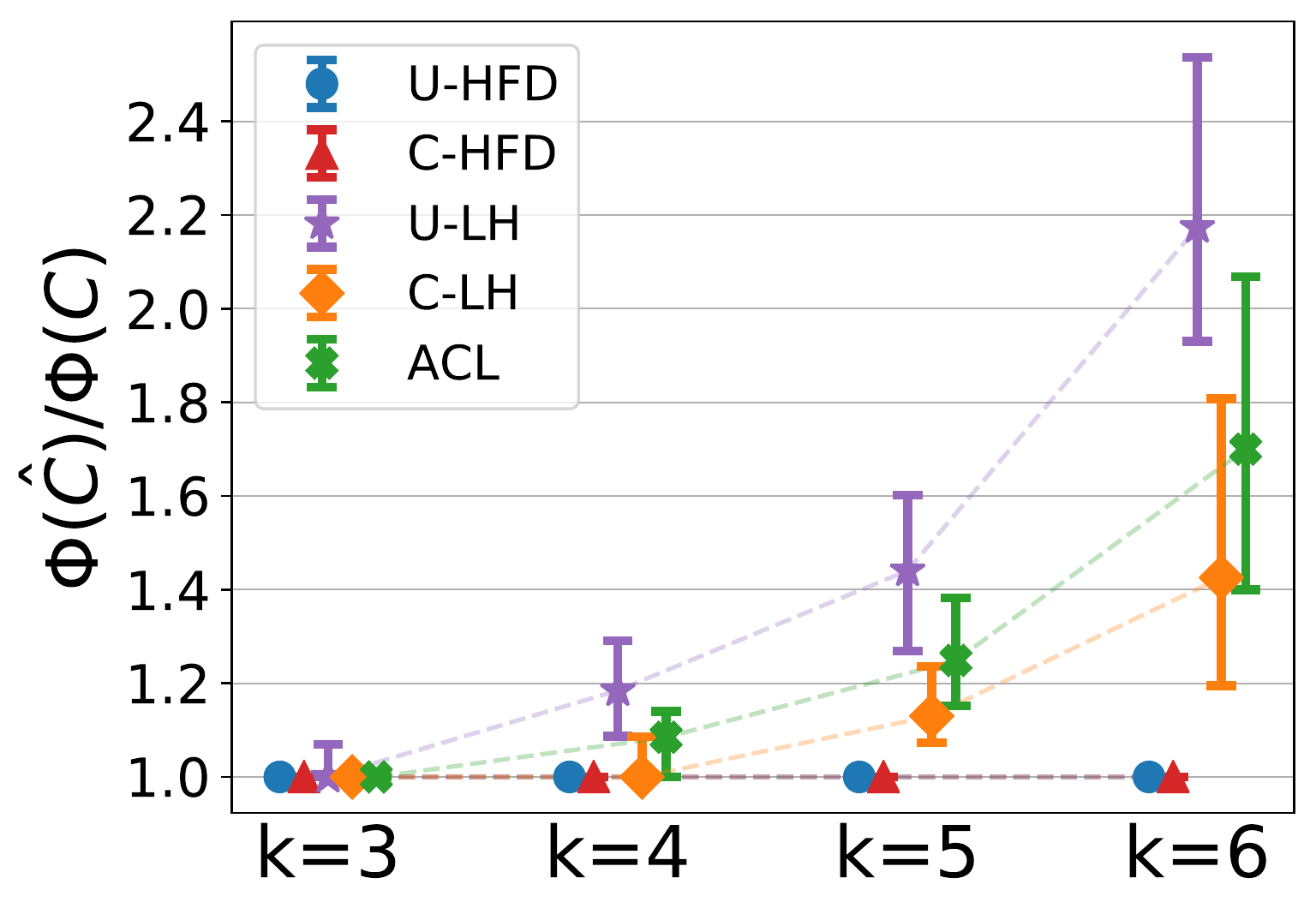}
	\end{subfigure}%
	\begin{subfigure}{0.45\textwidth}
		\centering
		\includegraphics[width=.95\textwidth]{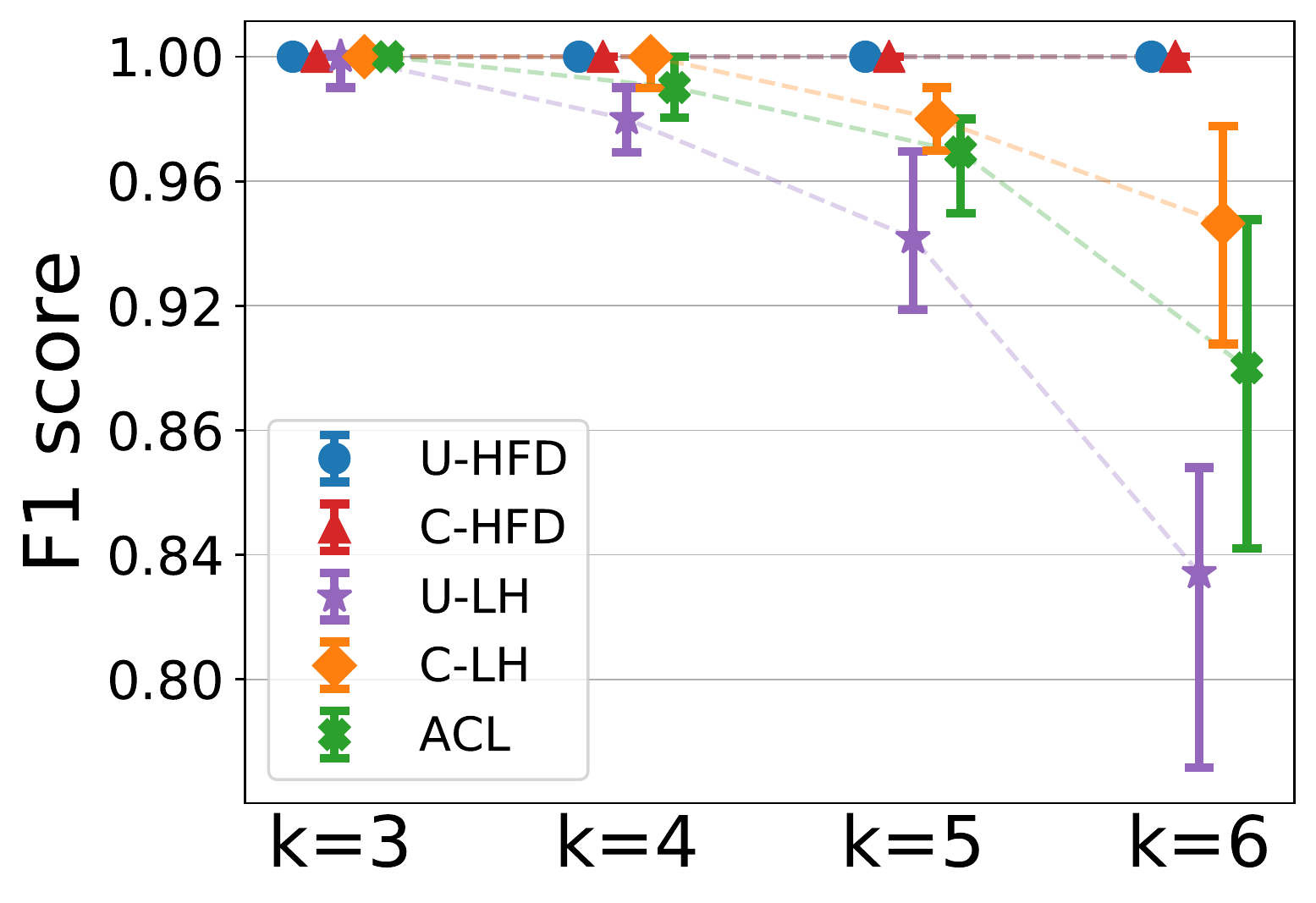}
	\end{subfigure}
	\caption{Conductance ratio and F1 score on $k$-uniform hypergraphs for $k \in \{3,4,5,6\}$. Target clusters have unit conductance around 0.20.} 
	\label{fig:edgesize_cond20}
\end{figure}

\begin{figure}[ht!]
	\centering
	\begin{subfigure}{0.45\textwidth}
		\centering
		\includegraphics[width=.95\textwidth]{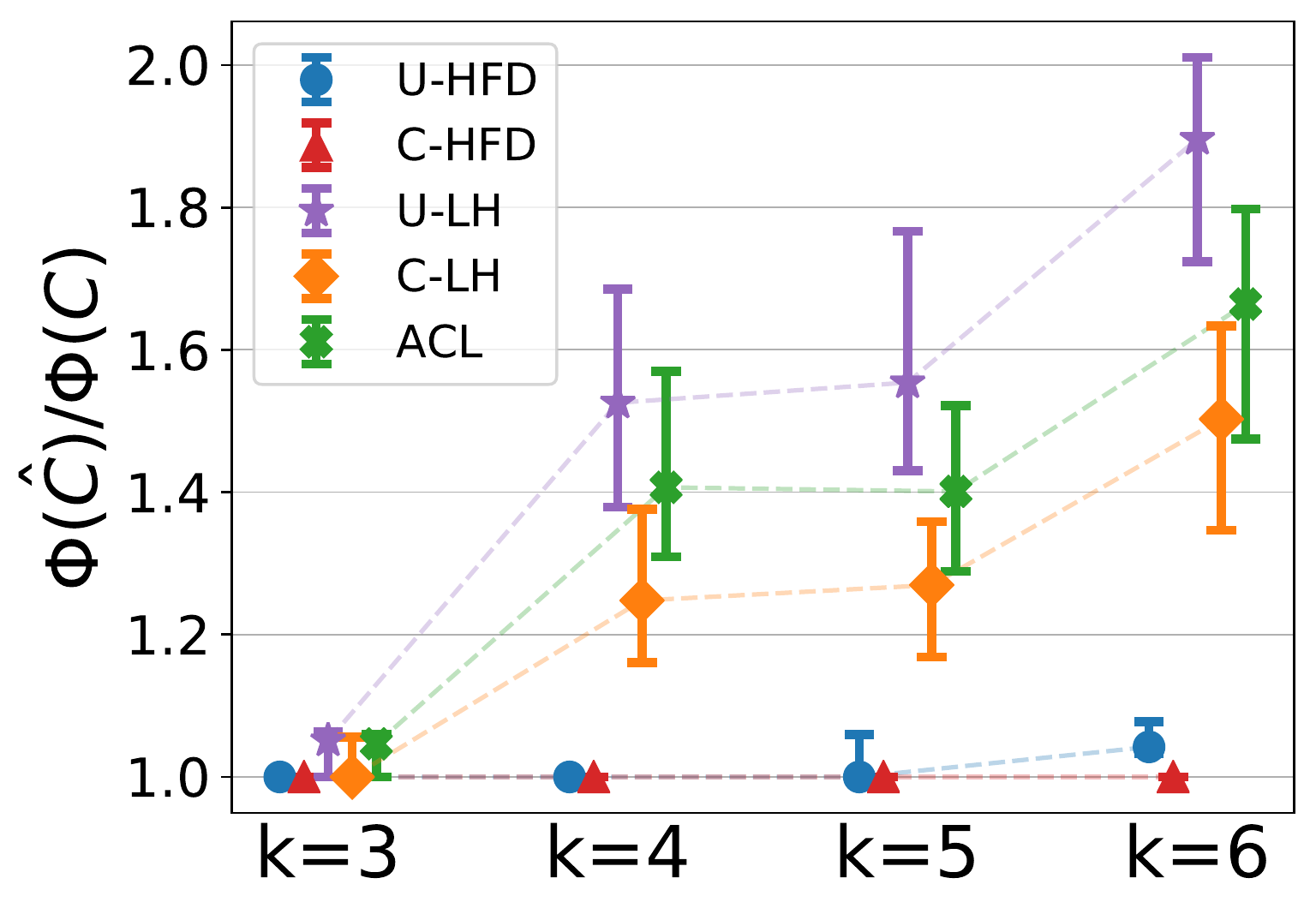}
	\end{subfigure}%
	\begin{subfigure}{0.45\textwidth}
		\centering
		\includegraphics[width=.95\textwidth]{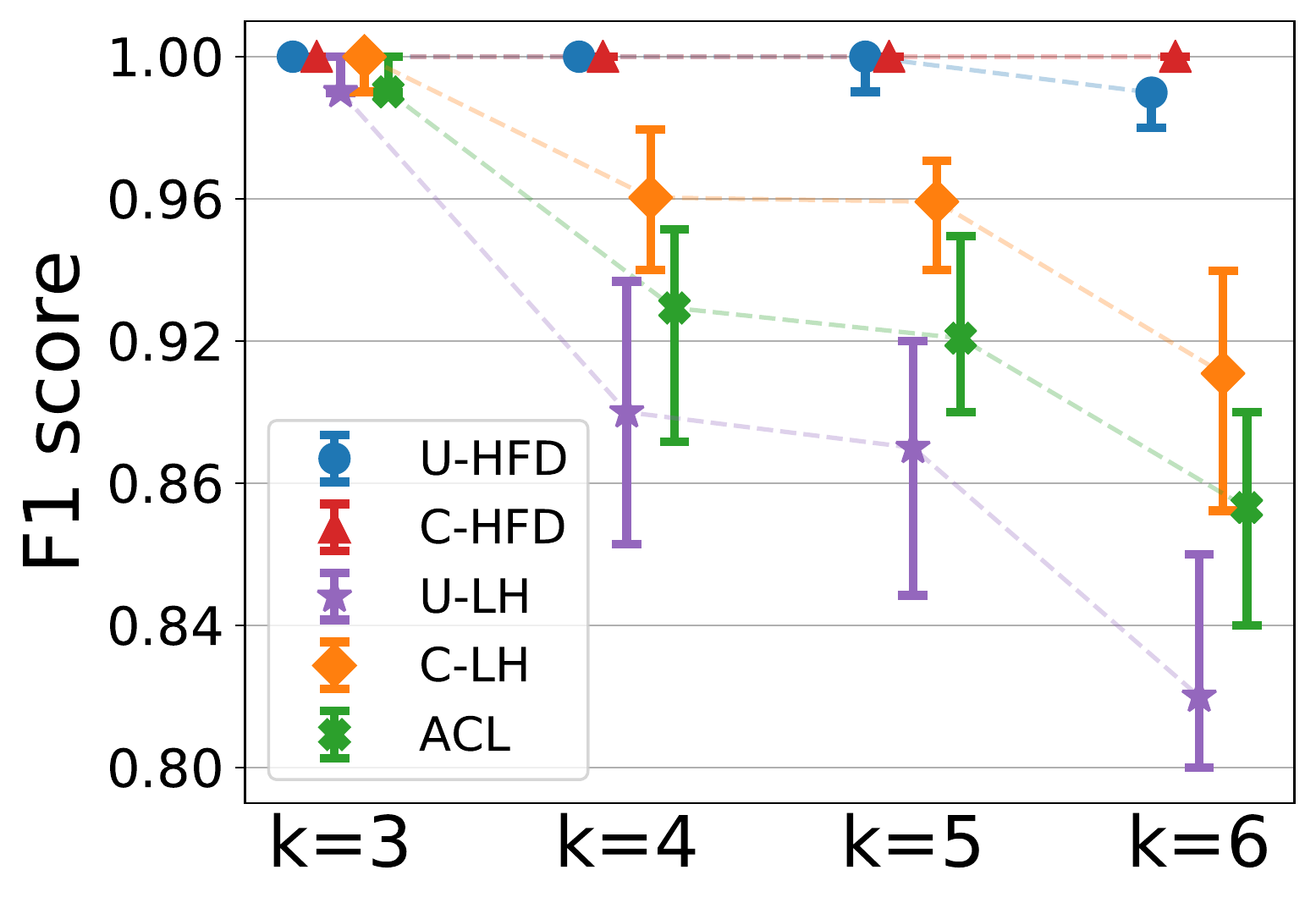}
	\end{subfigure}
	\caption{Conductance ratio and F1 score on $k$-uniform hypergraphs for $k \in \{3,4,5,6\}$. Target clusters have unit conductance around 0.25.} 
	\label{fig:edgesize_cond25}
\end{figure}

{\bf Why is the empirical performance of U-HFD better than U-LH?} For the unit cut-cost setting, the local clustering guarantee for HFD holds under much weaker assumptions than those required for LH. The assumptions for LH could fail in many cases, and consequently we see that U-HFD has significantly better performance than U-LH in the experiments with both synthetic and real data. More specifically, the theoretical framework for LH assumes that the node embeddings are global (i.e., the solution is dense). However, in order to obtain a localized algorithm, the authors use a regularization parameter $\kappa>0$ to impose sparsity in the solution. The localized algorithm computes a sparse approximation to the original global solution, but some clustering errors could also be introduced. In general, this does not seem to be a major issue, as localized solutions only seem to slightly affect the clustering performance as shown in Figure~\ref{fig:LHassump}. A more crucial assumption of LH is that its approximation guarantee relies on a strong condition that the conductance of the target cluster is upper bounded by $\frac{\gamma}{8c}$, where $\gamma \in (0,1)$ is a tuning parameter and $c$ is a constant that depends on both $\gamma$ and a specific sampling strategy for selecting a seed node from the target cluster. In our experiments we find that this assumption often breaks. In what follows we provide a simple illustrating example using synthetic hypergraphs. First of all, we sample a sequence of hypergraphs using $k$HSBM with $n=100$ nodes, two ground-truth communities each consisting of 50 nodes, constant $k=3$, varying $p$ and $q$. For each hypergraph we identify one ground-truth community as the target cluster, and we select a seed node uniformly at random from the target cluster. We compute the quantity $\frac{\gamma}{8c}$ and we find that this quantity is always less than 0.12 for any $\gamma \in (0,1)$. This means that in order for the assumption of LH to hold, the target cluster must have conductance no more than 0.12, which is a very strict requirement and cannot hold in general. In order to compare the performances of LH when its assumption holds or fails, respectively, we picked two hypergraphs (i.e., the fourth set of hypergraphs that we generate) that correspond to the two scenarios. The target clusters have conductance 0.05 and 0.3, respectively. Therefore, the assumption for LH holds for the first hypergraph but fails for the second hypergraph. Moreover, we consider both global and localized solutions for LH. The global solution demonstrates the performance of LH under the required theoretical framework, while the localized solution demonstrates what happens in practice when one uses sparse approximation for computational efficiency. For LH, we compute the global solution by simply setting the regularization parameter $\kappa$ to 0; we tune the localized solution and set $\kappa = 0.25r$ where $r$ is the ratio between the number of seed node(s) and the size of the target cluster. The way we pick $\kappa$ is similar to the authors' choice for LH. For HFD, we set $\sigma = 0.01$ and initial mass 3 times the volume of the target cluster. We run both methods multiple times, each time we use a different node from the target cluster as the single seed node. The median, lower and upper quantiles of F1 scores are shown in Figure~\ref{fig:LHassump}. For LH, observe that (i) for both hypergraphs where the assumption either holds or fails, localizing the solution slightly reduces the F1 score, and (ii) for both global and localized solutions, LH has much worse performance on the hypergraph where its assumption does not hold. On the other hand, HFD perfectly recovers the target clusters in both settings. 

\begin{figure}[htbp!]
	\centering
	\includegraphics[width=.7\textwidth]{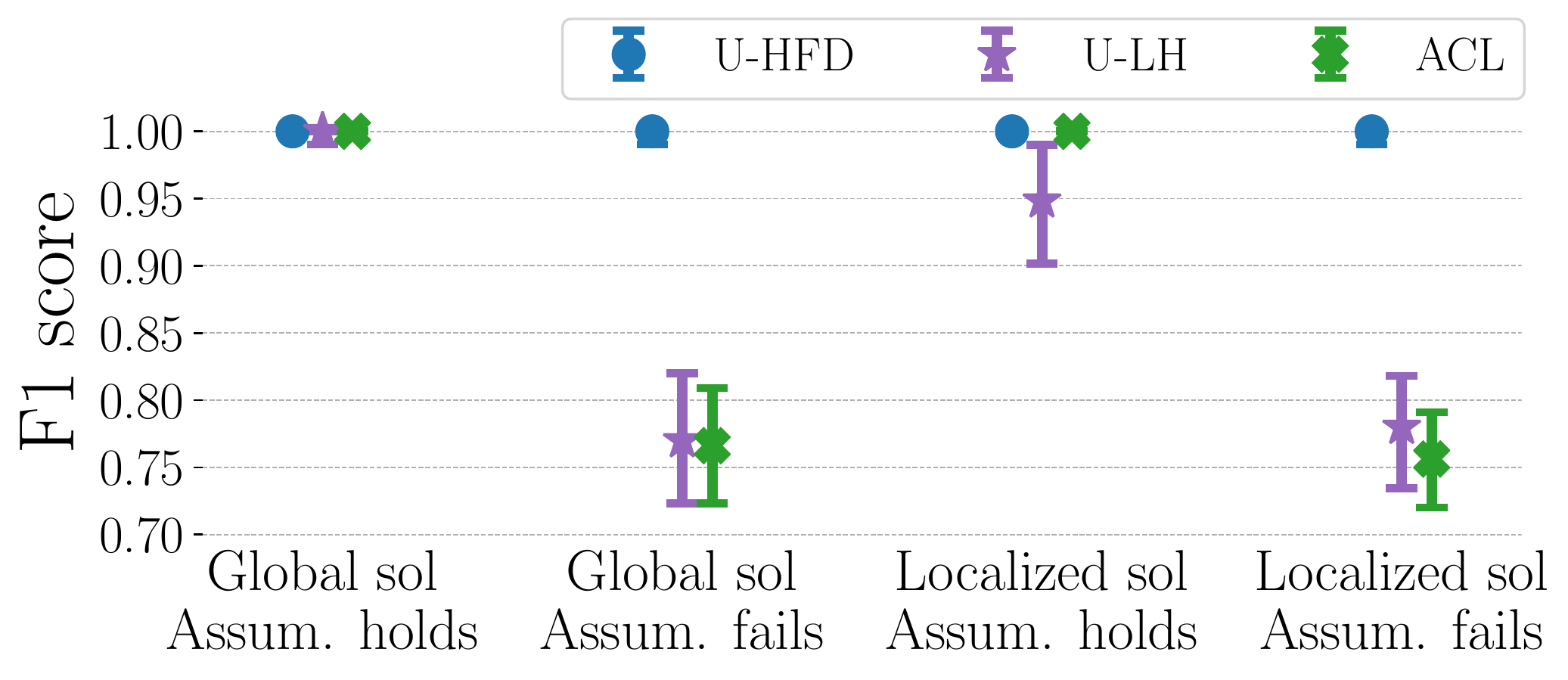}
	\caption{Local clustering results under various settings for LH. The markers show the median, the error bars show the 25th and 75th percentiles, respectively. The left-most case aligns with the required theoretical framework for LH, moreover, the srong assumption on the target cluster conductance is satisfied; the right-most case is what typically happens in practice when one applies localized algorithm for LH, moreover, the assumption on the target cluster conductance does not hold. ACL is a heuristic method that applies to the star expansion of hypergraphs. ACL has no performance guarantee. In practice, we observe that ACL and LH have similar performances.}
	\label{fig:LHassump}
\end{figure}

\subsection{Experiments using real-world data}

\subsubsection{Datasets and ground-truth clusters}
We provide complete details on the real hypergraphs we used in the experiments. The last three datasets are used for additional experiments in the appendix only.

\underline{{\it Amazon-reviews}}~\cite{ni2019justifying,VBK20}.
This is a hypergraph constructed from Amazon product review data, where each node represents a product.\ A set of products are connected by a hyperedge if they are reviewed by the same person.\ We use product category labels as ground truth cluster identities.\ In total there are 29 product categories. Because we are mostly interested in local clustering, we consider all clusters consisting of less than 10,000 nodes.

\underline{{\it Trivago-clicks}}~\cite{CVB2021}. The nodes in this hypergraph are accommodations/hotels.\ A set of nodes are connected by a hyperedge if a user performed ``click-out'' action during the same browsing session, which means the user was forwarded to a partner site. \ We use geographical locations as ground truth cluster identities.\ There are 160 such clusters. We consider all clusters in this dataset that consists of less than 1,000 nodes and has conductance less than 0.25.

\underline{{\it Florida Bay food network}}~\cite{LM17}. Nodes in this hypergraph correspond to different species or organisms that live in the Bay, and hyperedges correspond to transformed network motifs of the original dataset.\ Each species is labelled according its role in the food chain.

\underline{{\it High-school-contact}}~\cite{MFB2015,CVB2021}. Nodes in this hypergraph represent high school students.\ A group of people are connected by a hyperedge if they were all in proximity of one another at a given time, based on data from sensors worn by students.\ We use the classroom to which a student belongs to as ground truth. In total there are 9 classrooms.

\underline{{\it Microsoft-academic}}~\cite{SSSMEHW15,AVB20}. The original co-authorship network is a subset of the Microsoft Academic Graph where nodes are authors and hyperedges correspond to a publication from those authors. We take the dual of the original hypergraph by converting hyperedges to nodes and nodes to hyperedges.\ After constructing the dual hypergraph, we removed all hyperedges having just one node and we kept the largest connected component. \
In the resulting hypergraph, each node represents a paper and is labelled by its publication venue.\ A set of papers are connected by a hyperedge if they share a common coauthor.\ We combine similar computer science conferences into four broader categories: Data (KDD, WWW, VLDB, SIGMOD), ML (ICML, NeurIPS), TCS (STOC, FOCS), CV (ICCV, CVPR).

\underline{{\it Oil-trade network}}. This hypergraph is constructed using the 2017 international oil trade records from UN Comtrade Dataset. We adopt a similar modelling approach to Figure~1 in the main paper. Each node represents a country, $\{v_1,v_2,v_3,v_4\}$ form a hyperedge if the trade surplus from each of $v_1,v_2$ to each of $v_3,v_4$ exceeds 10 million USD (this is roughly 80\% percentile country-wise oil export value). Therefore, two countries belong to the same hyperedge if they share $\ge2$ important trading partners in common. We use this network to for the node ranking problem.

Table~\ref{tab:datasets} provides summary statistics about the hypergraphs. Table~\ref{tab:clusters} includes the statistics of all ground truth clusters that we used in the experiments.

\begin{table}[t!]
\caption{Summary of real-world hypergraphs}
\label{tab:datasets}
\centering
\begin{adjustbox}{max width=\textwidth}
\begin{tabular}{crrrrrr}
\toprule
Dataset & \makecell{Number of \\ nodes} & \makecell{Number of \\ hyperedges} & \makecell{Maximum \\ hyperedge size} & \makecell{Maximum \\ node degree} & \makecell{Median / Mean \\ hyperedge size} & \makecell{Median / Mean \\ node degree}\\
\midrule
Amazon-reviews & 2,268,231 & 4,285,363 & 9,350 & 28,973 & 8.0 / 17.1 & 11.0 / 32.2 \\
Trivago-clicks & 172,738 & 233,202 & 86 & 588 & 3.0 / 4.1 & 2.0 /  5.6 \\
Florida-Bay & 126 & 141,233 & 4 & 19,843 & 4.0 / 4.0 & 3,770.5 / 4,483.6 \\
Microsoft-academic & 44,216 & 22,464 & 187 & 21 & 3.0 / 5.4 & 2.0 / 2.7 \\
High-school-contact & 327 & 7,818 & 5 & 148 & 2.0 / 2.3 & 53.0 /  55.6 \\
Oil-trade & 229 & 100,639 & 4 & 16,394 & 4.0 / 4.0 & 175.0 / 1,757.9 \\
\bottomrule
\end{tabular}
\end{adjustbox}
\end{table}

\begin{table}[t!]
\caption{Summary of ground-truth clusters used in the experiments}
\label{tab:clusters}
\centering
\small
\begin{tabular}{clrrr}
\toprule
Dataset & \makecell{Cluster} & \makecell{Size} & Volume & Conductance \\
\midrule
\multirow{9}{*}{\rotatebox[origin=c]{90}{Amazon-reviews}}
& 1 - Amazon Fashion & 31 & 3042 & 0.06\\
& 2 - All Beauty & 85 & 4092 & 0.12\\
& 3 - Appliances & 48 & 183 & 0.18\\
& 12 - Gift Cards & 148 & 2965 & 0.13\\
& 15 - Industrial \& Scientific & 5334 & 72025 & 0.14\\
& 17 - Luxury Beauty & 1581 & 28074 & 0.11\\
& 18 - Magazine Subs. & 157 & 2302 & 0.13\\
& 24 - Prime Pantry& 4970 & 131114 & 0.10\\
& 25 - Software& 802 & 11884 & 0.14\\
\midrule
\multirow{10}{*}{\rotatebox[origin=c]{90}{Trivago-clicks}} 
& KOR - South Korea & 945 & 3696 & 0.24\\
& ISL - Iceland & 202 & 839 & 0.21\\
& PRI - Puerto Rico & 144 & 473 & 0.25\\
& UA-43 - Crimea & 200 & 1091 & 0.24\\
& VNM - Vietnam & 832 & 2322 & 0.24\\
& HKG - Hong Kong & 536 & 4606 & 0.24\\
& MLT - Malta & 157 & 495 & 0.24\\
& GTM - Guatemala & 199 & 652 & 0.24\\
& UKR - Ukraine & 264 & 648 & 0.24\\
& SET - Estonia & 158 & 850 & 0.23\\
\midrule
\multirow{3}{*}{\rotatebox[origin=c]{90}{\makecell{Florida-\\Bay}}}
& Producers & 17 & 10781 & 0.70\\
& Low-level consumers & 35 & 173311 & 0.58\\
& High-level consumers & 70 & 375807 & 0.54\\
\midrule
\multirow{4}{*}{\rotatebox[origin=c]{90}{\makecell{Microsoft-\\academic}}}
& Data & 15817 & 45060 & 0.06 \\
& ML & 10265 & 26765 & 0.16 \\
& TCS & 4159 & 10065 & 0.08 \\
& CV & 13974 & 38395 & 0.08\\
\midrule
\multirow{9}{*}{\rotatebox[origin=c]{90}{High-school-contact}}
& Class 1 & 36 & 1773 & 0.25\\
& Class 2 & 34 & 1947 & 0.29\\
& Class 3 & 40 & 2987 & 0.20\\
& Class 4 & 29 & 913 & 0.41\\
& Class 5 & 38 & 2271 & 0.26\\
& Class 6 & 34 & 1320 & 0.26\\
& Class 7 & 44 & 2951 & 0.16\\
& Class 8 & 39 & 2204 & 0.19\\
& Class 9 & 33 & 1826 & 0.25\\
\bottomrule
\end{tabular}
\end{table}

\subsubsection{Methods and parameter setting}

{\bf HFD} We use $\sigma = 0.0001$ for all the experiments.\ We set the total amount of initial mass $\|\Delta\|_1$ as a constant factor $t$ times the volume of the target cluster.\ For Amazon-reviews, on the smaller clusters 1, 2, 3, 12, 18, we used $t=200$; on the larger clusters 15, 17, 24, 25, we used $t = 50$. For both Trivago-clicks, High-school-contact and Microsoft-academic, we used $t = 3$.\ For Florida Bay food network, we used $t = 20, 10, 5$ for clusters 1, 2, 3, respectively.\ In all experiments, the choice of $t$ is to ensure that the diffusion process will cover some part of the target and incur a high cost in the objective function.\ For the single seed node setting, we simply set the initial mass on the seed node as $\|\Delta\|_1$.\ For the multiple seed nodes setting where we are given a seed set $S$, for each $v \in S$ we set the initial mass on $v$ as $d_v\|\Delta\|_1/\vol(S)$.\

{\bf LH, ACL} We used the parameters as suggested by the authors \cite{LVHLG20}. For both *-LH-2.0 and *-LH-1.4, we set $\gamma = 0.1$, $\rho = 0.5$, $\kappa = c \cdot r$ where $r$ is the ratio between the number of seed nodes and the size of the target cluster, and $c$ is a tuning constant.\ For Amazon-reviews, we set $c = 0.025$ as suggested in \cite{LVHLG20}.\ For Microsoft-academic, Trivago-clicks, and Florida-Bay we also used $c = 0.025$ because it produces good results.\ For High-school-contact we selected $c = 0.25$ after some tuning to make sure both *-LH-2.0 and *-LH-1.4 have good results.\ We set the parameters for ACL in exactly the same way as in \cite{LVHLG20}.\ We set $\delta = 1$ for U-LH-* and $\delta = \max_{e \in E}|e|$ for C-LH-*.

\subsubsection{Additional experiments}

{\bf Multiple seed nodes.} We conduct additional experiments using multiple seed nodes for Amazon-reviews and Trivago-clicks datasets.\ For each target cluster, we randomly select 1\% nodes from that cluster as seed nodes, and we enforce that at least 5 nodes are selected as seeds.\ For example, if a cluster consists of only 100 nodes, we still select 5 nodes to form a seed set.\ We run 30 trials for each cluster and report the median conductance and F1 score of the output clusters.\ The results are shown in Table~\ref{tab:amazon-results-complete} and Table~\ref{tab:trivago-results-complete}.\ For the multiple seed nodes setting, the results of U-LH-1.4, U-LH-2.0 and ACL on Amazon-reviews align with the ones reported in \cite{LVHLG20}: We reproduced almost identical numbers under the same setting, with only a few small differences due to randomness in seed nodes selection.\ In general, using more seed nodes improves the performance for all methods in terms of both conductance and F1.\ For Amazon-reviews, the output clusters of HFD always have the lowest conductance, even though in some cases, low conductance does not align well with the given ground-truth, and hence the lowest conductance does not lead to the highest F1 score.\ Similarly, for Trivago-clicks, both U-HFD and C-HFD consistently find the lowest conductance clusters among all methods, which in general (but not always) lead to a higher F1 score.\ Note that, if a method uses the unit cut-cost (resp. the cardinality-based cut-cost), then we compute the conductance of the output cluster using the unit cut-cost (resp. the cardinality-based cut-cost).\ Therefore, depending on the specific cut-cost, the conductances in Table~\ref{tab:trivago-results-complete} may have different scales.\ We highlight the lowest conductance for both cut-costs separately.\

\begin{table}[t!]
  \caption{Complete local clustering results for Amazon-reviews network}\label{tab:amazon-results-complete}
  \centering
  \small
  \setlength\tabcolsep{4pt}
  \begin{tabular}{cclccccccccc}
    \toprule
    & & & \multicolumn{9}{c}{Cluster} \\
    \cmidrule(l{2pt}r{2pt}){4-12}
    Metric & Seed & Method & 1 & 2  & 3 & 12 & 15 & 17 & 18 & 24 & 25 \\
    \midrule
    \multirow{8}{*}{\rotatebox[origin=c]{90}{Conductance}} & 
    \multirow{4}{*}{\rotatebox[origin=c]{90}{Single}} &
    U-HFD & \bf 0.17 & \bf 0.11 & \bf 0.12 & \bf 0.16 & \bf 0.36 & \bf 0.25 & \bf 0.17 & \bf 0.14 & \bf 0.28 \\
    & & U-LH-2.0 & 0.42 & 0.50 & 0.25 & 0.44 & 0.74 & 0.44 & 0.57 & 0.58 & 0.61 \\
    & & U-LH-1.4 & 0.33 & 0.44 & 0.25 & 0.36 & 0.81 & 0.40 & 0.51 & 0.54 & 0.59 \\
    & & ACL & 0.42 & 0.50 & 0.25 & 0.54 & 0.77 & 0.52 & 0.63 & 0.68 & 0.65  \\
    \cmidrule{2-12}
    & \multirow{4}{*}{\rotatebox[origin=c]{90}{Multiple}} &
    U-HFD & \bf 0.05 & \bf 0.10 & \bf 0.12 & \bf 0.13 & \bf 0.20 & \bf 0.16 & \bf 0.14 & \bf  0.11 & \bf 0.32 \\
    & & U-LH-2.0 & \bf 0.05 & 0.15 & 0.15 & 0.21 & 0.45 & 0.45 & 0.26 & 0.18 & 0.53 \\
    & & U-LH-1.4 & \bf 0.05 & 0.13 & 0.15 & 0.15 & 0.35 & 0.33 & 0.19 & 0.14 & 0.47 \\
    & & ACL & \bf 0.05 & 0.27 & 0.16 & 0.27 & 0.56 & 0.53 & 0.33 & 0.30 & 0.59 \\
    \midrule
    \multirow{8}{*}{\rotatebox[origin=c]{90}{F1 score}} & 
    \multirow{4}{*}{\rotatebox[origin=c]{90}{Single}} &
    U-HFD & \bf 0.45 & \bf 0.09 & \bf 0.65 & \bf 0.92 & 0.04 & \bf 0.10 & \bf 0.80 & \bf 0.81 & \bf 0.09 \\
    & & U-LH-2.0 & 0.23 & 0.07 & 0.23 & 0.29 & \bf 0.05 & 0.06 & 0.21 & 0.28 & 0.05 \\
    & & U-LH-1.4 & 0.23 & \bf 0.09 & 0.35 & 0.40 & 0.00 & 0.07 & 0.31 & 0.35 & 0.06 \\
    & & ACL &  0.23 & 0.07 & 0.22 & 0.25 & 0.04 & 0.05 & 0.17 & 0.20 & 0.04  \\
    \cmidrule{2-12}
    & \multirow{4}{*}{\rotatebox[origin=c]{90}{Multiple}} &
    U-HFD & 0.49 & \bf 0.50 & 0.69 & \bf 0.98 & 0.19 & \bf 0.36 & \bf 0.91 & \bf 0.89 & \bf 0.33 \\
    & & U-LH-2.0 & \bf 0.59 & 0.42 & \bf 0.73 & 0.77 & 0.22 & 0.25 & 0.65 & 0.62 & 0.17 \\
    & & U-LH-1.4 & 0.52 & 0.45 & \bf 0.73 & 0.90 & \bf 0.27 & 0.29 & 0.79 & 0.77 & 0.20 \\
    & & ACL & \bf 0.59 & 0.25 & 0.70 & 0.64 & 0.20 & 0.19 & 0.51 & 0.49 & 0.14 \\
    \bottomrule
  \end{tabular}
\end{table}

\begin{table}[t!]
  \caption{Complete local clustering results for Trivago-clicks network}\label{tab:trivago-results-complete}
  \centering
  \small
  \setlength\tabcolsep{4pt}
  \begin{tabular}{cclcccccccccc}
    \toprule
    & & & \multicolumn{10}{c}{Cluster} \\
    \cmidrule(l{2pt}r{2pt}){4-13}
    Metric & Seed & Method & KOR & ISL & PRI & UA-43 & VNM & HKG & MLT & GTM & UKR & EST\\
    \midrule
    \multirow{14}{*}{\rotatebox[origin=c]{90}{Conductance}} & 
    \multirow{7}{*}{\rotatebox[origin=c]{90}{Single}} &
    U-HFD & \bf 0.010 & \bf 0.023 & \bf 0.014 & \bf 0.011 & \bf 0.018 & \bf 0.017 & \bf 0.010 & \bf 0.007 & \bf 0.016 & \bf 0.012 \\
    & & U-LH-2.0 & 0.020 & 0.042 & 0.027 & 0.027 & 0.037 & 0.035 & 0.031 & 0.035 & 0.032 & 0.019 \\
    & & U-LH-1.4 & 0.036 & 0.069 & 0.047 & 0.039 & 0.060 & 0.052 & 0.040 & 0.045 & 0.065 & 0.036 \\
    & & ACL & 0.027 & 0.050 & 0.034 & 0.031 & 0.042 & 0.043 & 0.047 & 0.039 & 0.043 & 0.026 \\
    \cdashline{3-13}
    & & C-HFD & \bf 0.007 & \bf 0.016 & \bf 0.007 & \bf 0.005 & \bf 0.009 & \bf 0.011 & \bf 0.007 & \bf 0.003 & \bf 0.010 & \bf 0.009 \\
    & & C-LH-2.0 & 0.022 & 0.066 & 0.030 & 0.030 & 0.035 & 0.035 & 0.029 & 0.028 & 0.029 & 0.029 \\
    & & C-LH-1.4 & 0.043 & 0.095 & 0.042 & 0.048 & 0.071 & 0.059 & 0.053 & 0.047 & 0.075 & 0.046 \\
    \cmidrule{2-13}
    & \multirow{7}{*}{\rotatebox[origin=c]{90}{Multiple}} &
    U-HFD & \bf 0.009 & \bf 0.023 & \bf 0.011 & \bf 0.010 & \bf 0.014 & \bf 0.017 & \bf 0.010 & \bf 0.008 & \bf 0.017 & \bf 0.012 \\
    & & U-LH-2.0 & 0.023 & 0.034 & 0.018 & 0.021 & 0.054 & 0.030 & 0.021 & 0.022 & 0.041 & 0.018 \\
    & & U-LH-1.4 & 0.048 & 0.045 & 0.038 & 0.032 & 0.084 & 0.051 & 0.049 & 0.049 & 0.085 & 0.024 \\
    & & ACL & 0.030 & 0.037 & 0.018 & 0.024 & 0.064 & 0.033 & 0.021 & 0.024 & 0.045 & 0.020 \\
    \cdashline{3-13}
    & & C-HFD & \bf 0.006 & \bf 0.016 & \bf 0.006 & \bf 0.005 & \bf 0.006 & \bf 0.011 & \bf 0.007 & \bf 0.003 & \bf 0.011 & \bf 0.009 \\
    & & C-LH-2.0 & 0.024 & 0.062 & 0.021 & 0.021 & 0.047 & 0.034 & 0.023 & 0.017 & 0.036 & 0.029 \\
    & & C-LH-1.4 & 0.054 & 0.067 & 0.033 & 0.037 & 0.094 & 0.057 & 0.053 & 0.044 & 0.094 & 0.032 \\
    \midrule
    \multirow{14}{*}{\rotatebox[origin=c]{90}{F1 score}} & 
    \multirow{7}{*}{\rotatebox[origin=c]{90}{Single}} &
    U-HFD & 0.75 & \bf 0.99 & 0.89 & 0.85 	& 0.28 & 0.82 & \bf 0.98 & 0.94 & 0.60 & \bf 0.94 \\
    & & U-LH-2.0 & 0.70 & 0.86 & 0.79 & 0.70 & 0.24 & 0.92 & 0.88 & 0.82 & 0.50 & 0.90 \\
    & & U-LH-1.4 & 0.69 & 0.84 & 0.80 & 0.75 & 0.28 & 0.87 & 0.92 & 0.83 & 0.47 & 0.90 \\
    & & ACL & 0.65 & 0.84 	& 0.75 & 0.68 & 0.23 & 0.90 & 0.83 & 0.69 & 0.50 & 0.88 \\
    & & C-HFD & \bf 0.76 & \bf 0.99 & \bf 0.95 & \bf 0.94 & \bf 0.32 & 0.80 & \bf 0.98 & \bf 0.97 & \bf 0.68 & \bf 0.94 \\
    & & C-LH-2.0 & 0.73 & 0.90 & 0.84 & 0.78 & 0.27 & \bf 0.94 & 0.96 & 0.88 & 0.51 & 0.83 \\
    & & C-LH-1.4 & 0.71 & 0.88 & 0.84 & 0.78 & 0.27 & 0.88 & 0.93 & 0.85 	& 0.50 & 0.85 \\
    \cmidrule{2-13}
    & \multirow{7}{*}{\rotatebox[origin=c]{90}{Multiple}} &
    U-HFD & \bf 0.87 & \bf 0.99 & \bf 0.97 & 0.92 & 0.55 & 0.82 & \bf 0.98 & \bf 0.97 & 0.87 & \bf 0.94 \\
    & & U-LH-2.0 & 0.83 & 0.91 & 0.92 & 0.84 & 0.71 & 0.93 & 0.95 & 0.93 & 0.86 & 0.92 \\
    & & U-LH-1.4 & 0.78 & 0.84 & 0.83 & 0.79 & 0.74 & 0.85 & 0.85 & 0.84 & 0.75 & 0.87 \\
    & & ACL & 0.81 & 0.89 & 0.91 & 0.85 & 0.68 & 0.93 & 0.96 & 0.91 & 0.83 & 0.90 \\
    & & C-HFD & 0.86 & \bf 0.99 & \bf 0.97 & \bf 0.96 & 0.32 & 0.80 & \bf 0.98 & \bf 0.97 & 0.69 & \bf 0.94 \\
    & & C-LH-2.0 & 0.86 & 0.94 & 0.94 & 0.87 & \bf 0.76 & \bf 0.94 & 0.97 & 0.94 & \bf 0.88 & 0.91 \\
    & & C-LH-1.4 & 0.83 & 0.89 & 0.90 & 0.83 & 0.67 & 0.89 & 0.92 & 0.85 & 0.77 & 0.89 \\
    \bottomrule
  \end{tabular}
\end{table}

{\bf Additional datasets, local clustering using unit and cardinality cut-costs.} Table~\ref{tab:school-results} and Table~\ref{tab:academic-results} show local clustering results on High-school-contact and Microsoft-academic networks, respectively. We use the single seed node setting, run the methods from each node in a target cluster, and report the median conductance and F1 score. We cap the maximum number of runs to 500. Similar to the results on other datasets, the output clusters of HFD always have the lowest conductance, leading to the highest F1 score in most cases. We omit cardinality-based methods for Microsoft-academic because they are very similar to the unit cut-cost setting.

\begin{table}[t!]
  \caption{Local clustering results for High-school-contact network}\label{tab:school-results}
  \centering
  \small
  \setlength\tabcolsep{4pt}
  \begin{tabular}{clccccccccc}
    \toprule
    & & \multicolumn{9}{c}{Cluster} \\
    \cmidrule(l{2pt}r{2pt}){3-11}
    Metric & Method & Class 1 & Class 2 & Class 3 & Class 4 & Class 5 & Class 6 & Class 7 & Class 8 & Class 9 \\
    \midrule
    \multirow{7}{*}{\rotatebox[origin=c]{90}{Conductance}}
    & U-HFD & \bf 0.25 & \bf 0.29 & \bf 0.13 & \bf 0.42 & \bf 0.21 & \bf 0.26 & \bf 0.16 & \bf 0.19 & \bf 0.25 \\
    & U-LH-2.0 & 0.31 & 0.36 & 0.23 & 0.63 & 0.33 & 0.36 & 0.18 & 0.21 & 0.30 \\
    & U-LH-1.4 & 0.29 & 0.32 & 0.21 & 0.54 & 0.29 & 0.37 & \bf 0.16 & 0.22 & 0.29 \\
    & ACL & 0.62 & 0.64 & 0.61 & 0.98 & 0.61 & 0.60 & 0.59 & 0.55 & 0.59 \\
    \cdashline{3-11}
    & C-HFD  & \bf0.25 & \bf 0.28 & \bf 0.20 & \bf 0.41 & \bf 0.24 & \bf 0.26 & \bf 0.16 & \bf 0.19 & \bf 0.25 \\
    & C-LH-2.0 & 0.27 & 0.33 & \bf 0.20 & 0.57 & 0.29 & 0.32 & \bf 0.16 & 0.20 & 0.27 \\
    & C-LH-1.4 & 0.28 & 0.32 & \bf 0.20 & 0.52 & 0.28 & 0.33 & \bf 0.16 & 0.21 & 0.28 \\
    \midrule
    \multirow{7}{*}{\rotatebox[origin=c]{90}{F1 score}}
    & U-HFD & \bf 0.99 & \bf 1.00 & 0.59 & \bf 0.96 & 0.73 & \bf 1.00 & 0.88 & \bf 1.00 & \bf 0.99 \\
    & U-LH-2.0 & 0.91 & 0.83 & 0.93 & 0.66 & \bf 0.84 & 0.88 & 0.96 & 0.96 & 0.90 \\
    & U-LH-1.4 & 0.93 & 0.78 & 0.90 & 0.78 & 0.70 & 0.90 & 0.97 & 0.95 & 0.88 \\
    & ACL & 0.72 & 0.73 & 0.73 & 0.06 & 0.70 & 0.76 & 0.77 & 0.78 & 0.76 \\
    & C-HFD & \bf 0.99 & \bf 1.00 & \bf 1.00 & \bf 0.96 & 0.80 & \bf 1.00 & \bf 1.00 & \bf 1.00 & \bf 0.99 \\
    & C-LH-2.0 & 0.93 & 0.82 & 0.92 & 0.74 & 0.84 & 0.93 & 0.97 & 0.97 & 0.91 \\
    & C-LH-1.4 & 0.94 & 0.74 & 0.69 & 0.84 & 0.76 & 0.94 & 0.96 & 0.96 & 0.85 \\
    \bottomrule
  \end{tabular}
\end{table}

\begin{table}[t!]
  \caption{Local clustering results for Microsoft-academic network}\label{tab:academic-results}
  \centering
  \small
  \setlength\tabcolsep{4pt}
  \begin{tabular}{clcccc}
    \toprule
    & & \multicolumn{4}{c}{Cluster} \\
    \cmidrule(l{2pt}r{2pt}){3-6}
    Metric & Method & Data & ML & TCS & CV \\
    \midrule
    \multirow{4}{*}{\rotatebox[origin=c]{90}{Cond}}
    & U-HFD & \bf 0.03 & \bf 0.06 & \bf 0.06 & \bf 0.03 \\
    & U-LH-2.0 & 0.07 & 0.09 & 0.10 & 0.07 \\
    & U-LH-1.4 & 0.07 & 0.08 & 0.09 & 0.07 \\
    & ACL & 0.08 & 0.11 & 0.11 & 0.09 \\
    \midrule
    \multirow{4}{*}{\rotatebox[origin=c]{90}{F1 score}}
    & U-HFD & \bf 0.78 & \bf 0.54 & \bf 0.86 & \bf 0.73 \\
    & U-LH-2.0 & 0.67 & 0.46 & 0.71 & 0.61 \\
    & U-LH-1.4 & 0.65 & 0.46 & 0.59 & 0.59 \\
    & ACL & 0.64 & 0.43 & 0.70 & 0.57 \\
    \bottomrule
  \end{tabular}
\end{table}

{\bf Additional dataset, node ranking using general submodular cut-cost.} We provide another compelling use case of general submodular cut-cost.\ We consider the node ranking problem in the Oil-trade network.\ Our goal is to search the most related country of a queried country based on the trade-network structure.\ We use the hypergraph modelling shown in Figure~1 in the main paper.\ We compare HFD using unit (U-HFD, $\gamma_1=\gamma_2=1$), cardinality-based (C-HFD, $\gamma_1=1/2$ and $\gamma_2=1$) and submodular (S-HFD, $\gamma_1=1/2$ and $\gamma_2=0$) cut-costs.\ Table~\ref{tab:trade-rank} shows the top-2 ranking results.\ In this example, we use Iran as the seed node and we rank other countries according to the ordering of dual variables returned by HFD.\ In 2017, US imposed strict sanctions on Iran.\ However, Bangladesh (generally accepted as an American ally) is among the top two ranked countries based on unit or cardinality-based cut-cost, which does not make any sense.\ On the other hand, S-HFD ranks Iraq and Turkmenistan as the top two.\ Interested readers can easily verify that these counties share strong economic or historical ties with Iran.\

{\bf Additional method: $p$-norm HFD.} We tried HFD with unit cut-cost and $p=4$ (U-HFD-4.0). However, in practice we did not observe that a larger $p>2$ necessarily lead to better clustering results. We show a sample result of U-HFD-4.0 for Amazon-reviews in Table~\ref{tab:pnorm-amazon}. Notice that the performances of U-HFD-2.0 ($p=2$) and U-HFD-4.0 are very similar.

{\bf Additional method: LH $+$ flow improve.} We tried a flow-improve method for hypergraphs~\cite{VBK20}. We apply the flow-improve method to the output of U-LH-2.0. The method is slow in our experiments, so we only tried it on a few small instances. The results for the Florida Bay food network is shown in Table~\ref{tab:food-unit-results}. In general, we find that applying the flow-improve method does not lead to consistent performance improvements.

\begin{table}[t!]
  \caption{Top-2 node-ranking results for Oil-trade network}\label{tab:trade-rank}
  \centering
  \small
  \setlength\tabcolsep{4pt}
  \begin{tabular}{cl}
    \toprule
    Method & Query: Iran \\
    \midrule
    U-HFD & Kenya, Bangladesh \\
    C-HFD & Bangladesh, United Rep. of Tanzania \\
    S-HFD & Turkmenistan, Iraq \\
    \bottomrule
  \end{tabular}
\end{table}

\begin{table}[t!]
  \caption{Local clustering results for Amazon-reviews network using $p$-norm HFD}\label{tab:pnorm-amazon}
  \centering
  \small
  \setlength\tabcolsep{4pt}
  \begin{tabular}{cclccccccccc}
    \toprule
    & & & \multicolumn{9}{c}{Cluster} \\
    \cmidrule(l{2pt}r{2pt}){4-12}
    Metric & Seed & Method & 1 & 2  & 3 & 12 & 15 & 17 & 18 & 24 & 25 \\
    \midrule
    \multirow{4}{*}{\rotatebox[origin=c]{90}{Cond}} & 
    \multirow{2}{*}{Single} &
    U-HFD-2.0 & 0.17 & 0.11 & 0.12 & 0.16 & 0.36 & 0.25 & 0.17 & 0.14 & 0.28 \\
    & & U-HFD-4.0 & 0.17 & 0.10 & 0.12 & 0.16 & 0.35 & 0.26 & 0.17 & 0.14 & 0.38 \\
    \cmidrule{2-12}
    & \multirow{2}{*}{Multiple} &
    U-HFD-2.0 & 0.05 & 0.10 & 0.12 & 0.13 & 0.20 & 0.16 & 0.14 &  0.11 & 0.32 \\
    & & U-HFD-4.0 & 0.05 & 0.10 & 0.12 & 0.14 & 0.20 & 0.16 & 0.14 & 0.12 & 0.32 \\
    \midrule
    \multirow{4}{*}{\rotatebox[origin=c]{90}{F1 score}} & 
    \multirow{2}{*}{Single} &
    U-HFD-2.0 & 0.45 & 0.09 & 0.65 & 0.92 & 0.04 & 0.10 & 0.80 & 0.81 & 0.09 \\
    & & U-HFD-4.0 & 0.48 & 0.07 & 0.65 & 0.92 & 0.04 & 0.09 & 0.80 & 0.82 & 0.10 \\
    \cmidrule{2-12}
    & \multirow{2}{*}{Multiple} &
    U-HFD-2.0 & 0.49 & 0.50 & 0.69 & 0.98 & 0.19 & 0.36 & 0.91 & 0.89 & 0.33 \\
    & & U-HFD-4.0 & 0.49 & 0.50 & 0.69 & 0.98 & 0.19 & 0.36 & 0.91 & 0.88 & 0.35\\
    \bottomrule
  \end{tabular}
\end{table}

\begin{table}[t!]
  \caption{Local clustering results for the food network using unit cut-costs}\label{tab:food-unit-results}
  \centering
  \small
  \setlength\tabcolsep{4pt}
  \begin{tabular}{clccc}
    \toprule
    & & \multicolumn{3}{c}{Cluster} \\
    \cmidrule(l{2pt}r{2pt}){3-5}
    Metric & Method & Producers & Low-level consumers & High-level consumers \\
    \midrule
    \multirow{5}{*}{\rotatebox[origin=c]{90}{Conductance}}
    & U-HFD & \bf 0.49 & \bf 0.36 & \bf 0.35 \\
    & U-LH-2.0 & 0.51 & 0.39 & 0.39 \\
    & U-LH-2.0 $+$ flow & 0.52 & 0.39 & 0.40 \\
    & U-LH-1.4 & \bf 0.49 & 0.39 & 0.41 \\
    & ACL & 0.52 & 0.39 & 0.40 \\
    \midrule
    \multirow{5}{*}{\rotatebox[origin=c]{90}{F1 score}}
    & U-HFD & \bf 0.69 & \bf 0.47 & \bf 0.64 \\
    & U-LH-2.0 & \bf 0.69 & 0.45 & 0.57 \\
    & U-LH-2.0 $+$ flow & \bf 0.69 & 0.45 & 0.57 \\
    & U-LH-1.4 & \bf 0.69 & 0.45 & 0.58 \\
    & ACL & \bf 0.69 & 0.44 & 0.57 \\
    \bottomrule
  \end{tabular}
\end{table}

\subsection{Computing platform and implementation detail}

We implemented the AM algorithm~\cite{Beck2015} given in Algorithm~\ref{alg:lpAM} in Julia.\ The code is run on a personal laptop with 32GB RAM and 2.9 GHz 6-Core Intel Core i9.\ GPU is not used for computation.\ For the rest of this section, we discuss the implementation details on how we actually solve the nontrivial sub-problem in Algorithm~\ref{alg:lpAM} to obtain the update $(\phi^{(k+1)},r^{(k+1)})$.\

For the unit cut-cost case, we use an exact projection algorithm~\cite{li2020quadratic} to obtain the update $(\phi^{(k+1)},r^{(k+1)})$.\ Algorithmic details for exact projection is provided in Algorithm~\ref{alg:unit}.\ For cardinality-based or general submodular cut-costs, a conic Fujishige-Wolfe minimum norm algorithm~\cite{li2020quadratic} can be adopted to efficiently compute $(\phi^{(k+1)},r^{(k+1)})$.\ Our implementation uses alternative methods that are simpler.\ For the cardinality cut-cost, we use a projected subgradient method that works on a related dual problem to obtain the primal update in $(\phi^{(k+1)},r^{(k+1)})$. The subgradient method is easy to implement, requires less computation overhead, and works well in practice for the sub-problem.\ For the specialized submodular cut-cost shown in Figure~1, since the hyperedge consists of only 4 nodes and has a special structure, we simply perform an exhaustive search that allows us to exactly compute $(\phi^{(k+1)},r^{(k+1)})$ using constant number of vector-vector additions and multiplications.\ We provide details below.\


Recall that the sub-problem to compute $(\phi^{(k+1)},r^{(k+1)})$ decomposes into a sequence of separate problems indexed by $e \in E$ (cf.~\eqref{eq:r-step-primal}, in the following we assume $p=2$ for simplicity):
\begin{equation}
\label{eq:subprob-primal}
	\min_{\phi_e\ge0,r_e\in\phi_eB_e} \frac{1}{2}\phi_e^2 + \frac{1}{2\sigma}\|s_e - r_e\|_2^2.
\end{equation}
The dual problem of \eqref{eq:subprob-primal} is written as (cf.~\eqref{eq:r-step-dual}, here we have $p=q=2$)
\begin{equation}
\label{eq:subprob-dual}
	\min_{y_e} \frac{1}{2}f_e(y_e)^2 + \frac{\sigma}{2}\|y_e\|_2^2 - s_e^Ty_e.
\end{equation}
Let $(\phi_e^*, r_e^*)$ and $y_e^*$ denote primal and dual optimal solutions for \eqref{eq:subprob-primal} and \eqref{eq:subprob-dual}, respectively. Then we have that
\[
	r_e^* + \sigma y_e^* = s_e \quad \mbox{and} \quad {\phi_e^*}^2 = {r_e^*}^Ty_e^*.
\]
The dual problem~\eqref{eq:subprob-dual} can be derived following the same way that we derive the primal-dual HFD formulations, moreover, the above relations between $\phi_e^*,$ $r_e^*$ and $y_e^*$ follow immediately from the primal-dual derivation, dual optimality condition and simple algebraic work. Therefore, in order to find an optimal solution $(\phi_e^*, r_e^*)$ for the primal problem~\eqref{eq:subprob-dual}, it suffices to find an optimal solution $y_e^*$ for the dual problem~\eqref{eq:subprob-dual} and then recover $(\phi_e^*, r_e^*)$. Now, since $\one^Tr_e^* = 0$, we know that $\sigma\one^T y_e^* = \one^Ts_e$, i.e., $y_e^*$ lies in the hyperplane $\mathcal{H} := \{y_e | \sigma\one^T y_e = \one^Ts_e\}$. Let $h$ denote the objective function of the dual problem~\eqref{eq:subprob-dual}, we compute $y_e^*$ using projected subgradient method:
\[
	y_e^{(k+1)} := P_{\mathcal{H}}\left(y_e^{(k)} - \frac{1}{k}\frac{g^{(k)}}{\|g^{(k)}\|_2}\right),
\]
where $g^{(k)} \in \partial h(y_e^{(k)})$ is a subgradient at $y_e^{(k)}$, and $P_{\mathcal{H}}(\cdot)$ denotes the projection onto the hyperplane $\mathcal{H}$. We add the additional projection step so that, when we stop the subgradient method after $K$ iterations to get $y_e^* \approx \tilde{y}_e := y_e^{(K)}$, and approximately recover $r_e^*$ as $r_e^* \approx \tilde{r}_e  := s_e - \sigma \tilde{y}_e$, the resulting $\tilde{r}_e$ would still be a proper flow routing, i.e., $\one^T \tilde{r}_e = 0$ and hence it is possible to have $\tilde{r}_e \in \tilde{\phi}_e B_e$ for some $\tilde{\phi}_e$. In other words, the projection step is crucial because it permits the use of sub-optimal dual solution $\tilde{y}_e$ to obtain sub-optimal but feasible primal solution $\tilde{r}_e$.

For the cardinality cut-cost, our implementation uses the projected subgradient method we describe above to solve the sub-problem in Algorithm~\ref{alg:lpAM} for $\phi_e$ and $r_e$. In what follows we talk about how we deal with the specialized submodular cut-cost. 

Given $e = \{v_1,v_2,v_3,v_4\}$ and associated submodular cut-cost $w_e$ such that $w_e(\{v_i\}) = 1/2$ for $i = 1,2,3,4$, $w_e(\{v_1,v_2\}) = 0$, $w_e(\{v_1,v_3\}) = w_e(\{v_1,v_4\}) = 1$, and $w_e(S) = w_e(e \setminus S)$ for any $S \subseteq e$. Let $B_e$ be the base polytope of $w_e$.
The sub-problem for this hyperedge is given in \eqref{eq:subprob-primal}. Suppose $(\phi_e^*,r_e^*)$ is optimal for \eqref{eq:subprob-primal}, and $r_e^* = \phi_e^* \rho_e^*$ for some $\rho_e^* \in B_e$. If $\phi_e^* > 0$, then we know that $\phi_e^* = \frac{s_e^T\rho_e^*}{\sigma + \|\rho_e^*\|_2^2}$. To see this, substitute $r_e^* = \phi_e \rho_e^*$ into \eqref{eq:subprob-primal} and optimize for $\phi_e$ only. The relation $\phi_e^* = \frac{s_e^T\rho_e^*}{\sigma + \|\rho_e^*\|_2^2}$ follows from first-order optimality condition and the assumption that $\phi_e^* > 0$. On the other hand, if $\phi_e^* = 0$, then we simply have that $r_e^* = 0$. Therefore, in order to compute $(\phi_e^*,r_e^*)$ when $\phi_e^* > 0$, it suffices to find $\rho_e^*$. In order to find $\rho_e^*$, we look at the dual problem~\ref{eq:subprob-dual}. Let $y_e^*$ be an optimal dual solution, then we have that $\rho_e^* \in \argmax_{\rho_e \in B_e} \rho_e^Ty_e^*$. The subsequent claims are case analyses in order to determine all possible nontrivial candidates for $\rho_e^*$.

\begin{claim}
\label{claim:equal}
If $s_{e,v_1} = s_{e,v_2}$, then $\rho^*_{e,v_1} = \rho^*_{e,v_2} = 0$; if $s_{e,v_3} = s_{e,v_4}$, then $\rho^*_{e,v_3} = \rho^*_{e,v_4} = 0$.
\end{claim}
\begin{proof}
The optimality condition of the dual problem \eqref{eq:subprob-dual} is for some $\hat{\rho}_e \in \argmax_{\rho_e \in B_e}\rho_e^Ty_e^*$,
\begin{equation}
\label{eq:subprob-optcond}
	(\hat{\rho}_e^Ty_e^*)\hat{\rho}_e + \sigma y_e^* = s_e.
\end{equation}
Suppose $s_{e,v_1} = s_{e,v_2}$, then we must have $y_{e,v_1}^* = y_{e,v_2}^*$. Otherwise, say $y_{e,v_1}^* > y_{e,v_2}^*$, then we know that $\hat{\rho}_{e,v_1} = 1/2 > -1/2 = \hat{\rho}_{e,v_2}$, which follows from applying the greedy algorithm~\cite{Bach2011a} to find $\hat{\rho}_e$ using the order of indices in $y_e^*$. But then according to the optimality condition~\eqref{eq:subprob-optcond}, we have
\[
s_{e,v_1} = (\hat{\rho}_e^Ty_e^*)\hat{\rho}_{e,v_1} + \sigma y_{e,v_1}^* > (\hat{\rho}_e^Ty_e^*)\hat{\rho}_{e,v_2} + \sigma y_{e,v_2}^* = s_{e,v_2},
\]
which contradicts our assumption that $s_{e,v_1} = s_{e,v_2}$.\ Similarly, $y_{e,v_1}^* < y_{e,v_2}^*$ is not possible, either. Now, because $y_{e,v_1}^* = y_{e,v_2}^*$, by the optimality condition~\eqref{eq:subprob-optcond}, we must also have $\hat{\rho}_{e,v_1} = \hat{\rho}_{e,v_2}$. Finally, because $\hat{\rho}_e \in B_e$, we know that $\hat{\rho}_{e,v_1} + \hat{\rho}_{e,v_2} \le 0$ and $\hat{\rho}_{e,v_1} + \hat{\rho}_{e,v_2} = - (\hat{\rho}_{e,v_3} + \hat{\rho}_{e,v_4}) \ge -w_e(\{v_3,v_4\}) = 0$, so $\hat{\rho}_{e,v_1} + \hat{\rho}_{e,v_2} = 0$. Therefore, $\hat{\rho}_{e,v_1} = \hat{\rho}_{e,v_2} = 0$. Since $\hat{\rho}$ was chosen arbitrarily from the set $\argmax_{\rho_e \in B_e}\rho_e^Ty_e^*$, and $\rho^*_e \in \argmax_{\rho_e \in B_e}\rho_e^Ty_e^*$, we have that $\rho^*_{e,v_1} = \rho^*_{e,v_2} = 0$ as required. The other claim on nodes $v_3$ and $v_4$ follows the same way.
\end{proof}

\begin{claim}
\label{claim:halfequal}
If $s_{e,v_1} \neq s_{e,v_2}$ and $s_{e,v_3} = s_{e,v_4}$, then $\rho^*_{e,v_1},\rho^*_{e,v_2} \in \{1/2,-1/2\}$ and $\rho^*_{e,v_3} = \rho^*_{e,v_4} = 0$; if $s_{e,v_1} = s_{e,v_2}$ and $s_{e,v_3} \neq s_{e,v_4}$, then $\rho^*_{e,v_1} = \rho^*_{e,v_2} = 0$ and $\rho^*_{e,v_3},\rho^*_{e,v_4} \in \{1/2,-1/2\}$.\end{claim}

\begin{proof}
We will show the first case, the second case follows by symmetry. Let $\hat{\rho}_e \in \argmax_{\rho_e \in B_e} \rho_e^T y_e^*$. Suppose $s_{e,v_1} \neq s_{e,v_2}$ and $s_{e,v_3} = s_{e,v_4}$. Then by Claim~\ref{claim:equal} we have $\hat{\rho}_{e,v_3} = \hat{\rho}_{e,v_4} = 0$. Let us assume without loss of generality that $s_{e,v_1} > s_{e,v_2}$. If $y_{e,v_1}^* < y_{e,v_2}^*$, then apply the greedy algorithm we know that $\hat{\rho}_{e,v_1} = -1/2 < 1/2 = \hat{\rho}_{e,v_2}$. But this contradicts the optimality condition \eqref{eq:subprob-optcond}. Therefore we must have $y_{e,v_1}^* \ge y_{e,v_2}^*$. There are two cases. If $y_{e,v_1}^* > y_{e,v_2}^*$, then apply the greedy algorithm we get $\hat{\rho}_{e,v_1} = 1/2$ and $\hat{\rho}_{e,v_2} = -1/2$. If $y_{e,v_1}^* = y_{e,v_2}^*$, then because $\hat{\rho}_{e,v_1} + \hat{\rho}_{e,v_2} = 0$ (see the proof of Claim~\ref{claim:equal} for an argument for this) and $\hat{\rho}_{e,v_3} = \hat{\rho}_{e,v_4} = 0$, we have that $\hat{\rho}_e^Ty_e^* = 0$. But then this contradicts the optimality condition \eqref{eq:subprob-optcond}, because $s_{e,v_1} > s_{e,v_2}$ and $y_{e,v_1}^* = y_{e,v_2}^*$. Therefore we cannot have $y_{e,v_1}^* = y_{e,v_2}^*$. Since our choice of $\hat{\rho}_e\in \argmax_{\rho_e \in B_e} \rho_e^T y_e^*$ was arbitrary, and $\rho_{e,v_1}^* \in \argmax_{\rho_e \in B_e} \rho_e^T y_e^*$, so we know that $\rho_e^*$ must satisfy the properties satisfied by $\hat{\rho}_e$.
\end{proof}

\begin{claim}
\label{claim:notequal}
If $s_{e,v_1} \neq s_{e,v_2}$ and $s_{e,v_3} \neq s_{e,v_4}$, then $\rho^*_{e,v_1},\rho^*_{e,v_2} \in \{\pm 1/2,\pm a\}$ and $\rho^*_{e,v_3},\rho^*_{e,v_4} \in \{\pm1/2,\pm b\}$, where $a = (\frac{1}{2}+\sigma)(s_{e,v_1}-s_{e,v_2})/(s_{e,v_3}-s_{e,v_4})$ and $b = (\frac{1}{2}+\sigma)(s_{e,v_3}-s_{e,v_4})/(s_{e,v_1}-s_{e,v_2})$.
\end{claim}

\begin{proof}
Let us assume without loss of generality that $s_{e,v_1} > s_{e,v_2}$ and $s_{e,v_3} > s_{e,v_4}$. Let $\hat{\rho}_e \in \argmax_{\rho_e \in B_e} \rho_e^T y_e^*$. We have that $y_{e,v_1}^* \ge y_{e,v_2}^*$ and $y_{e,v_3}^* \ge y_{e,v_4}^*$ (see the proof of Claim~\ref{claim:halfequal} for an argument for this). There are four cases and we analyze them one by one in the following.

{\em Case 1}. If $y_{e,v_1}^* > y_{e,v_2}^*$ and $y_{e,v_3}^* > y_{e,v_4}^*$, then we have $\hat{\rho}_{e,v_1} = \hat{\rho}_{e,v_3} = 1/2$ and $\hat{\rho}_{e,v_2} = \hat{\rho}_{e,v_4} = -1/2$. 

{\em Case 2}. If $y_{e,v_1}^* = y_{e,v_2}^*$ and $y_{e,v_3}^* = y_{e,v_4}^*$, then $\hat{\rho}_e^Ty_e^* = 0$ and hence the optimality condition \eqref{eq:subprob-optcond} cannot be satisfied. This leads to a contradiction. 

{\em Case 3}. Suppose that $y_{e,v_1}^* = y_{e,v_2}^*$ and $y_{e,v_3}^* > y_{e,v_4}^*$. Then according to the optimality condition \eqref{eq:subprob-optcond}, because $s_{e,v_1} > s_{e,v_2}$ and $y_{e,v_1}^* = y_{e,v_2}^*$, we must have that $\hat{\rho}_{e,v_1} > \hat{\rho}_{e,v_2}$. Moreover, because $\hat{\rho}_{e,v_1} + \hat{\rho}_{e,v_2} = 0$, we know that $\hat{\rho}_{e,v_1} = a = -\hat{\rho}_{e,v_2}$ for some $a > 0$. We also know that $\hat{\rho}_{e,v_3} =1/2$ and $\hat{\rho}_{e,v_4} = -1/2$ since $y_{e,v_3}^* > y_{e,v_4}^*$. Substitute the primal-dual relation $\phi_e^* = \hat{\rho}_e^Ty_e^*$ into \eqref{eq:subprob-optcond} we have
\[
	\phi_e^*\hat{\rho}_{e,v_1} + \sigma y_{e,v_1}^* = s_{e,v_1} ~~\mbox{and}~~ \phi_e^*\hat{\rho}_{e,v_2} + \sigma y_{e,v_2}^* = s_{e,v_2}.
\]
Because $y_{e,v_1}^* = y_{e,v_2}^*$, we get that
\[
	\phi_e^*(\hat{\rho}_{e,v_1} - \hat{\rho}_{e,v_2}) =  s_{e,v_1}  - s_{e,v_2},
\]
and hence 
\begin{equation}
\label{eq:phi_e_1}
\phi_e^* = \frac{s_{e,v_1}  - s_{e,v_2}}{\hat{\rho}_{e,v_1} - \hat{\rho}_{e,v_2}} = \frac{s_{e,v_1}  - s_{e,v_2}}{2a}.
\end{equation}
Because $\hat{\rho} \in \argmax_{\rho_e \in B_e} \rho_e^T y_e^*$ was arbitrary, and $\rho_e^* \in \argmax_{\rho_e \in B_e} \rho_e^T y_e^*$, we know that $\rho_{e,v_1}^* = a = -\rho_{e,v_2}^*$ and $\rho_{e,v_3}^* = 1/2 = -\rho_{e,v_4}^*$. On the other hand, since $s_{e,v_1} > s_{e,v_2}$ we know that $\phi_e^* > 0$, therefore
\begin{equation}
\label{eq:phi_e_2}
	\phi_e^* = \frac{s_e^T\rho_e^*}{\sigma + \|\rho_e^*\|_2^2} = \frac{a(s_{e,v_1}  - s_{e,v_2}) + \frac{1}{2}(s_{e,v_3}  - s_{e,v_4})}{\sigma + 2a^2 + \frac{1}{2}}.
\end{equation}
Combining equations \eqref{eq:phi_e_1} and \eqref{eq:phi_e_2} we get that $a = (\frac{1}{2}+\sigma)(s_{e,v_1}-s_{e,v_2})/(s_{e,v_3}-s_{e,v_4})$. 

{\em Case 4}. Suppose that $y_{e,v_1}^* > y_{e,v_2}^*$ and $y_{e,v_3}^* = y_{e,v_4}^*$. The following a similar argument for Case 3, we get that $\rho_{e,v_1}^* = 1/2 = -\rho_{e,v_2}$ and $\rho_{e,v_3}^* = b = -\rho_{e,v_4}^*$ where $b = (\frac{1}{2}+\sigma)(s_{e,v_3}-s_{e,v_4})/(s_{e,v_1}-s_{e,v_2})$.
\end{proof}

Finally, combining Claims~\ref{claim:equal},~\ref{claim:halfequal},~\ref{claim:notequal} and the constraint that $\rho_{e,v_1}^* + \rho_{e,v_2}^* = \rho_{e,v_3}^* + \rho_{e,v_4}^* = 0$, there are at most 12 possible choices for $\rho_e^*$. Therefore, an exhaustive search among these candidate vectors for $\rho_e^*$ (and hence $\phi_e^* = \frac{s_e^T\rho_e^*}{\sigma + \|\rho_e^*\|_2^2}$ and $r_e^*=\phi_e^*\rho_e^*$) that minimizes \eqref{eq:subprob-primal} can be done using constant number of vector-vector additions and multiplications.\

\end{document}